\documentclass[11pt]{article}

\usepackage{amssymb}
\usepackage{amsmath}
\usepackage{amsthm}
\usepackage{amsfonts}
\usepackage{dsfont}
\usepackage{stmaryrd}
\usepackage[subnum]{cases}
\usepackage[utf8]{inputenc}
\usepackage[T1]{fontenc}
\usepackage{vmargin}
\usepackage{verbatim}
\usepackage{hyperref}
\usepackage[pdftex]{graphicx,color}
\usepackage{float}
\usepackage{natbib}

\bibpunct{(}{)}{;}{a}{,}{,}

\graphicspath{{./Figures/}{.}}

\setpapersize{A4}
\setmarginsrb{30mm}{10mm}{30mm}{30mm}{10mm}{10mm}{10mm}{10mm}

\newcommand{\ps}[1]{{\left \langle #1 \right \rangle}}
\newcommand{\argmax}{\operatornamewithlimits{argmax}}

\newcommand{\sign}{\operatornamewithlimits{sign}}

\newtheorem{example}{Example}
\newtheorem{definition}{Definition}
\newtheorem{lemma}{Lemma}

\newtheorem{proposition}{Proposition}
\newtheorem{remark}{Remark}
\newtheorem{theorem}{Theorem}
\newtheorem{corollary}{Corollary}
\newtheorem{hypothesis}{Hypothesis}

\long\def\acks#1{\vskip 0.3in\noindent{\large\bf Acknowledgements}
\noindent #1}

\begin{document}

\begin{titlepage}

\title{Combinatorial and Structural Results for $\gamma$-$\Psi$-dimensions}

\author{
    Yann Guermeur \\
    LORIA-CNRS \\
    Campus Scientifique, BP~239 \\
    54506 Vand\oe uvre-l\`es-Nancy Cedex, France \\
    (e-mail: {\tt Yann.Guermeur@loria.fr})
  }

\date{September 12, 2020}

\maketitle

\noindent{\bf Running Title}:
Combinatorial and Structural Results for $\gamma$-$\Psi$-dimensions



\thispagestyle{empty}

\end{titlepage}

\begin{abstract}
This article deals with the generalization performance of margin multi-category
classifiers, when minimal learnability hypotheses are made. In that context,
the derivation of a guaranteed risk is based on the handling
of capacity measures belonging to three main families:
Rademacher/Gaussian complexities, metric entropies and
scale-sensitive combinatorial dimensions. The scale-sensitive combinatorial
dimensions dedicated to the classifiers of this kind are
the $\gamma$-$\Psi$-dimensions. We introduce the combinatorial and structural
results needed to involve them in the derivation of 
guaranteed risks and establish the corresponding upper bounds
on the metric entropies and the Rademacher complexity.
Two major conclusions can be drawn:
\begin{enumerate}
\item the $\gamma$-$\Psi$-dimensions always bring
an improvement compared to the use of the fat-shattering dimension
of the class of margin functions;
\item thanks to their capacity to take into account 
basic features of the classifier,
they represent a promising alternative to performing the transition
from the multi-class case to the binary one with covering numbers.
\end{enumerate}
\end{abstract}

\section{Introduction}
\label{sec:introduction}

One of the main open problems of
the theory of margin multi-category pattern classification
is the characterization of the way the confidence interval of
an upper bound on the probability of error
should vary as a function of the three basic
parameters which are the sample size $m$, the number $C$ of categories
and the margin parameter $\gamma$
\citep[see][for a survey]{KonWei14}.
When working under minimal learnability hypotheses,
the derivation of such a {\em guaranteed risk}
is based on the handling of capacity measures
belonging to three main families:
Rademacher/Gaussian complexities \citep{BarMen02},
metric entropies \citep{KolTih61}
and scale-sensitive combinatorial dimensions \citep{KeaSch94}.
The scale-sensitive combinatorial dimensions dedicated to
the classifiers of interest are the $\gamma$-$\Psi$-dimensions \citep{Gue07b}.
Their usefulness to derive guaranteed risks rests on the availability
of two types of results. {\em Combinatorial results}
\citep{AloBenCesHau97,MenVer03,RudVer06,MusLauGue19} 
connect them to metric entropies.
{\em Structural results} \citep{Dua12,Mau16,Gue17}
perform the transition from the multi-class case to the bi-class one.
This article introduces such results for the two main
$\gamma$-$\Psi$-dimensions and incorporate them
in the derivation of upper bounds on the metric entropies and 
the Rademacher complexity.
The dependence of the resulting guaranteed risks
on $m$, $C$ and $\gamma$ is characterized.
This establishes that in the theoretical framework of interest,
introducing $\gamma$-$\Psi$-dimensions always brings an improvement
compared to the use of the fat-shattering dimension
of the class of margin functions.
Furthermore, the margin Natarajan dimension appears very promising
to take into account basic features of the classifier. 
In practice, for many popular classifiers,
applying a structural result to this capacity measure
rather than to covering numbers should improve the confidence interval,
primarily in its dependence on $\gamma$.

The organization of the paper is as follows.
Section~\ref{sec:margin-multi-category-classifiers}
introduces the theoretical framework.
Section~\ref{sec:connections-between-capacity-measures}
highlights the need for new tools to improve the multi-class bounds.
Section~\ref{sec:new-results-for-gamma-Psi-dimensions}
establishes that switching from the fat-shattering dimension of
the class of margin functions to the $\gamma$-$\Psi$-dimensions 
of the same class improves the combinatorial results.
Sections~\ref{sec:new-combinatorial-results-for-gamma-Psi-dimensions} and
\ref{sec:new-structural-results-for-gamma-Psi-dimensions}
introduce and discuss the new combinatorial and structural results
dedicated to these dimensions.
The corresponding bounds on the metric entropies and guaranteed risks
are derived in Section~\ref{sec:bounds-on-metric-entropy}.
At last, we draw conclusions in Section~\ref{sec:conclusions}.
To make reading easier,
all technical lemmas and proofs have been gathered in appendix.

\section{Margin Multi-category Classifiers}
\label{sec:margin-multi-category-classifiers}

We work under minimal assumptions on the data and the classifiers,
which exhibit one important feature:
for each description, they return one score per category.

\subsection{Theoretical Framework}

Let $\left \llbracket n_- ; n_+ \right \rrbracket$
denote the set of integers ranging from $n_-$ to $n_+$.
We consider the case of $C$-category pattern classification problems
with $C \in \mathbb{N} \setminus \llbracket 0; 2 \rrbracket$.
$\mathcal{X}$ is the description space
and $\mathcal{Y} = \llbracket 1; C \rrbracket$ the set of categories.
Their connection is utterly characterized by an unknown probability measure $P$.
Let $Z = \left ( X,Y \right )$ be a random pair
with values in $\mathcal{Z} = \mathcal{X} \times \mathcal{Y}$,
distributed according to $P$.
We are given an $m$-sample
$\mathbf{Z}_m = \left( Z_i  \right)_{1 \leqslant i \leqslant m} =
\left( \left ( X_i, Y_ i \right) \right)_{1 \leqslant i \leqslant m}$
made up of independent copies of $Z$ (in short $\mathbf{Z}_m \sim P^m$).
The classifiers are based on classes of vector-valued functions
with one component function per category.
We add a basic learnability hypothesis:
the classes of component functions are {\em uniform Glivenko-Cantelli}
(uGC) \citep{DudGinZin91}.
Those classes must be uniformly bounded up to additive constants.
We replace this property by a slightly stronger one:
the vector-valued functions take their values in a hypercube of
$\mathbb{R}^C$. To sum up, we make minimal hypotheses
to ensure that all capacity measures met in the sequel are finite
(none of the bounds formulated is trivial).

\begin{definition}[Margin classifier]
\label{def:margin-multi-category-classifiers}
Let $\mathcal{G} = \prod_{k=1}^C \mathcal{G}_k$ be a class of
functions from $\mathcal{X}$ into
$\left [-M_{\mathcal{G}}, M_{\mathcal{G}} \right ]^{C}$
with $M_{\mathcal{G}} \in \left [ 1, +\infty \right )$.
The classes $\mathcal{G}_k$ of component functions are supposed to be
uGC classes.
For each $g = \left ( g_{k} \right )_{1 \leqslant k \leqslant C}
\in \mathcal{G}$, a {\em margin multi-category classifier} on
$\mathcal{X}$ is obtained by application of
the {\em decision rule} $\text{dr}$ from $\mathcal{G}$ into
$\in \left ( \mathcal{Y} \bigcup \left \{ * \right \} \right )^{\mathcal{X}}$.
This classifier, $\text{dr}_g$,
returns either the index of the component
function whose value is the highest, or the dummy category $*$ in case of
ex \ae quo.
\end{definition}
The generalization
capabilities of such classifiers can be characterized by means of the values
taken by the differences of the component functions.
This calls for the introduction of the class of margin functions,
margin loss functions and the corresponding margin risks.

\begin{definition}[Class $\rho_{\mathcal{G}}$ of margin functions]
\label{def:class-of-margin-functions}
Let $\mathcal{G}$ be a function class satisfying
Definition~\ref{def:margin-multi-category-classifiers}.
For every $g \in \mathcal{G}$, the {\em margin function}
$\rho_g$ from $\mathcal{Z}$
into $\left [-M_{\mathcal{G}}, M_{\mathcal{G}} \right ]$ is defined by:
$\forall \left ( x, k \right ) \in \mathcal{Z}, \;
\rho_g \left ( x, k \right ) = \frac{1}{2} \left ( g_k \left ( x \right ) -
\max_{l \neq k} g_l \left ( x \right ) \right )$.
Then, $\rho_{\mathcal{G}}$ is defined as:
$\rho_{\mathcal{G}} = \left \{ \rho_g: \; g \in \mathcal{G} \right \}$.
\end{definition}
The risk of $g \in \mathcal{G}$ is given by:
$L \left ( g \right )
= \mathbb{E}_{\left ( X, Y \right ) \sim P}
\left [ \mathds{1}_{\left \{ \rho_g \left ( X, Y \right )
\leqslant 0 \right \}} \right ]
= P \left ( \text{dr}_g \left ( X \right ) \neq Y \right )$.

\begin{definition}[Margin loss functions]
\label{def:margin-loss-functions}
A class of {\em margin loss functions} $\phi_{\gamma}$
parameterized by $\gamma \in \left ( 0, 1 \right ]$ is a class
of nonincreasing functions from $\mathbb{R}$ into $\left [ 0, 1 \right ]$
satisfying:
$$
\begin{cases}
\forall \gamma \in \left ( 0, 1 \right ], \;
\phi_{\gamma} \left ( 0 \right ) = 1
\text{ and } \phi_{\gamma} \left ( \gamma \right ) = 0 \\
\forall \left ( \gamma, \gamma' \right ) \in \left ( 0, 1 \right ]^2, \;
\gamma < \gamma^{\prime} \Longrightarrow \phi_{\gamma^{\prime}}
\text{ majorizes } \phi_{\gamma}
\end{cases}.
$$
\end{definition}
Given $\phi_{\gamma}$,
the risk with margin $\gamma$ of $g$,  $L_{\gamma} \left ( g \right )$,
is defined as:
$L_{\gamma} \left ( g \right ) = \mathbb{E}_{Z \sim P}
\left [ \phi_{\gamma} \circ \rho_g \left ( Z \right ) \right ]$.
$L_{\gamma, m} \left ( g \right )$
designates the corresponding empirical risk,
measured on $\mathbf{Z}_m$.
When using $\phi_{\gamma}$, the behavior of the margin functions
outside the interval $\left [ 0, \gamma \right ]$
is irrelevant to characterize the generalization performance.
The idea to exploit this property by means of
a squashing function can be traced back to \citet{Bar98}.
The present study uses the function $\pi_{\gamma}$.

\begin{definition}[Squashing function $\pi_{\gamma}$]
\label{def:piecewise-linear-squashing-function}
For $\gamma \in \left (0, 1 \right ]$,
the {\em piecewise-linear squashing function} $\pi_{\gamma}$ is defined by:
$\forall t \in \mathbb{R}, \;\;
\pi_{\gamma} \left ( t \right ) =
t \mathds{1}_{\left \{ t \in \left ( 0, \gamma \right ] \right \}}
+ \gamma \mathds{1}_{\left \{ t > \gamma \right \}}$.
\end{definition}
Thus, when possible, we replace the class $\rho_{\mathcal{G}}$
with the class $\rho_{\mathcal{G}, \gamma}$.
\begin{definition}[Class $\rho_{\mathcal{G}, \gamma}$
of squashed margin functions]
\label{def:class-of-regret-functions}
Let $\mathcal{G}$ be a function class
satisfying Definition~\ref{def:margin-multi-category-classifiers} and
$\rho_{\mathcal{G}}$ the function class deduced from
$\mathcal{G}$ according to Definition~\ref{def:class-of-margin-functions}.
For every pair $\left ( g, \gamma \right )
\in \mathcal{G} \times \left (0, 1 \right ]$,
the function
$\rho_{g, \gamma}$ from $\mathcal{Z}$
into $\left [ 0, \gamma \right ]$ is defined by:
$\rho_{g, \gamma} = \pi_{\gamma} \circ \rho_g$.
Then, the class $\rho_{\mathcal{G}, \gamma}$ is defined as follows:
$\rho_{\mathcal{G}, \gamma} =
\left \{ \rho_{g, \gamma}: \; g \in \mathcal{G} \right \}$.
\end{definition}
The introduction of $\rho_{\mathcal{G}, \gamma}$,
whose capacity is always bounded from above by that of $\rho_{\mathcal{G}}$
(see Section~\ref{sec:state-of-the-art}),
can narrow the confidence interval of the guaranteed risk
without affecting its data-fit term
(since $\forall \gamma \in \left ( 0, 1 \right ], \;\;
\phi_{\gamma} \circ \pi_{\gamma} = \phi_{\gamma}$).
Thus, making the best of it is a major challenge.

\subsection{Guaranteed Risks}

In the theoretical framework of interest, the starting point of the derivation
of a guaranteed risk is a supremum inequality taking the form:
\begin{equation}
\label{eq:basic-supremum-inequality}
P^m \left \{ \sup_{g \in \mathcal{G}}
\left ( L_* \left ( g \right ) -
L_{\gamma, m} \left ( g \right ) \right )
> F_i \left ( m, \gamma, \delta, 
\text{cap} \left ( \rho_{\mathcal{G}, \gamma} \right ) \right )
\right \} \leqslant \delta,
\end{equation}
where $L_*$ is either $L$ or $L_{\gamma}$ and
the capacity measure $\text{cap} \left ( \rho_{\mathcal{G}, \gamma} \right )$
involved in the expression of the function $F_i$
depends on the choice of $\phi_{\gamma}$.
Then, the problem consists in upper bounding
$\text{cap} \left ( \rho_{\mathcal{G}, \gamma} \right )$
as a function of the basic parameters $m$, $C$ and $\gamma$,
so that eventually, with probability $1 - \delta$, the supremum of the empirical
process of interest is bounded from above by a function
$F_f$ of $m$, $C$, $\gamma$ and $\delta$ only, i.e.,
$$
\sup_{g \in \mathcal{G}}
\left ( L_* \left ( g \right ) - L_{\gamma, m} \left ( g \right ) \right )
\leqslant F_f \left ( m, C, \gamma, \delta \right ).
$$

We introduce the three types of capacity measures considered in this study,
using the notations of \citet{Gue17}.
Let $\left ( \mathcal{T}, \mathcal{A}_{\mathcal{T}} \right )$
be a measurable space and
let $\mathcal{F} \subset \mathbb{R}^{\mathcal{T}}$.
Let $T$ be a random variable with values in $\mathcal{T}$,
distributed according to a probability measure
on $\left ( \mathcal{T}, \mathcal{A}_{\mathcal{T}} \right )$ and
let $\mathbf{T}_n = \left( T_i  \right)_{1 \leqslant i \leqslant n}$
be an $n$-sample made up of independent copies of $T$.
The empirical Rademacher complexity of $\mathcal{F}$
given $\mathbf{T}_n$ is denoted by $\hat{R}_n \left ( \mathcal{F} \right )$
and the Rademacher complexity of $\mathcal{F}$ is denoted by
$R_n \left ( \mathcal{F} \right )$.
The classes $\mathcal{F}$ considered here are endowed with
empirical (pseudo-)metrics derived from the $L_p$-norms.
For $n \in \mathbb{N}^*$, let
$\mathbf{t}_n =
\left ( t_i \right )_{1 \leqslant i \leqslant n} \in \mathcal{T}^n$. Then,
$\forall \left ( f, f' \right ) \in \mathcal{F}^2$,
$\forall p \in \left [ 1, +\infty \right )$,
$d_{p, \mathbf{t}_n}  \left ( f, f' \right )
= \left ( \frac{1}{n} \sum_{i=1}^n
\left | f \left ( t_i \right ) -
f' \left ( t_i \right ) \right |^p
\right )^{\frac{1}{p}}$ and
$d_{\infty, \mathbf{t}_n}  \left ( f, f' \right )
= \max_{1 \leqslant i \leqslant n}
\left | f \left ( t_i \right ) - f' \left ( t_i \right ) \right |$.
Let $\bar{\mathcal{F}}$ be a subset of $\mathcal{F}$.
For $\epsilon \in \mathbb{R}_+^*$, $n \in \mathbb{N}^*$
and $p \in \left [ 1,  +\infty \right ]$,
$\mathcal{N} \left ( \epsilon, \bar{\mathcal{F}},
d_{p, \mathbf{t}_n} \right )$ and
$\mathcal{M} \left ( \epsilon, \bar{\mathcal{F}},
d_{p, \mathbf{t}_n} \right )$
respectively denote the $\epsilon$-covering number and
the $\epsilon$-packing number of
$\bar{\mathcal{F}}$ with respect to $d_{p, \mathbf{t}_n}$.
$\mathcal{N}_p \left ( \epsilon, \bar{\mathcal{F}}, n \right )$ and
$\mathcal{M}_p \left ( \epsilon, \bar{\mathcal{F}}, n \right )$
are the corresponding uniform covering and packing numbers.
$\mathcal{N}^{\text{int}}$ and $\mathcal{N}_p^{\text{int}}$
are used to denote proper covering numbers.
The binary logarithm of the covering number of a set 
is called its metric entropy.
The scale-sensitive combinatorial dimensions used are
$\gamma$-$\Psi$-dimensions,
i.e., scale-sensitive extensions
of the $\Psi$-dimensions \citep{BenCesHauLon95}.

\begin{definition}[$\gamma$-$\Psi$-dimensions, 
Definition~28 in \citealp{Gue07b}]
\label{def:gamma-Psi-dimensions}
Let $\mathcal{F} \subset \mathbb{R}^{\mathcal{Z}}$
be such that:
$$
\forall f \in \mathcal{F}, \; \forall x \in \mathcal{X}, \;
\max_{1 \leqslant k < l \leqslant C}
\left \{ f \left ( x, k \right ) + f \left ( x, l \right ) \right \} = 0.
$$
Let $\Psi$ be a family of mappings from $\mathcal{Y}$
into $\left \{ -1, 0, 1 \right \}$.
For $\gamma \in \mathbb{R}_+^*$,
a subset $s_{\mathcal{Z}^n} =
\left \{ z_i = \left ( x_i, y_i \right ): \;
1 \leqslant i \leqslant n \right \}$ of $\mathcal{Z}$
is said to be {\em $\gamma$-$\Psi$-shattered} by $\mathcal{F}$
if there is a vector $\boldsymbol{\psi}_n
= \left ( \psi^{\left ( i \right )}
\right )_{1 \leqslant i \leqslant n} \in \Psi^n$
satisfying
$\left ( \psi^{\left ( i \right )} \left ( y_i \right ) 
\right )_{1 \leqslant i \leqslant n} = \mathbf{1}_n$,
and a vector
$\mathbf{b}_n =  \left ( b_i \right )_{1 \leqslant i \leqslant n}
\in \mathbb{R}_+^n$
such that, for every vector
$\mathbf{s}_n = \left ( s_i \right )_{1 \leqslant i \leqslant n}
\in \left \{ -1, 1 \right \}^n$,
there is a function $f_{\mathbf{s}_n} \in \mathcal{F}$ satisfying
\begin{equation}
\label{eq:gamma-Psi-dimensions}
\forall i \in \llbracket 1; n \rrbracket, \;\;
\begin{cases}
\text{if } s_i = 1, \;
\max_{\left \{ k: \; \psi^{\left ( i \right )} 
\left ( k \right ) = 1 \right \}}
f_{\mathbf{s}_n} \left ( x_i, k \right ) - b_i \geqslant \gamma \\
\text{if } s_i = -1, \;
\max_{\left \{ l: \; \psi^{\left ( i \right )}
\left ( l \right ) = -1 \right \}}
f_{\mathbf{s}_n} \left ( x_i, l \right ) + b_i \geqslant \gamma
\end{cases}.
\end{equation}
The {\em $\gamma$-$\Psi$-dimension} of $\mathcal{F}$, denoted by
$\gamma\text{-}\Psi\text{-dim} \left ( \mathcal{F} \right )$,
is the maximal cardinality of a subset of
$\mathcal{Z}$ $\gamma$-$\Psi$-shattered by $\mathcal{F}$,
if such maximum exists.
Otherwise, $\mathcal{F}$
is said to have infinite $\gamma$-$\Psi$-dimension.
\end{definition}
In the degenerate case $C=2$, Definition~\ref{def:gamma-Psi-dimensions}
reduces to the definition of the main scale-sensitive combinatorial dimension,
the {\em fat-shattering} or $\gamma$-dimension
$\gamma\text{-dim}$ \citep{KeaSch94}, with a restricted domain for vector
$\mathbf{b}_n$.
This restriction to the positive hyperoctant is of central importance,
for different reasons which will appear gradually.
In the sequel, we apply it whenever a combinatorial dimension
of a function class $\mathcal{F}$ defined as in
Definition~\ref{def:gamma-Psi-dimensions} is considered
(including the fat-shattering dimension).

\begin{definition}[Margin Graph dimension and margin Natarajan dimension]
\label{def:gamma-G-N-dimension}
Let $\mathcal{F}$ be a function class defined as in 
Definition~\ref{def:gamma-Psi-dimensions}
and let $\gamma \in \mathbb{R}_+^*$.
The {\em Graph dimension with margin $\gamma$} of $\mathcal{F}$, denoted by
$\gamma\text{-G-dim} \left ( \mathcal{F} \right )$,
is the $\gamma$-$\Psi$-dimension of $\mathcal{F}$
corresponding to the following choice for $\Psi$:
$$
\Psi_G = \left \{ \left ( \psi_k: y \mapsto 
\mathds{1}_{\left \{ y = k \right \}}
-\mathds{1}_{\left \{ y \neq k \right \}} \right ): \; 
k \in \mathcal{Y} \right \}.
$$
The {\em Natarajan dimension with margin $\gamma$} of $\mathcal{F}$, denoted by
$\gamma\text{-N-dim} \left ( \mathcal{F} \right )$,
is the $\gamma$-$\Psi$-dimension of $\mathcal{F}$
corresponding to the following choice for $\Psi$:
$$
\Psi_N = \left \{ \left ( \psi_{k,l}: y \mapsto
\mathds{1}_{\left \{ y = k \right \}} 
-\mathds{1}_{\left \{ y = l \right \}} \right ): \;
\left \{ k, l \right \} \subset \mathcal{Y} \right \}.
$$
\end{definition}
\begin{remark}
The instantiation of
\eqref{eq:gamma-Psi-dimensions}
associated with the margin Graph dimension is obtained by setting
$\boldsymbol{\psi}_n
= \left ( \psi_{y_i} \right )_{1 \leqslant i \leqslant n}$
so that
$$
\forall i \in \llbracket 1; n \rrbracket, \;
\begin{cases}
\text{if } s_i = 1, \;
f_{\mathbf{s}_n} \left ( x_i, y_i \right ) - b_i \geqslant \gamma \\
\text{if } s_i = -1, \;
\max_{k \neq y_i} f_{\mathbf{s}_n} \left ( x_i, k \right ) + b_i
\geqslant \gamma
\end{cases}.
$$
In the case of the Natarajan dimension with margin $\gamma$, choosing
$\boldsymbol{\psi}_n$ is equivalent to choosing a vector
$\mathbf{c}_n =  \left ( c_i \right )_{1 \leqslant i \leqslant n}
\in \mathcal{Y}^n$ satisfying for every 
$i \in \left \llbracket 1; n \right \rrbracket$, $c_i \neq y_i$.
Then, $\boldsymbol{\psi}_n$ is set equal to
$\left ( \psi_{y_i, c_i} \right )_{1 \leqslant i \leqslant n}$, so that
\eqref{eq:gamma-Psi-dimensions} becomes
$$
\forall i \in \llbracket 1; n \rrbracket, \;
\begin{cases}
\text{if } s_i = 1, \;
f_{\mathbf{s}_n} \left ( x_i, y_i \right ) - b_i \geqslant \gamma \\
\text{if } s_i = -1, \;
f_{\mathbf{s}_n} \left ( x_i, c_i \right ) + b_i \geqslant \gamma
\end{cases}.
$$
\end{remark}

\subsection{Scheme of Derivation of the Guaranteed Risks}
\label{sec:whole-scheme}
For all known instances of Formula~\eqref{eq:basic-supremum-inequality},
the scheme of derivation of function $F_f$ involving 
the families of capacity measures considered in this study is standard.
It corresponds to the directed graph depicted in
Figure~\ref{figure:graph-of-transitions-between-f_i-and-f_f}.

\begin{figure}[h!]
$$
\begin{matrix}
\multicolumn{3}{c}
{\hspace{-4mm} F_i \left ( m, \gamma, \delta,
\text{cap} \left ( \rho_{\mathcal{G}, \gamma} \right ) \right )}
\hspace{-1mm}
& \\
\hspace{3mm} \textcolor{red}{\swarrow} & 
& \hspace{-3mm} \textcolor{red}{\searrow} \\
R_m \left ( \rho_{\mathcal{G}, \gamma} \right ) 
\hspace{-1mm} &
\hspace{-1mm} \textcolor{red}{\xrightarrow[]{\text{chaining}}} \hspace{-1mm} &
\mathcal{N}_p^{\text{int}} \left ( \epsilon,
\rho_{\mathcal{G}, \gamma}, m' \right ) & 
\hspace{-2mm} \leqslant \hspace{-2mm} &
\mathcal{M}_p \left ( \epsilon, \rho_{\mathcal{G}, \gamma},
m' \right ) \hspace{-1mm} &
\hspace{-1mm} \xrightarrow[]{\text{combinatorial result}} 
\hspace{-1mm} & \hspace{-1mm}
\epsilon'\mbox{-dim}\left ( \rho_{\mathcal{G}} \right ) \\
\Bigg\downarrow & &
\textcolor{red}{\Bigg\downarrow} &
\multicolumn{3}{c}{\textcolor{red}{\text{structural results}}}
&
\Bigg\downarrow \\
R_m \left ( \mathcal{G}_0 \right )
& \hspace{-1mm} \xrightarrow[]{\text{chaining}} \hspace{-1mm} &
\mathcal{N}_p^{\text{int}} \left ( \epsilon'', \mathcal{G}_0, m' \right ) &
\hspace{-2mm} \textcolor{red}{\leqslant} \hspace{-2mm} & 
\mathcal{M}_p \left ( \epsilon'', \mathcal{G}_0, m' \right ) \hspace{-1mm} &
\hspace{-1mm} \textcolor{red}{\xrightarrow[]{\text{combinatorial result}}}
\hspace{-1mm} & \hspace{-1mm}
\epsilon'''\mbox{-dim}\left ( \mathcal{G}_0 \right ) \\
\hspace{3mm} \searrow & & 
\Big\downarrow &
\multicolumn{3}{c}{\text{direct computations}}
&
\hspace{-3mm} \textcolor{red}{\swarrow} \\ \\
\multicolumn{7}{c}{F_f \left ( m, C, \gamma, \delta \right )}
\end{matrix}
$$
\caption{Graph of the transitions between the functions $F_i$ and $F_f$.}
\label{figure:graph-of-transitions-between-f_i-and-f_f}
\end{figure}
Here, $\mathcal{G}_0$ stands for a generic class of real-valued functions,
computed by a binary classifier whose nature varies with the context.
The value of $m'$ is either $m$ or $2m$, when the derivation of 
Inequality~\eqref{eq:basic-supremum-inequality}
involves a ghost sample \citep{VapChe71,Pol84}.
When following a path from the source to the target, two types of transitions
are met. A first group, the horizontal arrows,
corresponds to a change of capacity measure.
The standard sequence (from left to right)
consists in the chaining method \citep{Dud67,Tal14},
to connect the Rademacher complexity to covering numbers,
a transition through the corresponding packing numbers,
and then a combinatorial result, to switch
to a combinatorial dimension.
The second group, the layer of vertical arrows,
 is that of the structural results, performing
the transition from the capacity of $\rho_{\mathcal{G}, \gamma}$
to that of $\mathcal{G}_0$
(i.e., from the multi-class case to the bi-class one).
As an example, the paths in red are the ones explored in \citet{Gue17}.

\section{Shortcomings of the State-of-the-Art Structural Results}
\label{sec:connections-between-capacity-measures}

The literature provides us with structural results for all three types
of capacity measures considered. This section highlights their deficiencies
to optimize the confidence interval with respect to $C$ and $\gamma$.

\subsection{State-of-the-Art Structural Results}
\label{sec:state-of-the-art}

The sharpest structural result for the Rademacher complexity
of classes of vector-valued functions is due to \citet{Mau16}.
It is an improvement of the one introduced in \citet{LeiDogBinKlo15}.

\begin{lemma}[Corollary~4 in \citealp{Mau16}]
\label{lemma:corollary-4-in-Mau16}
Let $\mathcal{G}$ be a function class
satisfying Definition~\ref{def:margin-multi-category-classifiers}.
For $n \in \mathbb{N}^*$,
let $\mathcal{F} = \left \{ f_i: \; 1 \leqslant i \leqslant n \right \}$
be a class of real-valued functions on
$\left [-M_{\mathcal{G}}, M_{\mathcal{G}} \right ]^{C}$
which are $L_{\mathcal{F}}$-Lipschitz continuous
with respect to the $\ell_2$-norm. Then
$$
\mathbb{E}_{\boldsymbol{\sigma}_n} \left [
\sup_{g \in \mathcal{G}}
\sum_{i=1}^n \sigma_i f_i \circ g \left ( x_i \right ) \right ]
\leqslant \sqrt{2} L_{\mathcal{F}}
\mathbb{E}_{\boldsymbol{\sigma}_{n,C}} \left [
\sup_{g \in \mathcal{G}}
\sum_{i=1}^n \sum_{k=1}^C \sigma_{i,k} g_k \left ( x_i \right )
\right ],
$$
where $\boldsymbol{\sigma}_{n,C} = \left ( \sigma_{i,k}
\right )_{1 \leqslant i \leqslant n, 1 \leqslant k \leqslant C}$
is a Rademacher random matrix.
\end{lemma}
Let us apply Lemma~\ref{lemma:corollary-4-in-Mau16} by defining
the functions $f_i$ in such a way that
$\forall i \in \llbracket 1; n \rrbracket$,
$f_i \circ g \left ( x_i \right ) = \rho_g \left ( z_i \right )$.
Since they satisfy:
$\forall i \in \llbracket 1; n \rrbracket$,
$\forall \left ( g, g^{\prime} \right ) \in \mathcal{G}^2$,
$\left | \rho_g \left ( z_i \right ) - \rho_{g^{\prime}} \left ( z_i \right )
\right | \leqslant
\frac{1}{2} \left \| g \left ( x_i \right ) - g^{\prime} \left ( x_i \right )
\right \|_2$,
Talagrand's contraction lemma 
\citep[see for instance Lemma~4.2 in][]{MohRosTal12} gives:
\begin{corollary}
\label{corollary:corollary-4-in-Mau16}
Let $\mathcal{G}$ be a function class
satisfying Definition~\ref{def:margin-multi-category-classifiers}
and $\rho_{\mathcal{G}}$ the function class deduced
from $\mathcal{G}$ according to
Definition~\ref{def:class-of-margin-functions}.
For $\gamma \in \left ( 0, 1 \right ]$,
let $\rho_{\mathcal{G}, \gamma}$ be the function class deduced
from $\mathcal{G}$ according to
Definition~\ref{def:class-of-regret-functions}. Then,
$$
\forall n \in \mathbb{N}^*, \; R_n \left ( \rho_{\mathcal{G}, \gamma} \right )
\leqslant R_n \left ( \rho_{\mathcal{G}} \right )
\leqslant \frac{1}{\sqrt{2}n}
\mathbb{E}_{\boldsymbol{\sigma}_{n,C}} \left [
\sup_{g \in \mathcal{G}} 
\sum_{i=1}^n \sum_{k=1}^C \sigma_{i,k} g_k \left ( x_i \right )
\right ].
$$
\end{corollary}
It is noteworthy that under the assumption that
there is no coupling between the outputs of the classifier,
Corollary~\ref{corollary:corollary-4-in-Mau16}
implies a result in \citet{KuzMohSye14}:
$\forall n \in \mathbb{N}^*$, $R_n \left ( \rho_{\mathcal{G}, \gamma} \right )
\leqslant C R_n \left ( \bigcup_{k=1}^C \mathcal{G}_k \right )$,
whose proof does not hold true with $\rho_{\mathcal{G}, \gamma}$
replaced with $\rho_{\mathcal{G}}$.
The counterpart of Corollary~\ref{corollary:corollary-4-in-Mau16}
dealing with covering numbers
is the following structural result.

\begin{lemma}[Lemma~1 in \citealp{Gue17}]
\label{lemma:from-multivariate-to-univariate-L_p}
Let $\mathcal{G}$ be a function class
satisfying Definition~\ref{def:margin-multi-category-classifiers}
and $\rho_{\mathcal{G}}$ the function class deduced
from $\mathcal{G}$ according to
Definition~\ref{def:class-of-margin-functions}.
For $\gamma \in \left ( 0, 1 \right ]$,
let $\rho_{\mathcal{G}, \gamma}$ be the function class deduced
from $\mathcal{G}$ according to
Definition~\ref{def:class-of-regret-functions}.
Then, for $\epsilon \in \mathbb{R}_+^*$,
$n \in \mathbb{N}^*$,
and $\mathbf{z}_n = \left ( \left ( x_i, y_i \right )
\right )_{1 \leqslant i \leqslant n} \in \mathcal{Z}^n$,
$$
\forall p \in \left [ 1, +\infty \right ], \;\;
\mathcal{N}^{\text{int}} \left ( \epsilon,
\rho_{\mathcal{G}, \gamma}, d_{p, \mathbf{z}_n} \right ) \leqslant
\mathcal{N}^{\text{int}} \left ( \epsilon,
\rho_{\mathcal{G}}, d_{p, \mathbf{z}_n} \right )
\leqslant
\prod_{k=1}^C \mathcal{N}^{\text{int}}
\left ( C^{-\frac{1}{p}} \epsilon,
\mathcal{G}_k, d_{p, \mathbf{x}_n} \right ),
$$
where $\mathbf{x}_n = \left ( x_i \right )_{1 \leqslant i \leqslant n}$.
\end{lemma}
The main method available to derive structural results for
the $\gamma$-dimension
\citep[see for instance the proof of Lemma~6.2 in][]{Dua12}
consists in three main steps: upper bounding the dimension of interest
in terms of a metric entropy of the same class,
applying a decomposition (similar to
Lemma~\ref{lemma:from-multivariate-to-univariate-L_p}),
and applying a combinatorial result.
When applied to the class $\rho_{\mathcal{G}, \gamma}$,
it gives birth to the following Lemma.
\begin{lemma}
\label{lemma:from-gamma-dimension-to-gamma-dimensions-old}
Let $\mathcal{G}$ be a function class
satisfying Definition~\ref{def:margin-multi-category-classifiers}
and $\rho_{\mathcal{G}}$ the function class deduced
from $\mathcal{G}$ according to
Definition~\ref{def:class-of-margin-functions}.
For $\gamma \in \left ( 0, 1 \right ]$,
let $\rho_{\mathcal{G}, \gamma}$ be the function class deduced
from $\mathcal{G}$ according to
Definition~\ref{def:class-of-regret-functions}.
Then,
\begin{align}
\forall \epsilon \in \left ( 0, \frac{\gamma}{2} \right ], \;\;
\epsilon\text{-dim} \left ( \rho_{\mathcal{G}, \gamma} \right )
& \leqslant \; \epsilon\text{-dim} \left ( \rho_{\mathcal{G}} \right )
\nonumber \\
\label{eq:from-gamma-dimension-to-gamma-dimensions-old}
& \leqslant \;
320 \log_2 \left ( \frac{24 M_{\mathcal{G}} \sqrt{C}}{\epsilon} \right )
\sum_{k=1}^C \left (
\frac{\epsilon}{96 \sqrt{C}} \right )\text{-dim}
\left ( \mathcal{G}_k \right ).
\end{align}
\end{lemma}

\subsection{Discussion}
\label{sec:state-of-the-art-on-decompositions}
We reviewed the state-of-the-art decomposition results
associated with the three families of capacity measures involved in this study.
None is utterly satisfactory.
Under the assumption that there is no coupling between
the classifier outputs,
the decomposition involving Rademacher complexities
produces a function $F_f$ depending linearly on $C$,
whereas the decomposition involving covering numbers
is known to lead to a sublinear dependence
\citep[see for instance Theorem~3 in][]{MusLauGue19}.
Furthermore, Corollary~\ref{corollary:corollary-4-in-Mau16}
makes no use of the function $\pi_{\gamma}$,
which ``vanishes'' when using Lemma~\ref{lemma:corollary-4-in-Mau16}
since its Lipschitz constant is $1$.
The same holds true for the decompositions involving covering numbers
and fat-shattering dimensions.
When delaying the decomposition at these levels,
the function $\pi_{\gamma}$ is only exploited upstream,
by the chaining formulas or the combinatorial result.
Those limitations raise a question: can a change of combinatorial dimension
(replacing 
$\epsilon\text{-dim} \left ( \rho_{\mathcal{G}} \right )$
with a $\gamma$-$\Psi$-dimension of $\rho_{\mathcal{G}}$)
improve the dependence of function $F_f$ on the basic parameters?
The answers should spring from exploring, in the graph of transitions
(Figure~\ref{figure:graph-of-transitions-between-f_i-and-f_f}),
the paths highlighted in blue in
Figure~\ref{figure:new-path-of-transitions-between-f_i-and-f_f}.

\begin{figure}[h!]
$$
\begin{matrix}
\multicolumn{3}{c}
{\hspace{-4mm} F_i \left ( m, \gamma, \delta,
\text{cap} \left ( \rho_{\mathcal{G}, \gamma} \right ) \right )}
\hspace{-1mm}
& \\
\hspace{3mm} \textcolor{blue}{\swarrow} & 
& \hspace{-3mm} \textcolor{blue}{\searrow} \\
R_m \left ( \rho_{\mathcal{G}, \gamma} \right ) 
\hspace{-1mm} &
\hspace{-1mm} \textcolor{blue}{\xrightarrow[]{\text{chaining}}} \hspace{-1mm} &
\mathcal{N}_p^{\text{int}} \left ( \epsilon,
\rho_{\mathcal{G}, \gamma}, m' \right ) & 
\hspace{-2mm} \textcolor{blue}{\leqslant} \hspace{-2mm} &
\mathcal{M}_p \left ( \epsilon, \rho_{\mathcal{G}, \gamma},
m' \right ) \hspace{-1mm} &
\hspace{-1mm} \textcolor{blue}{\xrightarrow[]{\text{combinatorial result}}}
\hspace{-1mm} & \hspace{-1mm}
\begin{cases}
\epsilon'\mbox{-dim}\left ( \rho_{\mathcal{G}} \right ) \\
\epsilon'\text{-}\Psi\mbox{-dim}\left ( \rho_{\mathcal{G}} \right ) 
\end{cases} \\
\Bigg\downarrow & &
\Bigg\downarrow &
\multicolumn{3}{c}{\textcolor{blue}{\text{structural results}}}
&
\textcolor{blue}{\Bigg\downarrow} \\
R_m \left ( \mathcal{G}_0 \right )
& \hspace{-1mm} \xrightarrow[]{\text{chaining}} \hspace{-1mm} &
\mathcal{N}_p^{\text{int}} \left ( \epsilon'', \mathcal{G}_0, m' \right ) &
\hspace{-2mm} \leqslant \hspace{-2mm} & 
\mathcal{M}_p \left ( \epsilon'', \mathcal{G}_0, m' \right ) \hspace{-1mm} &
\hspace{-1mm} \xrightarrow[]{\text{combinatorial result}}
\hspace{-1mm} & \hspace{-1mm}
\epsilon'''\mbox{-dim}\left ( \mathcal{G}_0 \right ) \\
\hspace{3mm} \searrow & & 
\Big\downarrow &
\multicolumn{3}{c}{\text{direct computations}}
&
\hspace{-3mm} \textcolor{blue}{\swarrow} \\ \\
\multicolumn{7}{c}{F_f \left ( m, C, \gamma, \delta \right )}
\end{matrix}
$$
\caption{Paths from $F_i$ to $F_f$ involving combinatorial dimensions
of $\rho_{\mathcal{G}}$.}
\label{figure:new-path-of-transitions-between-f_i-and-f_f}
\end{figure}
The first answers, of qualitative nature, are exposed in the following section.

\section{Sharper Combinatorial Results with
\texorpdfstring{$\gamma$-$\Psi$-dimensions}{gamma-Psi-dimensions}}
\label{sec:new-results-for-gamma-Psi-dimensions}

We first establish by elementary means that the combinatorial results
involving the fat-shattering dimension of the class $\rho_{\mathcal{G}}$
of margin functions can always be improved by substituting to this dimension
the margin Graph dimension of the same class.

\subsection{Usefulness of the Margin Graph Dimension}
\label{sec:new-results-for-margin-Graph-dimension}

This comparative study benefits from the introduction of 
a new concept of margin operator.

\begin{definition}[Class $\tilde{\rho}_{\mathcal{G}}$
of margin functions]
\label{def:new-class-of-margin-functions}
Let $\mathcal{G}$ be a function class
satisfying Definition~\ref{def:margin-multi-category-classifiers} and
$\rho_{\mathcal{G}}$ the function class deduced from
$\mathcal{G}$ according to Definition~\ref{def:class-of-margin-functions}.
For every $g \in \mathcal{G}$, the {\em margin function}
$\tilde{\rho}_g$ from $\mathcal{Z}$
into $\left [-M_{\mathcal{G}}, M_{\mathcal{G}} \right ]$ is defined by:
$$
\forall \left ( x, k \right ) \in \mathcal{Z}, \;
\tilde{\rho}_g \left ( x, k \right ) 
= -\max_{l \neq k} \rho_g \left ( x, l \right )
= \left ( 2 \mathds{1}_{\left \{ k \in \argmax_{1 \leqslant l \leqslant C}
\rho_g \left ( x, l \right ) \right \}}
- 1 \right ) \max_{1 \leqslant l \leqslant C} \rho_g \left ( x, l \right ).
$$
Then, $\tilde{\rho}_{\mathcal{G}}$ is defined as:
$\tilde{\rho}_{\mathcal{G}}
= \left \{ \tilde{\rho}_g: \; g \in \mathcal{G} \right \}$.
\end{definition}
The class $\tilde{\rho}_{\mathcal{G}, \gamma}$ of squashed margin functions
is defined accordingly as
$\tilde{\rho}_{\mathcal{G}, \gamma} =
\left \{ \pi_{\gamma} \circ \tilde{\rho}_g: \; g \in \mathcal{G} \right \}$.
With these two function classes at hand, the main result
establishing the superiority of our approach over the canonical one
is obtained as a combination of three basic properties of the scale-sensitive
combinatorial dimensions.

\begin{proposition}
\label{prop:gamma-dimension-with-and-without-pi_gamma}
Let $\mathcal{F}$ be a real-valued function class. Then,
\begin{equation}
\label{eq:gamma-dimension-with-and-without-pi_gamma}
\forall \gamma \in \left ( 0, 1 \right ], \;
\forall \epsilon \in \left ( 0, \frac{\gamma}{2} \right ], \;\;
\epsilon\text{-dim} \left ( \pi_{\gamma} \circ \mathcal{F} \right )
\leqslant \epsilon\text{-dim} \left ( \mathcal{F} \right ).
\end{equation}
\end{proposition}

\begin{proposition}
\label{prop:ordering-on-the-dimensions}
Let $\mathcal{F}$ be a function class defined as in 
Definition~\ref{def:gamma-Psi-dimensions}. Then,
\begin{equation}
\label{eq:ordering-on-the-dimensions}
\forall \gamma \in \mathbb{R}_+^*, \;\;
\gamma\text{-N-dim} \left ( \mathcal{F} \right )
\leqslant \gamma\text{-G-dim} \left ( \mathcal{F} \right )
\leqslant \gamma\text{-dim} \left ( \mathcal{F} \right ).
\end{equation}
\end{proposition}

\begin{proposition}
\label{prop:capacities-of-the-two-classes-of-margin-functions}
Let $\mathcal{G}$ be a function class
satisfying Definition~\ref{def:margin-multi-category-classifiers} and
$\rho_{\mathcal{G}}$, $\tilde{\rho}_{\mathcal{G}}$
$\rho_{\mathcal{G}, \gamma}$ and $\tilde{\rho}_{\mathcal{G}, \gamma}$
the corresponding classes of margin functions
and squashed margin functions. Then,
\begin{numcases}
{\forall \gamma \in \mathbb{R}_+^*, }
\label{eq:unification-by-squashing}
\tilde{\rho}_{\mathcal{G}, \gamma} = \rho_{\mathcal{G}, \gamma} \\
\label{eq:identity-of-combinatorial-dimensions}
\gamma\text{-G-dim} \left ( \rho_{\mathcal{G}} \right )
= \gamma\text{-dim} \left ( \tilde{\rho}_{\mathcal{G}} \right ).
\end{numcases}
\end{proposition}
Note that the assumption that the biases $b_i$ of both dimensions
are nonnegative is mandatory
for Equation~\eqref{eq:identity-of-combinatorial-dimensions} to hold true. Let
\begin{equation}
\label{eq:canonical-combinatorial-result}
\mathcal{M}_p \left ( \epsilon, \rho_{\mathcal{G}, \gamma}, n \right )
\leqslant \text{VC}_p \left ( n, \gamma, \epsilon,
\epsilon'\mbox{-dim}\left ( \rho_{\mathcal{G}, \gamma} \right ) \right )
\leqslant \text{VC}_p \left ( n, \gamma, \epsilon,
\epsilon'\mbox{-dim}\left ( \rho_{\mathcal{G}} \right ) \right )
\end{equation}
represent the generic form taken by 
an $L_p$-norm combinatorial result for
$\rho_{\mathcal{G}, \gamma}$
The right-hand side inequality springs from
Inequality~\eqref{eq:gamma-dimension-with-and-without-pi_gamma},
whose application is in agreement with the observation of \citet{Bar98}
that in general, the introduction of a squashing operator does not
improve the bounds on the fat-shattering dimension
(see also the left-hand side inequality of
Formula~\eqref{eq:from-gamma-dimension-to-gamma-dimensions-old}).
Then, applying in sequence
\eqref{eq:unification-by-squashing},
\eqref{eq:gamma-dimension-with-and-without-pi_gamma} and
\eqref{eq:identity-of-combinatorial-dimensions} gives:
\begin{align}
\mathcal{M}_p \left ( \epsilon, \rho_{\mathcal{G}, \gamma}, n \right )
& \leqslant \;
\text{VC}_p \left ( n, \gamma, \epsilon,
\epsilon'\mbox{-dim}\left ( \rho_{\mathcal{G}, \gamma} \right ) \right ) 
\nonumber \\
& = \;
\text{VC}_p \left ( n, \gamma, \epsilon,
\epsilon'\mbox{-dim}\left ( \tilde{\rho}_{\mathcal{G}, \gamma} \right ) 
\right ) \nonumber \\
& \leqslant \;
\text{VC}_p \left ( n, \gamma, \epsilon,
\epsilon'\mbox{-dim}\left ( \tilde{\rho}_{\mathcal{G}} \right ) \right ) 
\nonumber \\
\label{eq:new-combinatorial-result}
& = \;
\text{VC}_p \left ( n, \gamma, \epsilon,
\epsilon'\mbox{-G-dim}\left ( \rho_{\mathcal{G}} \right ) \right ).
\end{align}
The superiority of Inequality~\eqref{eq:new-combinatorial-result} over
Inequality~\eqref{eq:canonical-combinatorial-result}
stems from Inequality~\eqref{eq:ordering-on-the-dimensions}
($\epsilon'\mbox{-G-dim}\left ( \rho_{\mathcal{G}} \right )
\leqslant \epsilon'\mbox{-dim}\left ( \rho_{\mathcal{G}} \right )$).
It is easy to provide examples where the gain can be quantified.
Example~\ref{example:usefulness-of-gamma-G-dim} is of this kind.

\begin{example}
\label{example:usefulness-of-gamma-G-dim}
Let $\mathcal{G}$ be a set of two functions $g^{(1)}$ and $g^{(2)}$
from $\mathcal{X} = \left \{ x \right \}$ into
$\left [-M_{\mathcal{G}}, M_{\mathcal{G}} \right ]^3$ given by
$g^{(1)} \left ( x \right )
= \left ( \frac{3}{4}, \frac{1}{4}, 0 \right )^T$ and
$g^{(2)} \left ( x \right )
= \left ( 0, \frac{1}{2}, \frac{1}{2} \right )^T$.
Then, $\frac{1}{4}\text{-dim} \left ( \rho_{\mathcal{G}} \right ) = 1$
and $\frac{1}{4}\text{-G-dim} \left ( \rho_{\mathcal{G}} \right ) = 0$.
\end{example}
Indeed, $\left ( \rho_{g^{(1)}} \left ( x, k \right )
\right )_{1 \leqslant k \leqslant 3}
= \left ( \frac{1}{4}, -\frac{1}{4}, -\frac{3}{8} \right )^T$,
$\left ( \rho_{g^{(2)}} \left ( x, k \right )
\right )_{1 \leqslant k \leqslant 3}
= \left ( -\frac{1}{4}, 0, 0 \right )^T$, 
$\left ( \tilde{\rho}_{g^{(1)}} \left ( x, k \right )
\right )_{1 \leqslant k \leqslant 3}
= \left ( \frac{1}{4}, -\frac{1}{4}, -\frac{1}{4} \right )^T$ and
$\left ( \tilde{\rho}_{g^{(2)}} \left ( x, k \right )
\right )_{1 \leqslant k \leqslant 3}
= \left ( 0, 0, 0 \right )^T$,
so that
$$
\begin{cases}
\rho_{g^{(1)}} \left ( x, 1 \right ) \geqslant \frac{1}{4} \\
-\rho_{g^{(2)}} \left ( x, 1 \right ) \geqslant \frac{1}{4}
\end{cases},
$$
i.e., the class $\rho_{\mathcal{G}}$ $\frac{1}{4}$-shatters
$\left \{ \left (x, 1 \right ) \right \}$ for $b=0$.
On the contrary, none of the three singletons
$\left \{ \left (x, k \right ) \right \}$ is
$\frac{1}{4}$-G-shattered by $\rho_{\mathcal{G}}$ since
$\max_{k \neq l}
\left \{ \rho_{g^{(1)}} \left ( x, k \right )
+ \rho_{g^{(2)}} \left ( x, l \right ) \right \} = \frac{1}{4}
< 2 \cdot \frac{1}{4}$ (or equivalently none of the three singletons
$\left \{ \left (x, k \right ) \right \}$ is
$\frac{1}{4}$-shattered by $\tilde{\rho}_{\mathcal{G}}$ since
$\max_{1 \leqslant k \leqslant 3}
\left \{ \left | \tilde{\rho}_{g^{(1)}} \left ( x, k \right )
- \tilde{\rho}_{g^{(2)}} \left ( x, k \right ) \right | \right \} = \frac{1}{4}
< 2 \cdot \frac{1}{4}$).

\subsection{From Margin Graph Dimension to Margin Natarajan Dimension}

After highlighting the relationship between
the $\gamma$-dimension and the margin Graph dimension, we do the same for
the two $\gamma$-$\Psi$-dimensions, by stating
a scale-sensitive counterpart of Theorem~10 in \citet{BenCesHauLon95}.

\begin{lemma}
\label{lemma:from-margin-Graph-dimension-to-margin-Natarajan-dimension}
Let $\mathcal{F}$ be a function class defined as in 
Definition~\ref{def:gamma-Psi-dimensions}.
Suppose that $\gamma \in \mathbb{R}_+^*$ is such that
$\gamma\text{-G-dim} \left ( \mathcal{F} \right )$ is finite.
Then,
\begin{equation}
\label{eq:from-margin-Graph-dimension-to-margin-Natarajan-dimension}
\gamma\text{-G-dim} \left ( \mathcal{F} \right ) \leqslant
32 \log_2^2 \left ( e \left ( C-1 \right ) \right )
\gamma\text{-N-dim} \left ( \mathcal{F} \right )^{\alpha \left ( C \right )},
\end{equation}
where
$\alpha \left ( C \right ) = 1 +
\frac{1}{4 \ln \left ( C-1 \right ) +2}$.
\end{lemma}
With Ben-David's theorem in mind, it is noticeable that
Lemma~\ref{lemma:from-margin-Graph-dimension-to-margin-Natarajan-dimension}
holds true for uncountable function classes, the only constraint
of finiteness regarding $\gamma\text{-G-dim} \left ( \mathcal{F} \right )$.
When applied to $\rho_{\mathcal{G}}$,
\eqref{eq:from-margin-Graph-dimension-to-margin-Natarajan-dimension}
is a non trivial bound in the sense that it no longer holds true
with $\gamma\text{-G-dim} \left ( \rho_{\mathcal{G}} \right )$
replaced with $\gamma\text{-dim} \left ( \rho_{\mathcal{G}} \right )$.
Once more, we can resort to
Example~\ref{example:usefulness-of-gamma-G-dim}
to establish this behavior.
Indeed, it exhibits a pair $\left ( \mathcal{G}, \gamma \right )$ for which
$\gamma\text{-dim} \left ( \rho_{\mathcal{G}} \right ) = 1$ but
$\gamma\text{-N-dim} \left ( \rho_{\mathcal{G}} \right )
= \gamma\text{-G-dim} \left ( \rho_{\mathcal{G}} \right ) = 0$.
We conclude the section with a property of the margin Natarajan dimension
that will prove useful to upper bound it.
Its formulation makes use of a standard convention: 
a function class $\mathcal{F}$
is said to $\gamma$-N-shatter a triplet
$\left ( s_{\mathcal{Z}^n}, \mathbf{b}_n, \mathbf{c}_n \right )$
if $\mathcal{F}$ $\gamma$-N-shatters $s_{\mathcal{Z}^n}$ and
$\left ( \mathbf{b}_n, \mathbf{c}_n \right )$ is a witness to this shattering.
The corresponding convention for the margin Graph dimension is also used.
\begin{proposition}
\label{prop:margin-Natarajan-dimension-alternate-definition}
Let $\mathcal{F}$ be a function class defined as in 
Definition~\ref{def:gamma-Psi-dimensions}.
Suppose that for $\gamma \in \mathbb{R}_+^*$, 
the subset $\bar{\mathcal{F}}$ of $\mathcal{F}$ $\gamma$-N-shatters
the triplet $\left ( s_{\mathcal{Z}^n}, \mathbf{b}_n, \mathbf{c}_n \right )$.
Then $\bar{\mathcal{F}}$ also $\gamma$-N-shatters another triplet,
$\left ( s'_{\mathcal{Z}^n}, \mathbf{b}'_n, \mathbf{c}'_n \right )$,
derived from the first one as follows:
$$
\forall i \in \llbracket 1; n \rrbracket, \;
\begin{cases}
\text{if } y_i < c_i, \; y'_i = y_i, \; b'_i = b_i, \; c'_i = c_i \\
\text{if } y_i > c_i, \; y'_i = c_i, \; b'_i = -b_i, \; c'_i = y_i
\end{cases}.
$$
As a consequence, the derivation of an upper bound on
$\gamma\text{-N-dim} \left ( \mathcal{F} \right )$ can make use of
a stronger hypothesis on $\left ( \mathbf{y}_n, \mathbf{c}_n \right )$:
$\forall i \in \llbracket 1; n \rrbracket$, $y_i < c_i$,
provided that the hypothesis of non-negativity of the biases $b_i$ is relaxed.
\end{proposition}

\section{Combinatorial Results}
\label{sec:new-combinatorial-results-for-gamma-Psi-dimensions}

The new results exposed in this section and the following one
are the building blocks
needed to derive upper bounds on the metric entropies of
$\rho_{\mathcal{G}, \gamma}$, for $p \geqslant 2$,
following the blue paths of 
Figure~\ref{figure:new-path-of-transitions-between-f_i-and-f_f}.
To keep the comparison with the literature simple,
we focus on the two most popular options: $p=\infty$ and $p=2$,
but the generalization is straightforward
using the ideas developed in the proof of Theorem~2 in \citet{MusLauGue19}.

In view of the appealing properties of the margin Graph dimension
exposed in Section~\ref{sec:new-results-for-margin-Graph-dimension},
the combinatorial results involving this capacity measure are given first.

\subsection{Margin Graph Dimension}
\begin{lemma}
\label{lemma:extended-Lemma-3.5-in-AloBenCesHau97-gamma-G-dim-optimized}
Let $\mathcal{G}$ be a function class
satisfying Definition~\ref{def:margin-multi-category-classifiers}
and $\rho_{\mathcal{G}}$ the function class deduced
from $\mathcal{G}$ according to
Definition~\ref{def:class-of-margin-functions}.
For $\gamma \in \left ( 0, 1 \right ]$,
let $\rho_{\mathcal{G}, \gamma}$ be the function class deduced
from $\mathcal{G}$ according to
Definition~\ref{def:class-of-regret-functions}.
For $\epsilon \in \left ( 0, M_{\mathcal{G}} \right ]$,
let $d_G \left ( \epsilon \right )
= \epsilon\text{-G-dim} \left ( \rho_{\mathcal{G}} \right )$.
Then for
$\epsilon \in \left ( 0, \gamma \right ]$ and
$n \in \mathbb{N}^*$ such that
$n \geqslant d_G \left ( \frac{\epsilon}{4} \right )$,
\begin{equation}
\label{eq:extended-Lemma-3.5-in-AloBenCesHau97-gamma-G-dim-optimized}
\mathcal{M}_{\infty} \left ( \epsilon, \rho_{\mathcal{G}, \gamma}, 
n \right )
\leqslant \left ( 
\frac{6 \gamma n}{\epsilon}
\right )^{d_G \left ( \frac{\epsilon}{4} \right ) \log_2 \left (
\frac{2 \gamma e n}
{d_G \left ( \frac{\epsilon}{4} \right ) \epsilon} \right )}.
\end{equation}
\end{lemma}
Inequality~\eqref{eq:extended-Lemma-3.5-in-AloBenCesHau97-gamma-G-dim-optimized}
compares with the application to $\rho_{\mathcal{G}, \gamma}$
of the state-of-the-art
$L_{\infty}$-norm combinatorial result:
Lemma~3.5 in \citet{AloBenCesHau97}. The resulting formula is
$$
\mathcal{M}_{\infty} \left ( \epsilon, \rho_{\mathcal{G}, \gamma}, 
n \right )
< 2 \left ( 
\frac{4 \gamma^2 n}{\epsilon^2}
\right )^{\left \lceil d \left ( \frac{\epsilon}{4} \right ) \log_2 \left (
\frac{2 \gamma e n}
{d \left ( \frac{\epsilon}{4} \right ) \epsilon} \right ) \right \rceil},
$$
where $d \left ( \epsilon \right )
= \epsilon\text{-dim} \left ( \rho_{\mathcal{G}} \right )$.
The main observation is that the gain exceeds the one already identified:
the replacement of the fat-shattering dimension with
the margin Graph dimension. The phenomenon appears
especially clearly for $\epsilon = \frac{\gamma}{2}$, the case of practical
interest as will be seen in Section~\ref{sec:guaranteed-risk-L_infty-norm}.
We now turn to the case $p=2$.

\begin{lemma}
\label{lemma:extended-Theorem-1-in-MenVer03}
Let $\mathcal{G}$ be a function class
satisfying Definition~\ref{def:margin-multi-category-classifiers}
and $\rho_{\mathcal{G}}$ the function class deduced
from $\mathcal{G}$ according to
Definition~\ref{def:class-of-margin-functions}.
For $\gamma \in \left ( 0, 1 \right ]$,
let $\rho_{\mathcal{G}, \gamma}$ be the function class deduced
from $\mathcal{G}$ according to
Definition~\ref{def:class-of-regret-functions}.
For $\epsilon \in \left ( 0, M_{\mathcal{G}} \right ]$, let
$d_G \left ( \epsilon \right ) =
\epsilon\text{-G-dim} \left ( \rho_{\mathcal{G}} \right )$.
Then for $\epsilon \in \left ( 0, \gamma \right ]$ and $n \in \mathbb{N}^*$,
\begin{equation}
\label{eq:extended-Theorem-1-in-MenVer03}
\mathcal{M}_2 \left ( \epsilon, \rho_{\mathcal{G}, \gamma}, n \right )
\leqslant
\left ( \frac{5 \gamma}{\epsilon}
\right )^{20 d_G \left ( \frac{\epsilon}{24} \right )}.
\end{equation}
\end{lemma}
Inequality~\eqref{eq:extended-Theorem-1-in-MenVer03}
compares with
the formula obtained with the state-of-the-art
$L_2$-norm combinatorial result,
Theorem~1 in \citet{MenVer03}:
$$
\mathcal{M}_2 \left ( \epsilon, \rho_{\mathcal{G}, \gamma}, n \right )
\leqslant
\left ( \frac{6 \gamma}{\epsilon}
\right )^{20 d \left ( \frac{\epsilon}{48} \right )}.
$$
Here again, the improvement exceeds the sole replacement of 
the fat-shattering dimension with the margin Graph dimension.
However, this statement must be qualified,
since the benefit is smaller, regarding constants only.

\subsection{Margin Natarajan Dimension}
As for the margin Natarajan dimension,
with Lemma~\ref{lemma:from-margin-Graph-dimension-to-margin-Natarajan-dimension}
at hand, 
Lemmas~\ref{lemma:extended-Lemma-3.5-in-AloBenCesHau97-gamma-G-dim-optimized}
and \ref{lemma:extended-Theorem-1-in-MenVer03} also provide us with
combinatorial results involving this capacity measure.
However, sharper bounds should spring from following the direct path, i.e.,
working directly  with this latter dimension
(without involving the margin Graph dimension).
We now state the corresponding combinatorial results 
(for $p=\infty$ then $p=2$) and perform the comparison.

\begin{lemma}
\label{lemma:extended-Lemma-3.5-in-AloBenCesHau97-gamma-N-dim}
Let $\mathcal{G}$ be a function class
satisfying Definition~\ref{def:margin-multi-category-classifiers}
and $\rho_{\mathcal{G}}$ the function class deduced
from $\mathcal{G}$ according to
Definition~\ref{def:class-of-margin-functions}.
For $\gamma \in \left ( 0, 1 \right ]$,
let $\rho_{\mathcal{G}, \gamma}$ be the function class deduced
from $\mathcal{G}$ according to
Definition~\ref{def:class-of-regret-functions}.
For $\epsilon \in \left ( 0, M_{\mathcal{G}} \right ]$,
let $d_N \left ( \epsilon \right )
= \epsilon\text{-N-dim} \left ( \rho_{\mathcal{G}} \right )$.
Then for
$\epsilon \in \left ( 0, \gamma \right ]$ and
$n \in \mathbb{N}^*$ such that 
$n \geqslant d_N \left ( \frac{\epsilon}{4} \right )$,
\begin{equation}
\label{eq:extended-Lemma-3.5-in-AloBenCesHau97-gamma-N-dim}
\mathcal{M}_{\infty} \left ( \epsilon, \rho_{\mathcal{G}, \gamma}, 
n \right ) \leqslant
\left ( \frac{6 \gamma \sqrt{C-1} n}{\epsilon} 
\right )^{d_N \left ( \frac{\epsilon}{4} \right )
\log_2 \left ( \frac{2 \gamma \left ( C-1 \right )e n}
{d_N \left ( \frac{\epsilon}{4} \right ) \epsilon} \right )}.
\end{equation}
\end{lemma}

\begin{lemma}
\label{lemma:extended-Theorem-1-in-MenVer03-gamma-N-dim}
Let $\mathcal{G}$ be a function class
satisfying Definition~\ref{def:margin-multi-category-classifiers}
and $\rho_{\mathcal{G}}$ the function class deduced
from $\mathcal{G}$ according to
Definition~\ref{def:class-of-margin-functions}.
For $\gamma \in \left ( 0, 1 \right ]$,
let $\rho_{\mathcal{G}, \gamma}$ be the function class deduced
from $\mathcal{G}$ according to
Definition~\ref{def:class-of-regret-functions}.
For $\epsilon \in \left ( 0, M_{\mathcal{G}} \right ]$, let
$d_N \left ( \epsilon \right ) =
\epsilon\text{-N-dim} \left ( \rho_{\mathcal{G}} \right )$.
Then for $\epsilon \in \left ( 0, \gamma \right ]$ and $n \in \mathbb{N}^*$,
\begin{equation}
\label{eq:extended-Theorem-1-in-MenVer03-gamma-N-dim}
\mathcal{M}_2 \left ( \epsilon, \rho_{\mathcal{G}, \gamma}, n \right )
\leqslant
\left ( \left ( C - 1 \right ) \left ( \frac{4 \gamma}{\epsilon} 
\right )^5 \right )^{
\frac{3}{2} \log_2 \left ( 2 \left ( \frac{14 \gamma}{\epsilon} \right )^2
\left ( C-1 \right ) \right )
d_N \left ( \frac{\epsilon}{28} \right )}.
\end{equation}
\end{lemma}
As expected,
as close to $1$ as $\alpha \left ( C \right )$ may be,
Inequalities~\eqref{eq:extended-Lemma-3.5-in-AloBenCesHau97-gamma-N-dim} and
\eqref{eq:extended-Theorem-1-in-MenVer03-gamma-N-dim} are better
than the bounds obtained by substitution of
\eqref{eq:from-margin-Graph-dimension-to-margin-Natarajan-dimension}
in the formulas involving the margin Graph dimension:
\eqref{eq:extended-Lemma-3.5-in-AloBenCesHau97-gamma-G-dim-optimized} and
\eqref{eq:extended-Theorem-1-in-MenVer03}, respectively.
Precisely, in both cases, the dependence on $n$ is unchanged,
while the dependences on $C$ and $\epsilon$ are slightly improved.
A quantitative characterization of the gain
requires to make assumptions on the dependence
of the margin Natarajan dimension on $C$ and $\epsilon$.
This is done in Section~\ref{sec:bounds-on-metric-entropies}
(see Hypothesis~\ref{hypothesis:restriction-gamma-N-dim}).

\section{Structural Results}
\label{sec:new-structural-results-for-gamma-Psi-dimensions}

We have seen that
the combination of Proposition~\ref{prop:ordering-on-the-dimensions} and
Lemma~\ref{lemma:from-gamma-dimension-to-gamma-dimensions-old}
provides us with a structural result of reference for
$\gamma\text{-G-dim} \left ( \rho_{\mathcal{G}} \right )$.
The proof of the Lemma makes use of the $L_2$-norm.
However, a significant improvement stems from choosing $p$ as a function of $C$.

\subsection{Margin Graph Dimension}

\begin{lemma}
\label{lemma:from-gamma-dimension-to-gamma-dimensions}
Let $\mathcal{G}$ be a function class
satisfying Definition~\ref{def:margin-multi-category-classifiers}
and $\rho_{\mathcal{G}}$ the function class deduced
from $\mathcal{G}$ according to
Definition~\ref{def:class-of-margin-functions}.
Then for every $\gamma \in \left ( 0, M_{\mathcal{G}} \right ]$,
\begin{equation}
\label{eq:from-gamma-dimension-to-gamma-dimensions}
\gamma\text{-G-dim} \left ( \rho_{\mathcal{G}} \right )
\leqslant
10 K_C \log_2 \left ( 2C \right )
\log_2 \left ( \frac{48 M_{\mathcal{G}}
\log_2^{\frac{1}{7}} \left ( 2C \right )}{\gamma} \right )
\sum_{k=1}^C \left (
\frac{\gamma}{144 \log_2 \left ( 2C \right )} \right )\text{-dim}
\left ( \mathcal{G}_k \right ),
\end{equation}
where $K_C = \min \left \{ 4 \left ( \frac{C}{C-2} \right )^2, 16 \right \}$.
\end{lemma}
The obvious benefit is an improved dependence on $C$.

\subsection{Margin Natarajan Dimension}
We now establish two structural results for
$\gamma\text{-N-dim} \left ( \rho_{\mathcal{G}} \right )$.

\begin{lemma}
\label{lemma:from-gamma-N-dimension-to-gamma-dimension}
Let $\mathcal{G}$ be a function class
satisfying Definition~\ref{def:margin-multi-category-classifiers}
and $\rho_{\mathcal{G}}$ the function class deduced
from $\mathcal{G}$ according to
Definition~\ref{def:class-of-margin-functions}.
Then
\begin{equation}
\label{eq:from-gamma-N-dimension-to-gamma-dimension}
\forall \gamma \in \left ( 0, M_{\mathcal{G}} \right ], \;\;
\gamma\text{-N-dim} \left ( \rho_{\mathcal{G}} \right )
\leqslant 320 \left ( C-1 \right )
\log_2 \left ( \frac{24 \sqrt{2} M_{\mathcal{G}}}{\gamma} \right )
\sum_{k=1}^C
\left ( \frac{\gamma}{96 \sqrt{2}} \right )\text{-dim}
\left ( \mathcal{G}_k \right ).
\end{equation}
Suppose further that there exists a Banach space such that
all the classes $\mathcal{G}_k$ are
subsets of a ball $\mathcal{G}_0$ about the origin. Then,
\begin{equation}
\label{eq:from-gamma-N-dimension-to-gamma-dimension-Banach}
\forall \gamma \in \left ( 0, M_{\mathcal{G}} \right ], \;\;
\gamma\text{-N-dim} \left ( \rho_{\mathcal{G}} \right )
\leqslant \binom{C}{2} \cdot \gamma\text{-dim} \left ( \mathcal{G}_0 \right ).
\end{equation}
\end{lemma}
It is noticeable that the hypothesis of the existence of the ball 
$\mathcal{G}_0$, which is satisfied by classifiers of reference
such as the multi-layer perceptrons (MLPs) \citep{AntBar99}
with linear output units
and the $C$-category support vector machines (SVMs) \citep{DogGlaIge16},
is enough to sharpen significantly the bound.
Indeed, the major advantage of working with
$\gamma\text{-N-dim} \left ( \rho_{\mathcal{G}} \right )$
is the possibility to take benefit
from the specificities of the classifier.
Formula~\eqref{eq:from-gamma-N-dimension-to-gamma-dimension-Banach}
exploits an algebraic property of the function class of interest.
We now illustrate the gain resulting from taking into account
the coupling between the outputs, by extending the study of the case of
the $C$-category SVMs.
We base their definition on that
of reproducing kernel Hilbert space (RKHS)
of $\mathbb{R}^C$-valued functions \citep{Wah92}.

\begin{definition}[RKHS $\mathbf{H}_{\kappa, C}$]
\label{def:class-H_bar}
Let $\kappa$ be a real-valued positive type function on $\mathcal{X}^2$ and let
$\left ( \mathbf{H}_{\kappa}, \ps{\cdot, \cdot}_{\mathbf{H}_{\kappa}} \right)$
be the corresponding RKHS.
Let $\tilde{\kappa}$ be the real-valued positive type function on
$\mathcal{Z}^2$ deduced from $\kappa$ as follows:
$\forall \left ( z, z' \right ) \in \mathcal{Z}^2$,
$\tilde{\kappa} \left ( z, z' \right )
= \delta_{y,y'} \kappa \left ( x, x' \right )$,
where $\delta$ is the Kronecker delta.
For every $z \in \mathcal{Z}$,
let us define the $\mathbb{R}^C$-valued function
$\tilde{\kappa}_z^{(C)}$ on $\mathcal{X}$
by the formula
\begin{equation}
\label{eq:multivariate-kernel}
\tilde{\kappa}_z^{(C)} \left ( \cdot \right ) =
\left ( \tilde{\kappa} \left ( z, \left ( \cdot, k \right ) \right )
\right )_{1 \leqslant k \leqslant C}.
\end{equation}
The RKHS of $\mathbb{R}^C$-valued functions
at the basis of a $C$-category SVM with kernel $\kappa$,
$\left ( \mathbf{H}_{\kappa, C},
\ps{\cdot, \cdot}_{\mathbf{H}_{\kappa, C}} \right)$,
consists of the linear manifold of all finite
linear combinations of functions of the form
\eqref{eq:multivariate-kernel}
and its closure with respect to the inner product:
$\forall \left ( z, z' \right ) \in \mathcal{Z}^2, \;\;
\ps{\tilde{\kappa}_z^{(C)},
\tilde{\kappa}_{z'}^{(C)}}_{\mathbf{H}_{\kappa, C}} =
\tilde{\kappa} \left ( z, z' \right )$.
\end{definition}
With Definition~\ref{def:class-H_bar} at hand, the specification of
the function class at the basis
of a $C$-category SVM rests on
the condition controlling the capacity through
a coupling between the outputs.
We consider the standard one, used for instance by \citet{LeiDogBinKlo15}.

\begin{definition}[Function class $\mathcal{H}_{\Lambda}$]
\label{def:function-class-of-a-C-category-SVM}
Let $\kappa$ be a real-valued positive type function on $\mathcal{X}^2$
and let $\Lambda \in \mathbb{R}_+^*$.
Let $\left ( \mathbf{H}_{\kappa, C},
\ps{\cdot, \cdot}_{\mathbf{H}_{\kappa, C}} \right)$
be the RKHS of $\mathbb{R}^C$-valued functions spanned by $\kappa$
according to Definition~\ref{def:class-H_bar}.
Then the function class $\mathcal{H}_{\Lambda}$ associated with
the $C$-category SVM parameterized by $\left ( \kappa, \Lambda \right )$ is:
$\mathcal{H}_{\Lambda} = \left \{
h = \left ( h_k \right )_{1 \leqslant k \leqslant C} 
\in \mathbf{H}_{\kappa, C}: \;
\sum_{k=1}^C h_k = 0_{\mathbf{H}_{\kappa}} \text{ and }
\left \| h \right \|_{\mathbf{H}_{\kappa, C}} \leqslant \Lambda \right \}$.
\end{definition}
Then, Lemma~\ref{lemma:gamma-N-dimension-of-M-SVMs} provides a sharper
bound on $\gamma\text{-N-dim} \left ( \rho_{\mathcal{H}_{\Lambda}} \right )$
than Lemma~\ref{lemma:from-gamma-N-dimension-to-gamma-dimension}.

\begin{lemma}
\label{lemma:gamma-N-dimension-of-M-SVMs}
For $\Lambda \in \mathbb{R}_+^*$,
let $\mathcal{H}_{\Lambda}$ be a function class satisfying
Definition~\ref{def:function-class-of-a-C-category-SVM}.
Suppose that for every $x \in \mathcal{X}$, $\kappa_x$ belongs to
the closed ball of radius $\Lambda_{\mathcal{X}}$ about the origin
in $\mathbf{H}_{\kappa}$. Then,
\begin{equation}
\label{eq:gamma-N-dimension-of-M-SVMs}
\forall \gamma \in \left ( 0, \Lambda \Lambda_{\mathcal{X}} \right ], \;
\gamma\text{-N-dim} \left ( \rho_{\mathcal{H}_{\Lambda}} \right )
\leqslant
C \left ( \frac{\Lambda \Lambda_{\mathcal{X}}}{2 \gamma} \right )^2.
\end{equation}
\end{lemma}
Loosely speaking,
Furmula~\eqref{eq:gamma-N-dimension-of-M-SVMs} tells us that
$\gamma\text{-N-dim} \left ( \rho_{\mathcal{H}_{\Lambda}} \right )$
can be upper bounded by $C$ times the standard upper bound 
on the $\gamma$-dimension of an SVM \citep[Theorem~4.6 in][]{BarSha99}.
Thus, taking into account the coupling between the outputs has turned
the quadratic dependence of 
Formula~\eqref{eq:from-gamma-N-dimension-to-gamma-dimension-Banach}
into a linear one.

\section{Guaranteed Risks}
\label{sec:bounds-on-metric-entropy}

Given the scheme adopted for the derivation of guaranteed risks,
in most of the options considered (paths in the graph of
Figure~\ref{figure:graph-of-transitions-between-f_i-and-f_f}),
the central formula is the upper bound on the metric entropy.
With the combinatorial and structural results of the preceding section at hand,
two new formulas are required to compare the functions $F_f$
resulting from the use of the margin Natarajan dimension
to those resulting from following a different path in the graph of
Figure~\ref{figure:graph-of-transitions-between-f_i-and-f_f}.
The first one is an upper bound on
the $\gamma$-dimensions of the classes $\mathcal{G}_k$
as a function of the scale parameter $\epsilon$.
The second one is an upper bound on the margin Natarajan dimension
of $\rho_{\mathcal{G}}$ as a function of $C$ and $\epsilon$.

\subsection{Bounds on the Metric Entropies}
\label{sec:bounds-on-metric-entropies}
For the first formula, we use the standard hypothesis: that of polynomial
$\gamma$-dimensions \citep{VaaWel96,Men03}. We have already seen
that it is satisfied by SVMs. This is also the case for MLPs with linear
output units \citep[see for instance Theorem~14.19 in][]{AntBar99}.
The second formula is a generic one.
It is designed to incorporate the hypothesis of polynomial $\gamma$-dimensions
in a decomposition result taking benefit from some knowledge on 
the function class $\mathcal{G}$ and a coupling between outputs.
It is thus inspired by the structural results of the previous section,
precisely 
Inequalities~\eqref{eq:from-gamma-N-dimension-to-gamma-dimension-Banach}
and \eqref{eq:gamma-N-dimension-of-M-SVMs}.

\begin{hypothesis}
\label{hypothesis:restriction-gamma-N-dim}
We consider function classes $\mathcal{G}$
satisfying Definition~\ref{def:margin-multi-category-classifiers}
plus the fact that there exists a quadruplet
$\left (  d_{\mathcal{G}, C}, d_{\mathcal{G}, \gamma}, 
K_{\mathcal{G}_0}, K_{\rho_{\mathcal{G}}} \right )
\in \left ( 0, 2 \right ] \times \left ( \mathbb{R}_+^* \right )^3$ such that
\begin{numcases}
{\forall \epsilon \in \left ( 0, M_{\mathcal{G}} \right ], }
\label{eq:gamma-dim-bound}
\max_{1 \leqslant k \leqslant C}
\epsilon\text{-dim} \left ( \mathcal{G}_k \right )
\leqslant K_{\mathcal{G}_0} \epsilon^{- d_{\mathcal{G}, \gamma}} \\
\label{eq:gamma-N-dim-bound-synthesis}
\epsilon\text{-N-dim} \left ( \rho_{\mathcal{G}} \right )
\leqslant K_{\rho_{\mathcal{G}}}
C^{d_{\mathcal{G}, C}} \epsilon^{- d_{\mathcal{G}, \gamma}}.
\end{numcases}
\end{hypothesis}
Under Hypothesis~\ref{hypothesis:restriction-gamma-N-dim},
the combinatorial results dedicated to the margin Natarajan dimension
(Lemmas~\ref{lemma:extended-Lemma-3.5-in-AloBenCesHau97-gamma-N-dim} and
\ref{lemma:extended-Theorem-1-in-MenVer03-gamma-N-dim})
give birth to the following bounds on the metric entropies.

\begin{theorem}
Let $\mathcal{G}$ be a function class
satisfying Hypothesis~\ref{hypothesis:restriction-gamma-N-dim}.
For $\gamma \in \left ( 0, 1 \right ]$,
let $\rho_{\mathcal{G}, \gamma}$ be the function class deduced
from $\mathcal{G}$ according to
Definition~\ref{def:class-of-regret-functions}.
For $\epsilon \in \left ( 0, \gamma \right ]$ and
$n \in \mathbb{N}^*$ such that
$n \geqslant K_{\rho_{\mathcal{G}}}
C^{d_{\mathcal{G}, C}} 
\left ( \frac{4}{\epsilon} \right )^{d_{\mathcal{G}, \gamma}}$,
\begin{equation}
\label{eq:metric-entropy-uniform-convergnece-norm}
\log_2 \left ( \mathcal{N}_{\infty}^{\text{int}} \left ( \epsilon,
\rho_{\mathcal{G}, \gamma}, n \right ) \right )
\leqslant
K_{\rho_{\mathcal{G}}} C^{d_{\mathcal{G}, C}}
\log_2^2 \left ( \frac{6 \gamma \left ( C-1 \right ) n}{\epsilon} 
\right ) \left ( \frac{4}{\epsilon} \right )^{d_{\mathcal{G}, \gamma}}.
\end{equation}
For
$\epsilon \in \left ( 0, \gamma \right ]$ and
$n \in \mathbb{N}^*$,
\begin{equation}
\label{eq:metric-entropy-L_2-norm}
\log_2 \left ( \mathcal{N}_2^{\text{int}} \left ( \epsilon,
\rho_{\mathcal{G}, \gamma}, n \right ) \right )
\leqslant
\frac{3}{2} K_{\rho_{\mathcal{G}}} C^{d_{\mathcal{G}, C}}
\log_2^2 \left ( \left ( C - 1 \right ) \left ( \frac{4 \gamma}{\epsilon} 
\right )^5 \right )
\left ( \frac{28}{\epsilon}
\right )^{d_{\mathcal{G}, \gamma}}.
\end{equation}
\end{theorem}
As expected, the main difference between those two bounds 
regarding the dependence on the basic parameters is that the second one
is dimension free (does not depend on the number $n$ of points).
They are significantly better than the bounds obtained
with the margin Graph dimension
(and thus the fat-shattering dimension of $\rho_{\mathcal{G}}$).
In the latter sequence of transitions, the limiting factor is obviously
the structural result:
Lemma~\ref{lemma:from-gamma-dimension-to-gamma-dimensions}.
The possibility to sharpen it for specific classifiers remains
an open question.
A partial conclusion emerges:
as soon as some features of the classifier of interest can be exploited,
then the best guaranteed risks involving 
a scale-sensitive combinatorial dimension of $\rho_{\mathcal{G}}$
(associated with a blue path
of Figure~\ref{figure:new-path-of-transitions-between-f_i-and-f_f}),
are obtained with the margin Natarajan dimension
(rather than the fat-shattering dimension or the margin Graph dimension).
With this observation at hand, the last promising alternative is
the use of the structural result involving covering numbers
(Lemma~\ref{lemma:from-multivariate-to-univariate-L_p}),
in conjunction with the state-of-the-art combinatorial results
(applied to the classes $\mathcal{G}_k$), and finally
Formula~\eqref{eq:gamma-dim-bound}.
This corresponds to the red paths in
Figure~\ref{figure:graph-of-transitions-between-f_i-and-f_f}.
The next section is devoted to the comparison.
To make it more concrete, we use as touchstones the functions $F_i$
corresponding to the state-of-the art basic supremum inequalities
associated with the two $L_p$-norms favored in this study.

\subsection{Comparative Study - Uniform Convergence Norm}
\label{sec:guaranteed-risk-L_infty-norm}
To the best of our knowledge, the sharpest instance
of Inequality~\eqref{eq:basic-supremum-inequality} 
involving the $L_{\infty}$-norm is
Formula~(20) in \citet{Gue17}.
It is a multi-class extension of Lemma~4 in \citet{Bar98},
with the first symmetrization being derived from the basic lemma
of Section~4.5.1 in \citet{Vap98}.
This bound corresponds to the following choices:
$L_* = L$ and
$\phi_{\gamma} \left ( t \right ) 
= \mathds{1}_{\left \{ t < \gamma \right \}}$, and produces:
$$
F_i \left ( m, \gamma, \delta, 
\text{cap} \left ( \rho_{\mathcal{G}, \gamma} \right ) \right )
= \sqrt{ \frac{2}{m} \left ( \ln \left ( 
\mathcal{N}_{\infty}^{\text{int}} \left ( \frac{\gamma}{2},
\rho_{\mathcal{G}, \gamma},  2m \right ) \right )
+ \ln \left ( \frac{2}{\delta} \right )
\right )} + \frac{1}{m}.
$$
In that case, \eqref{eq:metric-entropy-uniform-convergnece-norm} becomes for
$m \geqslant \frac{1}{2} K_{\rho_{\mathcal{G}}}
C^{d_{\mathcal{G}, C}} 
\left ( \frac{8}{\gamma} \right )^{d_{\mathcal{G}, \gamma}}$:
\begin{equation}
\label{eq:New-bound-on-metric-entropy-infty-norm}
\log_2 \left ( \mathcal{N}_{\infty}^{\text{int}} \left ( \frac{\gamma}{2},
\rho_{\mathcal{G}, \gamma}, 2m \right ) \right )
\leqslant K_{\rho_{\mathcal{G}}} C^{d_{\mathcal{G}, C}}
\log_2^2 \left ( 24 \left ( C-1 \right ) m \right )
\left ( \frac{8}{\gamma} \right )^{d_{\mathcal{G}, \gamma}}.
\end{equation}
On the other hand, applying the decomposition with covering numbers
(Lemma~\ref{lemma:from-multivariate-to-univariate-L_p})
and Lemma~3.5 in \citet{AloBenCesHau97} as combinatorial result yields for
$m \geqslant \frac{1}{2}
K_{\mathcal{G}_0} \left ( \frac{8}{\gamma} \right )^{d_{\mathcal{G}, \gamma}}$:
\begin{equation}
\label{eq:Old-bound-on-metric-entropy-infty-norm}
\log_2 \left ( \mathcal{N}_{\infty}^{\text{int}} \left ( \frac{\gamma}{2},
\rho_{\mathcal{G}, \gamma}, 2m \right ) \right )
\leqslant C \left \{ 1 + K_{\mathcal{G}_0} \log_2^2 
\left ( \frac{128 M_{\mathcal{G}}^2 m}{\gamma^2} \right ) 
\left ( \frac{8}{\gamma} \right )^{d_{\mathcal{G}, \gamma}} \right \}.
\end{equation}
Thus, the functions $F_f$ associated with
Inequalities~\eqref{eq:New-bound-on-metric-entropy-infty-norm}
and \eqref{eq:Old-bound-on-metric-entropy-infty-norm}
exhibit the same dependence on $m$
(a $O\left ( \frac{\ln \left ( m \right )}{\sqrt{m}} \right )$).
As for the dependence on $\gamma$, the new formula induces
a gain of a factor $\ln \left ( \gamma^{-1} \right )$.
The dependence on $C$ will also be improved for
$d_{\mathcal{G}, C} < 1$ (the computations could take into account
a strong coupling between the outputs).

\subsection{Comparative Study - $L_2$-norm}
Turning to the case of the $L_2$-norm,
the best instance of Inequality~\eqref{eq:basic-supremum-inequality}
involves the Rademacher complexity as capacity measure.
It is a partial result in the proof of Theorem~8.1 in \citet{MohRosTal12}
(with $\rho_{\mathcal{G}}$ replaced with $\rho_{\mathcal{G}, \gamma}$).
Its margin loss function is given by
$\phi_{\gamma} \left ( t \right ) 
= \mathds{1}_{\left \{ t \leqslant 0 \right \}}
+ \left ( 1 - \frac{t}{\gamma} \right )
\mathds{1}_{\left \{ t \in \left ( 0, \gamma \right ] \right \}}$
(parameterized truncated hinge loss)
and $L_* = L_{\gamma}$. The analytical expression of function $F_i$ is:
\begin{equation}
\label{eq:basic-supremum-bound-L_2-norm}
F_i \left ( m, \gamma, \delta, 
\text{cap} \left ( \rho_{\mathcal{G}, \gamma} \right ) \right )
= \frac{2}{\gamma}
R_m \left ( \rho_{\mathcal{G}, \gamma} \right )
+ \sqrt{\frac{\ln \left ( \frac{1}{\delta} \right ) }{2m}}.
\end{equation}
In accordance with the graph of the transitions
(Figures~\ref{figure:graph-of-transitions-between-f_i-and-f_f}~and
\ref{figure:new-path-of-transitions-between-f_i-and-f_f}),
the Rademacher complexity is upper bounded as a function of the metric entropy
by means of Dudley's chaining method.
We use the following formula, whose degrees of freedom can be exploited
to optimize the dependence on the basic parameters.

\begin{theorem}[Theorem~9 in \citealp{Gue17}]
\label{theorem:Dudley's-metric-entropy-bound}
Let $\mathcal{F}$ be a class of bounded
real-valued functions on $\mathcal{T}$.
For $n \in \mathbb{N}^*$, let
$\mathbf{t}_n \in \mathcal{T}^n$
and let $\text{diam} \left ( \mathcal{F} \right )$
be the diameter of $\mathcal{F}$ with respect to
the pseudo-metric $d_{2, \mathbf{t}_n}$.
Let $h$ be a positive and decreasing function on $\mathbb{N}$ such that
$h \left ( 0 \right ) \geqslant \text{diam} \left ( \mathcal{F} \right )$.
Then for $N \in \mathbb{N}^*$,
\begin{equation}
\label{eq:generalized-Dudley-metric-entropy-bound}
\hat{R}_n \left ( \mathcal{F} \right ) \leqslant h \left ( N \right )
+ 2 \sum_{j=1}^N \left ( h \left ( j \right ) + h \left ( j-1 \right ) \right )
\sqrt{\frac{\ln \left (
\mathcal{N}^{\text{int}} \left ( h \left ( j \right ), \mathcal{F},
d_{2, \mathbf{t}_n} \right ) \right )}{n}}.
\end{equation}
\end{theorem}
To upper bound the metric entropy above, the formula of reference (obtained by
combining the structural result dedicated to covering numbers
with the combinatorial result of \citet{MenVer03}) is:
\begin{equation}
\label{eq:Old-bound-on-metric-entropy}
\forall \epsilon \in \left ( 0, \gamma \right ], \;\;
\log_2 \left ( \mathcal{N}_2^{\text{int}} \left ( \epsilon,
\rho_{\mathcal{G}, \gamma}, n \right ) \right )
\leqslant
20 K_{\mathcal{G}_0} C
\log_2 \left ( \frac{12 M_{\mathcal{G}} \sqrt{C}}{\epsilon} \right )
\left ( \frac{48 \sqrt{C}}{\epsilon} \right )^{d_{\mathcal{G}, \gamma}}.
\end{equation}
Using instead our new bound,
Inequality \eqref{eq:metric-entropy-L_2-norm},
a substitution into
\eqref{eq:generalized-Dudley-metric-entropy-bound} gives:
\begin{equation}
\label{eq:New-bound-on-Rademacher-complexity}
R_m \left ( \rho_{\mathcal{G}, \gamma} \right )
\leqslant
h \left ( N \right )
+ 4 \sqrt{\frac{F_1 \left ( C \right )}{m}}
\sum_{j \in \mathcal{J}}
\frac{h \left ( j \right ) + h \left ( j-1 \right )}
{h \left ( j
\right )^{\frac{d_{\mathcal{G}, \gamma}}{2}}}
\ln \left ( \left ( C-1 \right ) 
\left ( \frac{4 \gamma}{h \left ( j \right )} \right )^5 \right )
\end{equation}
where
\begin{equation}
\label{eq:function-F_1}
F_1 \left ( C \right ) =
{28}^{d_{\mathcal{G}, \gamma}} K_{\rho_{\mathcal{G}}}
C^{d_{\mathcal{G}, C}},
\end{equation}
with $\mathcal{J} = \left \{ j \in \llbracket 1; N \rrbracket:
\; h \left ( j \right )
\leqslant \gamma \right \}$.
With the last formula at hand, the derivation of the confidence interval
amounts to studying the phase transitions highlighted
by Theorem~18 in \citet{Men03}.

\begin{theorem}
\label{theorem:dependence-on-m-C-gamma-L_2-norm}
Let $\mathcal{G}$ be a function class
satisfying Hypothesis~\ref{hypothesis:restriction-gamma-N-dim}.
For $\gamma \in \left ( 0, 1 \right ]$,
let $\rho_{\mathcal{G}, \gamma}$ be the function class deduced
from $\mathcal{G}$ according to
Definition~\ref{def:class-of-regret-functions}.

\noindent If $d_{\mathcal{G}, \gamma} \in \left ( 0, 2 \right )$, then

$$
R_m \left ( \rho_{\mathcal{G}, \gamma} \right )
\leqslant
4 \left ( 1 + 2^{\frac{2}{2 - d_{\mathcal{G}, \gamma}}} \right )
\sqrt{\frac{F_1 \left ( C \right )}{m}} F_2 \left ( C \right )
\gamma^{1 - \frac{d_{\mathcal{G}, \gamma}}{2}},
$$
where
$F_1 \left ( C \right )$ is given by Equation~\eqref{eq:function-F_1} and
$F_2 \left ( C \right ) = \ln \left ( \left ( C-1 \right ) 4^5 \right ) +
10 \frac{1 +\ln \left ( 2 \right )}{2 - d_{\mathcal{G}, \gamma}}$.

\noindent If $d_{\mathcal{G}, \gamma} = 2$, then
$$
R_m \left ( \rho_{\mathcal{G}, \gamma} \right )
\leqslant \gamma \frac{\log_2 \left ( m \right )}{\sqrt{m}}
+ 12 \sqrt{\frac{F_1 \left ( C \right )}{m}}
\left \lceil \log_2 \left ( 
\frac{\sqrt{m}}{\log_2 \left ( m \right )} \right ) \right \rceil
\left \{
\ln \left ( \left ( C-1 \right ) 4^5 \right ) 
+ \frac{5}{2} \ln \left (
4 \frac{\sqrt{m}}{\log_2 \left ( m \right )} \right ) \right \}.
$$
At last, if $d_{\mathcal{G}, \gamma} > 2$, then
$$
R_m \left ( \rho_{\mathcal{G}, \gamma} \right )
\leqslant
\gamma \left ( \frac{\log_2 \left ( m \right )}{m} \right )^{\frac{1}
{d_{\mathcal{G}, \gamma}}}
$$
$$
\times
\left [ 1 + 8 \left ( 1 + 2^{\frac{2}{d_{\mathcal{G}, \gamma}-2}} \right ) 
\left ( \frac{1}{\gamma} \right )^{\frac{d_{\mathcal{G}, \gamma}}{2}}
\sqrt{\frac{F_1 \left ( C \right )}{ \log_2 \left ( m \right )}}
\ln \left ( \left ( C-1 \right ) 
\left ( 4 \left ( \frac{m}{\log_2 \left ( m \right )} 
\right )^{\frac{1}{d_{\mathcal{G}, \gamma}}}
\right )^5 \right ) \right ].
$$
\end{theorem}
Comparing Theorem~\ref{theorem:dependence-on-m-C-gamma-L_2-norm}
with the result based on \eqref{eq:Old-bound-on-metric-entropy}:
Theorem~7 in \citet{Gue17} \citep[see also Theorem~3 in][]{MusLauGue19}
produces the following observations.
The growth of $F_f$
with the inverse of the margin parameter $\gamma$ is now a
$O\left ( \gamma^{-\frac{d_{\mathcal{G}, \gamma}}{2}} \right )$,
whereas the previous dependence was a
$O\left ( \gamma^{-\frac{d_{\mathcal{G}, \gamma}}{2}}
\sqrt{\ln \left ( \gamma^{-1} \right )} \right )$.
Regarding the dependence on the number $C$ of categories,
it is now a 
$O\left ( C^{\frac{d_{\mathcal{G}, C}}{2}} \ln \left ( C \right ) \right )$,
implying that it is always sublinear except when $d_{\mathcal{G}, C}$
takes its maximum value $2$, i.e.,
when no coupling between the outputs can be exploited.
The only prize to pay occurs for $d_{\mathcal{G}, \gamma} \geqslant 2$
(complex classifiers).
Then, the dependence on the sample size $m$ increases
by a factor $\sqrt{\ln \left ( m \right )}$.

\section{Conclusions}
\label{sec:conclusions}
We have established that the guaranteed risks involving
the fat-shattering dimension of the class of margin functions
$\rho_{\mathcal{G}}$
can always be simply improved by replacing this dimension with
the margin Graph dimension of the same class
(Lemmas~\ref{lemma:extended-Lemma-3.5-in-AloBenCesHau97-gamma-G-dim-optimized}
and \ref{lemma:extended-Theorem-1-in-MenVer03}).
Currently, the gain is limited by the lack of malleability
of the corresponding structural result:
Lemma~\ref{lemma:from-gamma-dimension-to-gamma-dimensions}.
Fortunately, the use of another $\gamma$-$\Psi$-dimension, the margin Natarajan
dimension, makes it possible to exploit basic features of the classifier
of interest 
(Formula~\eqref{eq:from-gamma-N-dimension-to-gamma-dimension-Banach} and
Lemma~\ref{lemma:gamma-N-dimension-of-M-SVMs}).
The major consequence
is an improved dependence of the confidence
interval on the margin parameter $\gamma$. This holds true both
with the $L_{\infty}$-norm
(Inequality~\eqref{eq:New-bound-on-metric-entropy-infty-norm})
and the $L_2$-norm
(Theorem~\ref{theorem:dependence-on-m-C-gamma-L_2-norm}).
As soon as it is possible to
take into account the coupling between the component functions
of the classifier, i.e., for $d_{\mathcal{G}, C} < 2$, 
the dependence on the number $C$ of categories becomes sublinear.
The only drawback is that the convergence rate of the $L_2$-norm bound
can be worsened by a factor $\sqrt{\ln \left ( m \right )}$ when
the underlying binary classifiers are complex (large
values of $d_{\mathcal{G}, \gamma}$).
The phenomenon is a direct consequence of 
the appearance of the logarithmic function of $\epsilon^{-1}$ in the exponent of
Inequality~\eqref{eq:extended-Theorem-1-in-MenVer03-gamma-N-dim}
(compared to Inequality~\eqref{eq:extended-Theorem-1-in-MenVer03}).
Whether this term can be eliminated is the subject of an ongoing research.
Regarding the structural results, a promising idea for their improvement
is the one recently developed by Kontorovich
\citep[see for instance][]{Kon18}.

\acks{The author would like to thank R.~Vershynin for his explanations
on the proof of Theorem~1 in \citet{MenVer03}.
This work was partly funded by a CNRS research grant (PEPS).}

\bibliography{App}
\bibliographystyle{plainnat}

\appendix

\section{Proofs of the Basic Results on the 
\texorpdfstring{$\gamma$-$\Psi$-dimensions}{gamma-Psi-dimensions}}

This appendix gathers the proofs of the basic properties of the margin
Graph dimension and the margin Natarajan dimension.

\subsection{Margin Graph Dimension}

The proof of Proposition~\ref{prop:gamma-dimension-with-and-without-pi_gamma}
is the following one.

\begin{proof}
For $\gamma \in \left ( 0, 1 \right ]$ and
$\epsilon \in \left ( 0, \frac{\gamma}{2} \right ]$,
let $s_{\mathcal{T}^n} = \left \{ t_i: \;
1 \leqslant i \leqslant n \right \}$ be a subset of $\mathcal{T}$
$\epsilon$-shattered by the subset
$\left \{ f_{\gamma}^{\mathbf{s}_n} = \pi_{\gamma} \circ f^{\mathbf{s}_n}: \;
\mathbf{s}_n \in \left \{ -1, 1 \right \}^n \right \}$
of $\pi_{\gamma} \circ \mathcal{F}$ and let
$\mathbf{b}_n = \left ( b_i \right )_{1 \leqslant i \leqslant n}$
be a witness to this shattering. Obviously,
$\mathbf{b}_n \in \left [ \epsilon, \gamma-\epsilon \right ]^n$.
Consequently,
$$
f_{\gamma}^{\mathbf{s}_n} \left ( t_i \right ) - b_i \geqslant \epsilon
\Longrightarrow 
f_{\gamma}^{\mathbf{s}_n} \left ( t_i \right ) \geqslant 2 \epsilon > 0
\Longrightarrow
f^{\mathbf{s}_n} \left ( t_i \right ) \geqslant
f_{\gamma}^{\mathbf{s}_n} \left ( t_i \right )
\Longrightarrow
f^{\mathbf{s}_n} \left ( t_i \right ) - b_i \geqslant \epsilon.
$$
Similarly,
$$
-f_{\gamma}^{\mathbf{s}_n} \left ( t_i \right ) + b_i \geqslant \epsilon
\Longrightarrow
f_{\gamma}^{\mathbf{s}_n} \left ( t_i \right ) \leqslant \gamma - 2 \epsilon
< \gamma \Longrightarrow
f^{\mathbf{s}_n} \left ( t_i \right ) \leqslant
f_{\gamma}^{\mathbf{s}_n} \left ( t_i \right )
\Longrightarrow
-f^{\mathbf{s}_n} \left ( t_i \right ) + b_i \geqslant \epsilon.
$$
\end{proof}
The proof of Proposition~\ref{prop:ordering-on-the-dimensions}
is the following one.

\begin{proof}
For $\gamma \in \mathbb{R}_+^*$, let $s_{\mathcal{Z}^n} =
\left \{ z_i = \left ( x_i, y_i \right ): \;
1 \leqslant i \leqslant n \right \}$ be a subset of $\mathcal{Z}$
$\gamma$-N-shattered by
$\left \{ f_{\mathbf{s}_n}: \;
\mathbf{s}_n \in \left \{ -1, 1 \right \}^n \right \} \subset \mathcal{F}$
and let $\left ( \mathbf{b}_n, \mathbf{c}_n \right )$
be a witness to this shattering.
To prove the left-hand side inequality of 
Formula~\eqref{eq:ordering-on-the-dimensions},
it suffices to notice that for a given vector $\mathbf{s}_n$,
the function $f_{\mathbf{s}_n} \in \mathcal{F}$ satisfying
$$
\forall i \in \llbracket 1; n \rrbracket, \;\;
\begin{cases}
\text{if } s_i = 1, \;
f_{\mathbf{s}_n} \left ( x_i, y_i \right ) - b_i \geqslant \gamma \\
\text{if } s_i = -1, \;
f_{\mathbf{s}_n} \left ( x_i, c_i \right ) + b_i \geqslant \gamma
\end{cases}
$$
also satisfies
$$
\forall i \in \llbracket 1; n \rrbracket, \;\;
\begin{cases}
\text{if } s_i = 1, \;
f_{\mathbf{s}_n} \left ( x_i, y_i \right ) - b_i \geqslant \gamma \\
\text{if } s_i = -1, \;
\max_{k \neq y_i} f_{\mathbf{s}_n} \left ( x_i, k \right ) + b_i
\geqslant \gamma
\end{cases}.
$$
Keeping the notations above, proving
the right-hand side inequality of
Formula~\eqref{eq:ordering-on-the-dimensions}
boils down to establishing that
$\max_{k \neq y_i} f_{\mathbf{s}_n} \left ( x_i, k \right )
\leqslant -f_{\mathbf{s}_n} \left ( x_i, y_i \right )$.
Indeed,
\begin{align*}
\forall f \in \mathcal{F}, \; \forall x \in \mathcal{X}, \;
\max_{1 \leqslant k < l \leqslant C}
\left \{ f \left ( x, k \right ) + f \left ( x, l \right ) \right \} = 0
& \Longrightarrow \;
f_{\mathbf{s}_n} \left ( x_i, y_i \right )
+ \max_{k \neq y_i} f_{\mathbf{s}_n} \left ( x_i, k \right ) \leqslant 0 \\
& \Longrightarrow \;
\max_{k \neq y_i} f_{\mathbf{s}_n} \left ( x_i, k \right )
\leqslant -f_{\mathbf{s}_n} \left ( x_i, y_i \right ).
\end{align*}
\end{proof}
The proof of 
Proposition~\ref{prop:capacities-of-the-two-classes-of-margin-functions}
is the following one.

\begin{proof}
Let $g$ be any function in $\mathcal{G}$.
According to Definitions~\ref{def:class-of-margin-functions} and
\ref{def:new-class-of-margin-functions}, a necessary condition for
$\rho_g$ and $\tilde{\rho}_g$ to differ on $z \in \mathcal{Z}$ is that
$\rho_g \left ( z \right ) < 0$. In that case,
$\tilde{\rho}_g \left ( z \right ) \leqslant 0$, so that
for every $\gamma \in \mathbb{R}_+^*$,
$\rho_{g, \gamma} \left ( z \right ) 
= \tilde{\rho}_{g, \gamma} \left ( z \right ) = 0$.
Equation~\eqref{eq:unification-by-squashing} has been proved.
For $\gamma \in \mathbb{R}_+^*$, let $s_{\mathcal{Z}^n} =
\left \{ z_i = \left ( x_i, y_i \right ): \;
1 \leqslant i \leqslant n \right \}$ be a subset of $\mathcal{Z}$
$\gamma$-G-shattered by
$\left \{ \rho_{g^{\mathbf{s}_n}}: \;
\mathbf{s}_n \in \left \{ -1, 1 \right \}^n \right \} 
\subset \rho_{\mathcal{G}}$
and let $\mathbf{b}_n =  \left ( b_i \right )_{1 \leqslant i \leqslant n}
\in \mathbb{R}_+^n$
be a witness to this shattering.
$$
\rho_{g^{\mathbf{s}_n}} \left ( z_i \right ) - b_i \geqslant \gamma
\Longrightarrow \rho_{g^{\mathbf{s}_n}} \left ( z_i \right ) > 0
\Longrightarrow \tilde{\rho}_{g^{\mathbf{s}_n}} \left ( z_i \right ) =
\rho_{g^{\mathbf{s}_n}} \left ( z_i \right ).
$$
Furthermore,
$$
\max_{k \neq y_i} \rho_{g^{\mathbf{s}_n}} \left ( x_i, k \right ) =
- \tilde{\rho}_{g^{\mathbf{s}_n}} \left ( x_i, y_i \right ).
$$
To sum up,
$$
\forall i \in \llbracket 1; n \rrbracket, \;\;
\begin{cases}
\text{if } s_i = 1, \;
\rho_{g^{\mathbf{s}_n}} \left ( x_i, y_i \right ) - b_i \geqslant \gamma \\
\text{if } s_i = -1, \;
\max_{k \neq y_i} \rho_{g^{\mathbf{s}_n}} \left ( x_i, k \right ) + b_i
\geqslant \gamma
\end{cases}
$$
implies that
$$
\forall i \in \llbracket 1; n \rrbracket, \;\;
s_i \left (
\tilde{\rho}_{g^{\mathbf{s}_n}} \left ( x_i, y_i \right ) - b_i \right )
\geqslant \gamma.
$$
We have established that
$\gamma\text{-G-dim} \left ( \rho_{\mathcal{G}} \right )
\leqslant \gamma\text{-dim} \left ( \tilde{\rho}_{\mathcal{G}} \right )$.
The complementary inequality is obtained in the same way
(by making use of $\mathbf{b}_n \in \mathbb{R}_+^n$),
which concludes the proof of
Formula~\eqref{eq:identity-of-combinatorial-dimensions}.
\end{proof}

\subsection{Margin Natarajan Dimension}

The proof of
Lemma~\ref{lemma:from-margin-Graph-dimension-to-margin-Natarajan-dimension}
is the following one.

\begin{proof}
Let us set $d_G = \gamma\text{-G-dim} \left ( \mathcal{F} \right )$
and $d_N = \gamma\text{-N-dim} \left ( \mathcal{F} \right )$.
Formula~\eqref{eq:from-margin-Graph-dimension-to-margin-Natarajan-dimension}
is trivially true for $d_G \leqslant 1$. 
Thus, we prove it under the assumption that $d_G \geqslant 2$.
Let $\mathcal{F}_0$ be a subset of $\mathcal{F}$
of cardinality $2^{d_G}$ that $\gamma$-G-shatters a subset
$s_{\mathcal{Z}^{d_G}}$ of $\mathcal{Z}$ of cardinality $d_G$.
For notational simplicity, we set
$s_{\mathcal{Z}^{d_G}} 
= \left \{ z_i: 1 \leqslant i \leqslant d_G \right \}$.
Let the vector $\mathbf{b}_{d_G}
=  \left ( b_i \right )_{1 \leqslant i \leqslant d_G}
\in \mathbb{R}_+^{d_G}$
be a witness to this shattering.
For $\bar{\mathcal{F}} \subset \mathcal{F}_0$
satisfying $\left | \bar{\mathcal{F}} \right | \geqslant 2$, let
$\mathcal{S} \left ( \bar{\mathcal{F}} \right )$ be the set of the subsets
$s^{\left ( 1 \right )}_{\mathcal{Z}^r} 
= \left \{ z^{\left ( 1 \right )}_i = 
\left ( x^{\left ( 1 \right )}_i, y^{\left ( 1 \right )}_i \right ): \;
1 \leqslant i \leqslant r \right \}$
of $s_{\mathcal{Z}^{d_G}}$ such that:
$$
s^{\left ( 1 \right )}_{\mathcal{Z}^r}
\in \mathcal{S} \left ( \bar{\mathcal{F}} \right )
\Longleftrightarrow
\forall \left \{ f, f^{\prime} \right \} \subset \bar{\mathcal{F}}, \;
\exists z_i \in s^{\left ( 1 \right )}_{\mathcal{Z}^r}: \;
\begin{cases}
f \left ( z_i \right ) - b_i \geqslant \gamma \\
\max_{k \neq y_i} f^{\prime} \left ( x_i, k \right )
+ b_i \geqslant \gamma
\end{cases}
\text{ or vice versa.}
$$
The meaning of the formula $z_i \in s^{\left ( 1 \right )}_{\mathcal{Z}^r}$ is
the obvious one, i.e., 
$\exists z^{\left ( 1 \right )}_j \in s^{\left ( 1 \right )}_{\mathcal{Z}^r}:
z_i = z^{\left ( 1 \right )}_j$.
Notice first that $s_{\mathcal{Z}^{d_G}}$ belongs
to all the sets $\mathcal{S} \left ( \bar{\mathcal{F}} \right )$.
For $\bar{\mathcal{F}} \subset \mathcal{F}_0$ 
satisfying $\left | \bar{\mathcal{F}} \right | \geqslant 2$ and
$s^{\left ( 1 \right )}_{\mathcal{Z}^r} 
\in \mathcal{S} \left ( \bar{\mathcal{F}} \right )$,
let $h \left ( \bar{\mathcal{F}}, 
s^{\left ( 1 \right )}_{\mathcal{Z}^r} \right )$
be the number of triplets
$\left ( s^{\left ( 2 \right )}_{\mathcal{Z}^u}, 
\mathbf{b}^{\left ( 2 \right )}_u, 
\mathbf{c}^{\left ( 2 \right )}_u \right )$ satisfying:
$$
\begin{cases}
s^{\left ( 2 \right )}_{\mathcal{Z}^u}
= \left \{ z^{\left ( 2 \right )}_i =
\left ( x^{\left ( 2 \right )}_i, y^{\left ( 2 \right )}_i \right ): \;
1 \leqslant i \leqslant u \right \}
\subset s^{\left ( 1 \right )}_{\mathcal{Z}^r} \\
\forall \left ( i, j \right ): 1 \leqslant i < j \leqslant u,
\left ( z^{\left ( 2 \right )}_i, z^{\left ( 2 \right )}_j \right )
= \left ( z_v, z_w \right )
\Longrightarrow 1 \leqslant v < w \leqslant d_G \\
\forall \left ( i, j \right ) \in \llbracket 1; u \rrbracket \times
\llbracket 1; d_G \rrbracket, \;
z^{\left ( 2 \right )}_i = z_j
\Longrightarrow b^{\left ( 2 \right )}_i = b_j \\
\forall i \in \llbracket 1; u \rrbracket, \;
c^{\left ( 2 \right )}_i
\in \mathcal{Y} \setminus \left \{ y^{\left ( 2 \right )}_i \right \}
\end{cases}
$$
which are $\gamma$-N-shattered by $\bar{\mathcal{F}}$.
The function $h$ exhibits two major properties which play a central part
in the derivation of the upper bound on $d_G$.

\paragraph{Inductive property}
Let $\bar{\mathcal{F}}$ and $s^{\left ( 1 \right )}_{\mathcal{Z}^r}$
be defined as above. By the pigeonhole principle,
there exist $z_{i_0} \in s^{\left ( 1 \right )}_{\mathcal{Z}^r}$,
a value $c_{i_0} \in \mathcal{Y} \setminus \left \{ y_{i_0} \right \}$
and two subsets $\bar{\mathcal{F}}_+$ and $\bar{\mathcal{F}}_-$
of $\bar{\mathcal{F}}$ of cardinalities
$\left | \bar{\mathcal{F}}_+ \right | \geqslant
\left \lceil \left \lfloor
\frac{\left | \bar{\mathcal{F}} \right |}{2} \right \rfloor
\frac{1}{r} \right \rceil$
and
$\left | \bar{\mathcal{F}}_- \right | \geqslant
\left \lceil \left \lceil \left \lfloor
\frac{\left | \bar{\mathcal{F}} \right |}{2} \right \rfloor
\frac{1}{r} \right \rceil \frac{1}{C-1} \right \rceil$
such that
$$
\begin{cases}
\forall f_+ \in \bar{\mathcal{F}}_+, \;\;
f_+ \left ( z_{i_0} \right ) - b_{i_0} \geqslant \gamma \\
\forall f_- \in \bar{\mathcal{F}}_-, \;\;
f_- \left ( x_{i_0}, c_{i_0} \right ) + b_{i_0} \geqslant \gamma
\end{cases}.
$$
Let $s^{\left ( 2 \right )}_{\mathcal{Z}^{r-1}} 
= s^{\left ( 1 \right )}_{\mathcal{Z}^r} \setminus \left \{ z_{i_0} \right \}$.
By construction, $s^{\left ( 2 \right )}_{\mathcal{Z}^{r-1}}
\in \mathcal{S} \left ( \bar{\mathcal{F}}_+ \right )
\bigcap \mathcal{S} \left ( \bar{\mathcal{F}}_- \right )$.
Clearly, $\bar{\mathcal{F}}$ $\gamma$-N-shatters all the triplets
$\left ( s^{\left ( 3 \right )}_{\mathcal{Z}^u}, 
\mathbf{b}^{\left ( 3 \right )}_u, \mathbf{c}^{\left ( 3 \right )}_u \right )$
with $s^{\left ( 3 \right )}_{\mathcal{Z}^u}
\subset s^{\left ( 2 \right )}_{\mathcal{Z}^{r-1}}$
$\gamma$-N-shattered by either $\bar{\mathcal{F}}_+$ or $\bar{\mathcal{F}}_-$
plus $\left ( \left \{ z_{i_0} \right \}, b_{i_0}, c_{i_0} \right )$.
Moreover, if the triplet
$\left ( s^{\left ( 3 \right )}_{\mathcal{Z}^u}, 
\mathbf{b}^{\left ( 3 \right )}_u, \mathbf{c}^{\left ( 3 \right )}_u \right )$
is $\gamma$-N-shattered by both $\bar{\mathcal{F}}_+$ and $\bar{\mathcal{F}}_-$,
then $\bar{\mathcal{F}}$ $\gamma$-N-shatters the triplet
$\left ( s^{\left ( 4 \right )}_{\mathcal{Z}^{u+1}}, 
\mathbf{b}^{\left ( 4 \right )}_{u+1}, 
\mathbf{c}^{\left ( 4 \right )}_{u+1} \right )$
deduced from
$\left ( s^{\left ( 3 \right )}_{\mathcal{Z}^u}, 
\mathbf{b}^{\left ( 3 \right )}_u, \mathbf{c}^{\left ( 3 \right )}_u \right )$
by inserting the components of
$\left ( \left \{ z_{i_0} \right \}, b_{i_0}, c_{i_0} \right )$
at the right place.
Since by construction, $s^{\left ( 4 \right )}_{\mathcal{Z}^{u+1}}
\subset s^{\left ( 1 \right )}_{\mathcal{Z}^r}$ and
$s^{\left ( 4 \right )}_{\mathcal{Z}^{u+1}}
\not\subset s^{\left ( 2 \right )}_{\mathcal{Z}^{r-1}}$, it follows that:
\begin{equation}
\label{eq:from-margin-Graph-dimension-to-margin-Natarajan-dimension-part-1}
h \left ( \bar{\mathcal{F}}, s^{\left ( 1 \right )}_{\mathcal{Z}^r} \right )
\geqslant 
h \left ( \bar{\mathcal{F}}_+, 
s^{\left ( 2 \right )}_{\mathcal{Z}^{r-1}} \right )
+ h \left ( \bar{\mathcal{F}}_-, 
s^{\left ( 2 \right )}_{\mathcal{Z}^{r-1}} \right )
+1.
\end{equation}

\paragraph{Termination property}
By definition of $\bar{\mathcal{F}}$, $s^{\left ( 1 \right )}_{\mathcal{Z}^r}$
and the Natarajan dimension with margin $\gamma$,
\begin{equation}
\label{eq:from-margin-Graph-dimension-to-margin-Natarajan-dimension-part-2}
h \left ( \bar{\mathcal{F}}, 
s^{\left ( 1 \right )}_{\mathcal{Z}^r} \right ) \geqslant 1.
\end{equation}

These two properties call for a reasoning based on a binary tree
whose root is $\mathcal{F}_{0,1} = \mathcal{F}_0$ and whose depth is $d_G$.
Its derivation utterly rests on the specification of two functions,
$z^*$ and $c^*$, respectively returning for each inner node
the example and the category defining its two sons.
For $i \in \left \llbracket 0; d_G \right \rrbracket$, the nodes
of the tree at depth $i$ are denoted $\mathcal{F}_{i,j}$.
If $\left | \mathcal{F}_{i,j} \right | \geqslant 2$ 
($\mathcal{F}_{i,j}$ is an inner node),
then its two sons are $\mathcal{F}_{i+1,2j-1}$ and $\mathcal{F}_{i+1,2j}$.
The values of $z^* \left ( 0, 1 \right )$ and $c^* \left ( 0, 1 \right )$
(defining the two sons, $\mathcal{F}_{1,1}$
and $\mathcal{F}_{1,2}$, of $\mathcal{F}_{0,1}$), are given by
$z^* \left ( 0, 1 \right ) = z_1$, so that
$$
\mathcal{F}_{1,1} = \left \{ f \in \mathcal{F}_{0,1}: \;
f \left ( z_1 \right ) - b_1 \geqslant \gamma \right \}.
$$
and
$$
c^* \left ( 0, 1 \right ) = \min \left \{
\argmax_{k \neq y_1} \left |
\left \{ f \in \mathcal{F}_{0,1} \setminus \mathcal{F}_{1,1}: 
\; f \left ( x_1, k \right ) = \max_{l \neq y_1} 
f \left ( x_1, l \right ) \right \} \right | \right \},
$$
so that
$$
\mathcal{F}_{1,2} = \left \{ f \in \mathcal{F}_{0,1}
\setminus \mathcal{F}_{1,1}: \;
f \left ( x_1, c^* \left ( 0, 1 \right ) \right ) = \max_{l \neq y_1} 
f \left ( x_1, l \right ) 
\right \}.
$$
Note that
$$
\forall f \in \mathcal{F}_{1,2}, \;
f \left ( x_1, c^* \left ( 0, 1 \right ) \right ) + b_1 \geqslant \gamma,
$$
$\left | \mathcal{F}_{1,1} \right | = 2^{d_G-1}$, and according to the
pigeonhole principle,
$\left | \mathcal{F}_{1,2} \right | \geqslant \left \lceil
\frac{2^{d_G-1}}{C-1} \right \rceil$.
For $i \in \left \llbracket 1; d_G-1 \right \rrbracket$, 
if $\mathcal{F}_{i,j}$ is not a leaf, then the construction of its two sons
depends on the value of $j$. 

\paragraph{For $j=1,$} $\mathcal{F}_{i+1,1}$
and $\mathcal{F}_{i+1,2}$ are defined according to the same principle as
$\mathcal{F}_{1,1}$ and $\mathcal{F}_{1,2}$, i.e.,
$z^* \left ( i, 1 \right ) = z_{i+1}$, so that
$$
\mathcal{F}_{i+1,1} = \left \{ f \in \mathcal{F}_{i,1}: \;
f \left ( z_{i+1} \right ) - b_{i+1} \geqslant \gamma \right \},
$$
and $c^* \left ( i, 1 \right )$ 
is such that among the $C-1$ sets
$$
\mathcal{F}_{i+1,2,k} =
\left \{ f \in \mathcal{F}_{i,1}
\setminus \mathcal{F}_{i+1,1}: \; f \left ( x_{i+1}, k \right ) 
= \max_{l \neq y_{i+1}} 
f \left ( x_{i+1}, l \right )\right \}
$$
for $k \neq y_{i+1}$,
$\mathcal{F}_{i+1,2}$ is a set of maximal cardinality. 
As in the original case, this implies that
$\left | \mathcal{F}_{i+1,1} \right | = 2^{d_G-i-1}$ and
$\left | \mathcal{F}_{i+1,2} \right | \geqslant \left \lceil
\frac{2^{d_G-i-1}}{C-1} \right \rceil \geqslant
\frac{2^{d_G-i-1}}{C-1}$.

\paragraph{For $j>1,$} the derivation of the sets $\mathcal{F}_{i+1,2j-1}$
and $\mathcal{F}_{i+1,2j}$ follows that of the sets
$\bar{\mathcal{F}}_+$ and $\bar{\mathcal{F}}_-$ of the inductive property
($z^* \left ( i, j \right ) = z_{i_0}$ and
$c^* \left ( i, j \right ) = c_{i_0}$).
As a consequence,
$$
\begin{cases}
\left | \mathcal{F}_{i+1,2j-1} \right | \geqslant
\frac{\left | \mathcal{F}_{i,j} \right |}{3 \left ( d_G-i \right )} \\
\left | \mathcal{F}_{i+1,2j} \right | \geqslant
\frac{\left | \mathcal{F}_{i,j} \right |}
{3 \left ( d_G-i \right ) \left ( C-1 \right )}
\end{cases}.
$$
The binary tree has been built in such a way that
all the triplets $\left ( \mathcal{F}_{i,j}, \mathcal{F}_{i+1,2j-1},
\mathcal{F}_{i+1,2j} \right )$ satisfy
\eqref{eq:from-margin-Graph-dimension-to-margin-Natarajan-dimension-part-1}
(whether $j>1$ or not).
Furthermore, any set which is not a leaf satisfies
\eqref{eq:from-margin-Graph-dimension-to-margin-Natarajan-dimension-part-2}.
By combination of these two inequalities,
\begin{equation}
\label{eq:from-margin-Graph-dimension-to-margin-Natarajan-dimension-part-3}
h \left ( \mathcal{F}_{0,1}, s_{\mathcal{Z}^{d_G}} \right )
\geqslant \ell \left ( \mathcal{F}_{0,1} \right ) -1,
\end{equation}
where the function $\ell$ returns the number
of leaves of the (sub)tree whose root is its argument.
Thus, the next step of the proof consists in deriving a lower bound on
$\ell \left ( \mathcal{F}_{0,1} \right )$.
Since the assumption $d_G \geqslant 2$ has been made, a simple induction gives
\begin{equation}
\label{eq:from-margin-Graph-dimension-to-margin-Natarajan-dimension-part-8}
\ell \left ( \mathcal{F}_{0,1} \right ) = 
\ell \left ( \mathcal{F}_{d_G,1} \right )
+ \sum_{i=1}^{d_G} \ell \left ( \mathcal{F}_{i,2} \right )
= 2 + \sum_{i=1}^{d_G-1} \ell \left ( \mathcal{F}_{i,2} \right ).
\end{equation}
We now establish by induction that any node
$\mathcal{F}_{i,j}$ which is not a leaf satisfies:
\begin{equation}
\label{eq:from-margin-Graph-dimension-to-margin-Natarajan-dimension-part-4}
\ell \left ( \mathcal{F}_{i,j} \right )
\geqslant \left | \mathcal{F}_{i,j} \right 
|^{\frac{1}{\log_2 \left ( 3 \sqrt{C-1} \left ( d_G-i \right ) \right )}}.
\end{equation}
The induction is on the depth $i$ of the node.
Inequality~\eqref{eq:from-margin-Graph-dimension-to-margin-Natarajan-dimension-part-4}
is obviously true for $i=d_G-1$, i.e., the node $\mathcal{F}_{d_G-1,1}$, since 
$\ell \left ( \mathcal{F}_{d_G-1,1} \right ) =
\left | \mathcal{F}_{d_G-1,1} \right | = 2$.
Suppose now that it holds true for all the nodes
at depth ranging from $i+1$ to $d_G-1$. Then,

\begin{align*}
\ell \left ( \mathcal{F}_{i,j} \right ) 
& = \;
\ell \left ( \mathcal{F}_{i+1,2j-1} \right )
+ \ell \left ( \mathcal{F}_{i+1,2j} \right ) \\
& \geqslant \;
\left ( \frac{\left | \mathcal{F}_{i,j} \right |}{3 \left ( d_G-i \right )}
\right )^{\frac{1}
{\log_2 \left ( 3 \sqrt{C-1} \left ( d_G-i \right ) \right )}}
+ \left ( \frac{\left | \mathcal{F}_{i,j} \right |}
{3 \left ( d_G-i \right ) \left ( C-1 \right )} \right )^{\frac{1}
{\log_2 \left ( 3 \sqrt{C-1} \left ( d_G-i \right ) \right )}} \\
& = \;
\frac{1}{2} \left ( \left ( \sqrt{C-1} \right )^{\frac{1}
{\log_2 \left ( 3 \sqrt{C-1} \left ( d_G-i \right ) \right )}}
+ \left ( \sqrt{C-1} \right )^{-\frac{1}
{\log_2 \left ( 3 \sqrt{C-1} \left ( d_G-i \right ) \right )}} \right )
\left | \mathcal{F}_{i,j} \right |^{\frac{1}
{\log_2 \left ( 3 \sqrt{C-1} \left ( d_G-i \right ) \right )}} \\
& \geqslant \;
\frac{1}{2} \min_{t \in \mathbb{R}_+^*} \left ( t + \frac{1}{t} \right )
\left | \mathcal{F}_{i,j} \right |^{\frac{1}
{\log_2 \left ( 3 \sqrt{C-1} \left ( d_G-i \right ) \right )}} \\
& = \;
\left | \mathcal{F}_{i,j} \right |^{\frac{1}
{\log_2 \left ( 3 \sqrt{C-1} \left ( d_G-i \right ) \right )}}.
\end{align*}
We thus get for the whole tree:
\begin{equation}
\label{eq:from-margin-Graph-dimension-to-margin-Natarajan-dimension-part-0}
\ell \left ( \mathcal{F}_{0,1} \right )
\geqslant
2^{\frac{d_G}{\log_2 \left ( 3 \sqrt{C-1} d_G \right )}}
\end{equation}
(the sharper bound provided by
\eqref{eq:from-margin-Graph-dimension-to-margin-Natarajan-dimension-part-8}
does not bring any improvement here).
Function $h$ can be bounded from above in a classical way.
Since $d_N \geqslant 1$, combinatorics produces
$$
h \left ( \mathcal{F}_0, s_{\mathcal{Z}^{d_G}} \right )
\leqslant \sum_{u=1}^{d_N} {d_G \choose u} \left ( C-1 \right )^u,
$$
which gives birth to a handy formula thanks to a well-known computation
\citep[see for instance the proof of Corollary~3.3 in][]{MohRosTal12}:
\begin{equation}
\label{eq:from-margin-Graph-dimension-to-margin-Natarajan-dimension-part-5}
h \left ( \mathcal{F}_0, s_{\mathcal{Z}^{d_G}} \right ) 
\leqslant \left ( \frac{\left ( C-1 \right ) e d_G}{d_N} \right )^{d_N} -1.
\end{equation}
Combining the lower bound
(Inequalities~\eqref{eq:from-margin-Graph-dimension-to-margin-Natarajan-dimension-part-3} and
\eqref{eq:from-margin-Graph-dimension-to-margin-Natarajan-dimension-part-0})
and the upper one
(Inequality~\eqref{eq:from-margin-Graph-dimension-to-margin-Natarajan-dimension-part-5}) produces by transitivity
\begin{align}
d_G 
& \leqslant \;
 d_N \log_2 \left ( \frac{\left ( C-1 \right ) e d_G}{d_N} \right )
\log_2 \left ( 3 \sqrt{C-1} d_G \right ) \nonumber \\
\label{eq:from-margin-Graph-dimension-to-margin-Natarajan-dimension-part-6}
& \leqslant \;
\frac{1}{\ln^2 \left ( 2 \right )} d_N
\ln \left ( F \left ( C \right ) \frac{d_G}{d_N} \right )
\ln \left ( F \left ( C \right ) d_G \right ),
\end{align}
where $F \left ( C \right ) = e \left ( C-1 \right )$.
To bound from above the right-hand side of Inequality 
\eqref{eq:from-margin-Graph-dimension-to-margin-Natarajan-dimension-part-6},
we resort to the following statement:
\begin{equation}
\label{eq:from-margin-Graph-dimension-to-margin-Natarajan-dimension-part-7}
\forall \left ( u, u_0 \right ) 
\in \left [ 1, +\infty \right )^2, \;
\ln \left ( u \right ) \leqslant 2 u_0 u^{\frac{1}{4 u_0}},
\end{equation}
with $u_0 = \ln \left ( F \left ( C \right ) \right )$.
We then obtain
$$
\begin{cases}
\ln \left ( F \left ( C \right ) \frac{d_G}{d_N} \right )
\leqslant 2 e^{\frac{1}{4}} \ln \left ( F \left ( C \right ) \right )
\left ( \frac{d_G}{d_N}
\right )^{\frac{1}{4 \ln \left ( F \left ( C \right ) \right )}} \\
\ln \left ( F \left ( C \right ) d_G \right )
\leqslant 2 e^{\frac{1}{4}} \ln \left ( F \left ( C \right ) \right )
d_G^{\frac{1}{4 \ln \left ( F \left ( C \right ) \right )}}
\end{cases}.
$$
By substitution into
\eqref{eq:from-margin-Graph-dimension-to-margin-Natarajan-dimension-part-6},

\begin{align*}
d_G & \leqslant \; \frac{4 \sqrt{e}}{\ln^2 \left ( 2 \right )}
\ln^2 \left ( F \left ( C \right ) \right )
d_G^{\frac{1}{2 \ln \left ( F \left ( C \right ) \right )}}
d_N^{\frac{4 \ln \left ( F \left ( C \right ) \right ) -1}
{4 \ln \left ( F \left ( C \right ) \right )}} \\
& = \;
\left ( \frac{4 \sqrt{e}}{\ln^2 \left ( 2 \right )}
\right )^{\frac{2 \ln \left ( F \left ( C \right ) \right )}
{2 \ln \left ( F \left ( C \right ) \right ) -1}}
\left ( \ln \left ( F \left ( C \right ) \right )
\right )^{\frac{4 \ln \left ( F \left ( C \right ) \right )}
{2 \ln \left ( F \left ( C \right ) \right ) -1}}
d_N^{\frac{4 \ln \left ( F \left ( C \right ) \right ) -1}
{4 \ln \left ( F \left ( C \right ) \right ) -2}} \\
& = \;
\left ( \frac{4 \sqrt{e}}{\ln^2 \left ( 2 \right )}
\right )^{\frac{2 \ln \left ( F \left ( C \right ) \right )}
{2 \ln \left ( F \left ( C \right ) \right ) -1}}
\left ( \ln \left ( F \left ( C \right ) \right )
\right )^{\frac{2}
{2 \ln \left ( F \left ( C \right ) \right ) -1}}
\ln^2 \left ( 2 \right ) \log_2^2 \left ( F \left ( C \right ) \right )
d_N^{\frac{4 \ln \left ( F \left ( C \right ) \right ) -1}
{4 \ln \left ( F \left ( C \right ) \right ) -2}} \\
& < \;
32 \log_2^2 \left ( F \left ( C \right ) \right )
d_N^{1 + \frac{1}
{4 \ln \left ( F \left ( C \right ) \right ) -2}}.
\end{align*}
\end{proof}
The proof of 
Proposition~\ref{prop:margin-Natarajan-dimension-alternate-definition}
is the following one.
\begin{proof}
Without loss of generality, we assume that $\bar{\mathcal{F}}$ is of minimal
cardinality $2^n$ and set accordingly
$\bar{\mathcal{F}} = \left \{ f_{\mathbf{s}_n}: \;
\mathbf{s}_n \in \left \{ -1, 1 \right \}^n \right \}$.
Consider the following bijection on $\left \{ -1, 1 \right \}^n$:
$$
\begin{array}{l l l l}
B:
& \left \{ -1, 1 \right \}^n & \longrightarrow & \left \{ -1, 1 \right \}^n \\
& \mathbf{s}_n & \mapsto & \mathbf{s}'_n
\end{array}
$$
$$
\forall i \in \llbracket 1; n \rrbracket, \;\;
\begin{cases}
\text{if } y_i < c_i, \; s'_i = s_i \\
\text{if } y_i > c_i, \; s'_i = -s_i
\end{cases}.
$$
Then,
$$
\forall i \in \llbracket 1; n \rrbracket, \;\;
\begin{cases}
\text{if } s'_i = 1, \;
f_{\mathbf{s}'_n} \left ( x_i, y_i \right ) - b_i \geqslant \gamma \\
\text{if } s'_i = -1, \;
f_{\mathbf{s}'_n} \left ( x_i, c_i \right ) + b_i \geqslant \gamma
\end{cases}
\Longrightarrow
\forall i \in \llbracket 1; n \rrbracket, \;\;
\begin{cases}
\text{if } s_i = 1, \;
f_{\mathbf{s}'_n} \left ( x_i, y'_i \right ) - b'_i \geqslant \gamma \\
\text{if } s_i = -1, \;
f_{\mathbf{s}'_n} \left ( x_i, c'_i \right ) + b'_i \geqslant \gamma
\end{cases}.
$$
By definition, the triplet 
$\left ( s'_{\mathcal{Z}^n}, \mathbf{b}'_n, \mathbf{c}'_n \right )$
is $\gamma$-N-shattered by $\bar{\mathcal{F}}$, which concludes the proof.
\end{proof}

\section{Proofs of the Combinatorial Results}

This appendix gathers the proofs of the four new combinatorial results.
It starts with three lemmas which are common to all proofs.

\subsection{Shared Technical Lemmas}

Each of the combinatorial results in the literature is built upon
a basic lemma that involves two (possibly identical)
function classes whose domain
and codomain are finite sets (so that their cardinalities are also finite).
It upper bounds the cardinality of one of them in terms
of a combinatorial dimension of the other.
In the case of margin classifiers, the combinatorial dimension
of the basic lemma is
a variant of the scale-sensitive dimension of the combinatorial result,
variant designed to take
benefit from the aforementioned restrictions. The first capacity measure
of this kind is a variant of the $\gamma$-dimension:
the strong dimension \citep[Definition~3.1 in][]{AloBenCesHau97}.
The strong $\Psi$-dimensions extend the $\gamma$-$\Psi$-dimensions
according to the same principle.

\begin{definition}[Strong $\Psi$-dimensions]
\label{def:Strong-Psi-dimension}
Let $\mathcal{F}$ be a class of functions from
$\bar{\mathcal{Z}} = \bar{\mathcal{X}} \times \mathcal{Y}$
into $\left \llbracket -M_{\mathcal{F}}; M_{\mathcal{F}} \right \rrbracket$,
where $\bar{\mathcal{X}}$ is a finite subset of $\mathcal{X}$ and
$M_{\mathcal{F}} \in \mathbb{N}^*$. $\mathcal{F}$ satisfies:
$$
\forall f \in \mathcal{F}, \; \forall x \in \bar{\mathcal{X}}, \;
\max_{1 \leqslant k < l \leqslant C}
\left \{ f \left ( x, k \right ) + f \left ( x, l \right ) \right \} = 0.
$$
Let $\Psi$ be a family of mappings from $\mathcal{Y}$
into $\left \{ -1, 0, 1 \right \}$.
A subset $s_{\bar{\mathcal{Z}}^n} =
\left \{ z_i = \left ( x_i, y_i \right ): \;
1 \leqslant i \leqslant n \right \}$ of $\bar{\mathcal{Z}}$
is said to be {\em strongly $\Psi$-shattered} by $\mathcal{F}$
if there is a vector $\boldsymbol{\psi}_n
= \left ( \psi^{\left ( i \right )}
\right )_{1 \leqslant i \leqslant n} \in \Psi^n$
satisfying 
$\left ( \psi^{\left ( i \right )} \left ( y_i \right ) 
\right )_{1 \leqslant i \leqslant n} = \mathbf{1}_n$,
and a vector
$\mathbf{b}_n =  \left ( b_i \right )_{1 \leqslant i \leqslant n}
\in \left \llbracket 0; M_{\mathcal{F}} - 1 \right \rrbracket^n$
such that, for every vector
$\mathbf{s}_n = \left ( s_i \right )_{1 \leqslant i \leqslant n}
\in \left \{ -1, 1 \right \}^n$,
there is a function $f_{\mathbf{s}_n} \in \mathcal{F}$ satisfying
$$
\forall i \in \llbracket 1; n \rrbracket, \;\;
\begin{cases}
\text{if } s_i = 1, \;
\max_{\left \{ k: \; \psi^{\left ( i \right )} 
\left ( k \right ) = 1 \right \}}
f_{\mathbf{s}_n} \left ( x_i, k \right ) - b_i \geqslant 1 \\
\text{if } s_i = -1, \;
\max_{\left \{ l: \; \psi^{\left ( i \right )}
\left ( l \right ) = -1 \right \}}
f_{\mathbf{s}_n} \left ( x_i, l \right ) + b_i \geqslant 1
\end{cases}.
$$
The {\em strong $\Psi$-dimension} of $\mathcal{F}$, denoted by
$\text{S-}\Psi\text{-dim} \left ( \mathcal{F} \right )$,
is the maximal cardinality of a subset of
$\bar{\mathcal{Z}}$ strongly $\Psi$-shattered by $\mathcal{F}$,
if such maximum exists.
Otherwise, $\mathcal{F}$
is said to have infinite strong $\Psi$-dimension.
\end{definition}
In what follows, the finiteness of the domain is simply obtained by application
of a restriction to an appropriately chosen set of data points.
As for the finiteness of the codomain, it results from the application
of the following discretization operator.

\begin{definition}[$\eta$-discretization operator,
Definition~33 in \citealp{Gue07b}]
\label{def:eta-discretization-operator}
Let $\mathcal{F}$ be a class of functions from $\mathcal{T}$ into
the interval $\left [ M_{\mathcal{F}-}, M_{\mathcal{F}+} \right ]$.
For $\eta \in \mathbb{R}_+^*$, define the {\em $\eta$-discretization}
as an operator on $\mathcal{F}$ such that:
$$
\begin{array}{l l l l}
\left ( \cdot \right )^{\left ( \eta \right )}:
& \mathcal{F} & \longrightarrow & \mathcal{F}^{\left ( \eta \right )} \\
& f & \mapsto & f^{\left ( \eta \right )}
\end{array}
$$
$$
\forall t \in \mathcal{T}, \;\;
f^{\left ( \eta \right )} \left ( t \right ) =
\sign \left ( f \left ( t \right ) \right ) \cdot
\left \lfloor \frac{ \left | f \left ( t \right ) \right |}{\eta}
\right \rfloor.
$$
\end{definition}
The transitions from continuous functions to discrete ones and back
are obtained by application of the two following lemmas.

\begin{lemma}
\label{lemma:lemma-3.2-in-AloBenCesHau97-for-all-packing-numbers}
Let $\mathcal{F}$ be a class of functions from $\mathcal{T}$
into the interval $\left [ 0, M_{\mathcal{F}} \right ]$
with $M_{\mathcal{F}} \in \mathbb{R}_+^*$.
For $n \in \mathbb{N}^*$, let $\mathbf{t}_n = \left ( t_i
\right )_{1 \leqslant i \leqslant n} \in \mathcal{T}^n$.
Let $N$ be a positive integer.
For every $\epsilon \in \left ( 0, M_{\mathcal{F}} \right ]$ and every
$\eta \in \left ( 0, \frac{\epsilon}{N+1} \right ]$,
\begin{equation}
\label{eq:L_2-norm-lemma-3.2-in-AloBenCesHau97}
\forall \left ( f, f^{\prime} \right ) \in \mathcal{F}^2, \;\;
d_{2, \mathbf{t}_n} \left ( f, f^{\prime} \right )
\geqslant \epsilon \Longrightarrow
d_{2, \mathbf{t}_n} \left ( f^{\left ( \eta \right )},
{f^{\prime}}^{\left ( \eta \right )} \right )
\geqslant N,
\end{equation}
with the consequence that if the subset $\bar{\mathcal{F}}$ of $\mathcal{F}$
is $\epsilon$-separated with respect to the pseudo-metric $d_{2, \mathbf{t}_n}$,
then it is in bijection with
the subset $\bar{\mathcal{F}}^{\left ( \eta \right )}$
of $\mathcal{F}^{\left ( \eta \right )}$,
which is $N$-separated with respect to the same pseudo-metric.
Similarly, for every $\epsilon \in \left ( 0, M_{\mathcal{F}} \right ]$ 
and every $\eta \in \left ( 0, \frac{\epsilon}{2} \right ]$,
\begin{equation}
\label{eq:L_infty-norm-lemma-3.2-in-AloBenCesHau97}
\mathcal{M} \left ( \epsilon, \mathcal{F}, d_{\infty, \mathbf{t}_n} \right )
\leqslant
\mathcal{M} \left ( 2, \mathcal{F}^{\left ( \eta \right )}, 
d_{\infty, \mathbf{t}_n} \right ).
\end{equation}
\end{lemma}

\begin{proof}
For $f \in \mathcal{F}$ and $i \in \llbracket 1; n \rrbracket$,
let us denote the Euclidean division
of $f \left ( t_i \right )$ by $\eta$ as follows:
$$
\forall i \in \llbracket 1; n \rrbracket, \;\;
f \left ( t_i \right ) =
\eta f^{\left ( \eta \right )} \left ( t_i \right ) + r_i.
$$
With the notation introduced above,
$$
d_{2, \mathbf{t}_n} \left ( f, f^{\prime} \right )^2
= \frac{1}{n} \sum_{i=1}^n \left ( \eta \left (
f^{\left ( \eta \right )} \left ( t_i \right ) -
{f^{\prime}}^{\left ( \eta \right )} \left ( t_i \right ) \right )
+ r_i - r_i^{\prime} \right )^2.
$$
For $i \in \llbracket 1; n \rrbracket$, let $\delta_i =
\left |
f^{\left ( \eta \right )} \left ( t_i \right ) -
{f^{\prime}}^{\left ( \eta \right )} \left ( t_i \right )
\right |$.
\begin{align}
\left (
d_{2, \mathbf{t}_n} \left ( f, f^{\prime} \right )
\geqslant \epsilon \right ) \text{and} \left ( \eta \in \left ( 0,
\frac{\epsilon}{N+1} \right ] \right ) & \Longrightarrow \;
\left ( \frac{1}{n} \sum_{i=1}^n \left ( \eta \delta_i
+ \left | r_i - r_i^{\prime} \right | \right )^2 \right )^{\frac{1}{2}}
\geqslant \epsilon \nonumber \\
& \Longrightarrow \; \left ( \frac{1}{n} \sum_{i=1}^n \left (
\delta_i + 1 \right )^2 \right )^{\frac{1}{2}}
\geqslant \frac{\epsilon}{\eta} \nonumber \\
\label{eq:L_2-norm-lemma-3.2-in-AloBenCesHau97-part-1}
& \Longrightarrow \; \left ( \frac{1}{n} \sum_{i=1}^n \left (
\delta_i + 1 \right )^2 \right )^{\frac{1}{2}}
\geqslant N+1 \\
\label{eq:L_2-norm-lemma-3.2-in-AloBenCesHau97-part-2}
& \Longrightarrow \; \left ( \frac{1}{n} \sum_{i=1}^n
\delta_i^2 \right )^{\frac{1}{2}} +1
\geqslant N+1 \\
& \Longrightarrow \; d_{2, \mathbf{t}_n} \left  (
f^{\left ( \eta \right )},
{f^{\prime}}^{\left ( \eta \right )} \right ) \geqslant N, \nonumber
\end{align}
where the transition from 
\eqref{eq:L_2-norm-lemma-3.2-in-AloBenCesHau97-part-1} to
\eqref{eq:L_2-norm-lemma-3.2-in-AloBenCesHau97-part-2}
is provided by the triangle inequality.
To sum up, we have established
\eqref{eq:L_2-norm-lemma-3.2-in-AloBenCesHau97}, i.e., the part of the lemma
dealing with the $L_2$-norm.
To prove \eqref{eq:L_infty-norm-lemma-3.2-in-AloBenCesHau97},
it is enough to observe that
$f \left ( t \right ) - f^{\prime} \left ( t \right ) \geqslant \epsilon
\Longrightarrow f^{\left ( \frac{\epsilon}{2} \right )} \left ( t \right )
- {f^{\prime}}^{\left ( \frac{\epsilon}{2} \right )} \left ( t \right )
\geqslant 2$.
\end{proof}

\begin{lemma}
\label{lemma:extended-lemma-3.2-in-AloBenCesHau97-for-G-and-N}
Let $\mathcal{F}$ be a function class defined as in 
Definition~\ref{def:gamma-Psi-dimensions}.
Suppose further that the functions in $\mathcal{F}$ 
are defined on a finite subset 
$\bar{\mathcal{Z}} = \bar{\mathcal{X}} \times \mathcal{Y}$ of
$\mathcal{Z}$ and take their values in
$\left [ -M_{\mathcal{F}}, M_{\mathcal{F}} \right ]$
with $M_{\mathcal{F}} \in \mathbb{R}_+^*$.
For every $\eta \in \left ( 0, M_{\mathcal{F}} \right ]$
and every $\epsilon \in \left ( 0, \frac{\eta}{2} \right ]$,
\begin{numcases}{}
\label{eq:extended-lemma-3.2-in-AloBenCesHau97-for-G}
\text{S-G-dim} \left ( \mathcal{F}^{\left ( \eta \right )} \right )
\leqslant \epsilon\text{-G-dim} \left ( \mathcal{F} \right ) \\
\label{eq:extended-lemma-3.2-in-AloBenCesHau97-for-N}
\text{S-N-dim} \left ( \mathcal{F}^{\left ( \eta \right )} \right )
\leqslant \epsilon\text{-N-dim} \left ( \mathcal{F} \right ).
\end{numcases}
\end{lemma}

\begin{proof}
To prove \eqref{eq:extended-lemma-3.2-in-AloBenCesHau97-for-G},
it is enough to establish that any set strongly G-shattered by
$\mathcal{F}^{(\eta)}$
is also G-shattered with margin $\frac{\eta}{2}$ by $\mathcal{F}$.
Suppose that the subset $s_{\bar{\mathcal{Z}}^n} =
\left \{ z_i = \left ( x_i, y_i \right ): \;
1 \leqslant i \leqslant n \right \}$ of $\bar{\mathcal{Z}}$
is strongly G-shattered by $\mathcal{F}^{(\eta)}$.
Then, according to Definitions~\ref{def:Strong-Psi-dimension}
and \ref{def:eta-discretization-operator},
there exist a vector
$\mathbf{b}_n =  \left ( b_i \right )_{1 \leqslant i \leqslant n}
\in \left \llbracket 0;
\left \lfloor \frac{M_{\mathcal{F}}}{\eta} \right \rfloor -1
\right \rrbracket^n$
and a set
$\left \{ f_{\mathbf{s}_n}:
\mathbf{s}_n \in \left \{ -1, 1 \right \}^n \right \} \subset \mathcal{F}$
such that
$$
\forall \mathbf{s}_n = \left ( s_i \right )_{1 \leqslant i \leqslant n}
\in \left \{ -1, 1 \right \}^n, \;
\forall i \in \llbracket 1; n \rrbracket, \;\;
\begin{cases}
\text{if } s_i = 1, \;
f_{\mathbf{s}_n}^{\left ( \eta \right )}
\left ( x_i, y_i \right ) - b_i \geqslant 1 \\
\text{if } s_i = -1, \;
\max_{k \neq y_i} f_{\mathbf{s}_n}^{\left ( \eta \right )}
\left ( x_i, k \right ) + b_i \geqslant 1
\end{cases}.
$$
As a consequence, a proof is obtained by exhibiting a vector
$\mathbf{b}^{\prime}_n =  \left ( b^{\prime}_i
\right )_{1 \leqslant i \leqslant n} \in \mathbb{R}_+^n$
such that
$$
\forall \mathbf{s}_n = \left ( s_i \right )_{1 \leqslant i \leqslant n}
\in \left \{ -1, 1 \right \}^n, \;
\forall i \in \llbracket 1; n \rrbracket, \;\;
\begin{cases}
\text{if } s_i = 1, \;
f_{\mathbf{s}_n} \left ( x_i, y_i \right )
- b^{\prime}_i \geqslant \frac{\eta}{2} \\
\text{if } s_i = -1, \;
\max_{k \neq y_i} f_{\mathbf{s}_n} \left ( x_i, k \right )
+ b^{\prime}_i \geqslant \frac{\eta}{2}
\end{cases}.
$$
A feasible solution consists in setting
$\mathbf{b}^{\prime}_n =  \left ( \eta \left ( b_i + \frac{1}{2} \right )
\right )_{1 \leqslant i \leqslant n}$.
The same line of reasoning can be used to prove
\eqref{eq:extended-lemma-3.2-in-AloBenCesHau97-for-N}.
In short, if $\left ( \mathbf{b}_n, \mathbf{c}_n \right )$
is a witness to the strong N-shattering, then
$\left ( \mathbf{b}^{\prime}_n, \mathbf{c}^{\prime}_n \right )$
is a witness to the $\frac{\eta}{2}$-N-shattering,
where the vector $\mathbf{b}^{\prime}_n$ is deduced from $\mathbf{b}_n$
as above and $\mathbf{c}^{\prime}_n = \mathbf{c}_n$.
\end{proof}

\begin{lemma}
\label{lemma:from-separation-to-strong-G-and-N-shattering}
Let $\mathcal{G}$ be a function class
satisfying Definition~\ref{def:margin-multi-category-classifiers}
and $\rho_{\mathcal{G}}$ the function class deduced
from $\mathcal{G}$ according to
Definition~\ref{def:class-of-margin-functions}.
For $\gamma \in \left ( 0, 1 \right ]$,
let $\rho_{\mathcal{G}, \gamma}$ be the function class deduced
from $\mathcal{G}$ according to
Definition~\ref{def:class-of-regret-functions}.
Suppose that there exist $\left ( g, g^{\prime} \right ) \in \mathcal{G}^2$,
$\gamma \in \left ( 0, 1 \right ]$,
$\eta \in \left ( 0, \frac{\gamma}{2} \right ]$
and $z = \left ( x, y \right ) \in \mathcal{Z}$ such that
\begin{equation}
\label{eq:from-separation-to-strong-G-and-N-shattering}
\rho_{g, \gamma}^{\left ( \eta \right )} \left ( z \right ) -
\rho_{g^{\prime}, \gamma}^{\left ( \eta \right )} \left ( z \right )
\geqslant 2.
\end{equation}
For every
$b \in \left \llbracket
\rho_{g^{\prime}, \gamma}^{\left ( \eta \right )} \left ( z \right ) + 1;
\rho_{g, \gamma}^{\left ( \eta \right )} \left ( z \right ) - 1
\right \rrbracket$
and $c \in \argmax_{k \neq y} \rho_{g^{\prime}}^{\left ( \eta \right )}
\left ( x, k \right )$,
\begin{enumerate}
\item the set $\left \{ \rho_{g, \gamma}^{\left ( \eta \right )},
\rho_{g^{\prime}, \gamma}^{\left ( \eta \right )} \right \}$
strongly shatters the pair
$\left ( \left \{ z \right \}, b \right )$;
\item the set $\left \{ \rho_g^{\left ( \eta \right )},
\rho_{g^{\prime}}^{\left ( \eta \right )} \right \}$
strongly G-shatters the same pair
and strongly N-shatters the triplets
$\left ( \left \{ z \right \}, b, c \right )$.
\end{enumerate}
\end{lemma}

\begin{proof}
The first assertion is obvious, since by definition of $b$,
$$
\begin{cases}
\rho_{g, \gamma}^{\left ( \eta \right )} \left ( z \right ) -b \geqslant 1 \\
-\rho_{g^{\prime}, \gamma}^{\left ( \eta \right )} \left ( z \right )
+b \geqslant 1
\end{cases}.
$$
It springs from \eqref{eq:from-separation-to-strong-G-and-N-shattering} that
$\rho_{g, \gamma}^{\left ( \eta \right )} \left ( z \right ) > 0$,
which implies that
$\rho_g \left ( z \right )
\geqslant \rho_{g, \gamma} \left ( z \right )$
and thus
$\rho_g^{\left ( \eta \right )} \left ( z \right )
\geqslant \rho_{g, \gamma}^{\left ( \eta \right )} \left ( z \right )$.
Consequently,
\begin{align*}
\rho_g^{\left ( \eta \right )} \left ( z \right ) - b
& \geqslant \;
\rho_{g, \gamma}^{\left ( \eta \right )} \left ( z \right ) - b \\
& \geqslant \; 1.
\end{align*}
Furthermore, if
$\rho_{g^{\prime}}^{\left ( \eta \right )} \left ( z \right ) \geqslant 0$, then
$\rho_{g^{\prime}, \gamma}^{\left ( \eta \right )} \left ( z \right )
< \left \lfloor \frac{\gamma}{\eta} \right \rfloor$
implies that
$\rho_{g^{\prime}} \left ( z \right )
= \rho_{g^{\prime}, \gamma} \left ( z \right )$,
with the consequence that
$\rho_{g^{\prime}}^{\left ( \eta \right )} \left ( z \right )
= \rho_{g^{\prime}, \gamma}^{\left ( \eta \right )} \left ( z \right )$.
Then,
$\max_{k \neq y} \rho_{g^{\prime}}^{\left ( \eta \right )} \left ( x, k \right )
= -\rho_{g^{\prime}}^{\left ( \eta \right )} \left ( z \right )
= -\rho_{g^{\prime}, \gamma}^{\left ( \eta \right )} \left ( z \right )$.
Otherwise,
$\max_{k \neq y} \rho_{g^{\prime}}^{\left ( \eta \right )} \left ( x, k \right )
\geqslant 0 
= -\rho_{g^{\prime}, \gamma}^{\left ( \eta \right )} \left ( z \right )$.
Thus, in both cases,
$\max_{k \neq y}
\rho_{g^{\prime}}^{\left ( \eta \right )} \left ( x, k \right ) \geqslant
-\rho_{g^{\prime}, \gamma}^{\left ( \eta \right )} \left ( z \right )$,
leading to
\begin{align}
\max_{k \neq y}
\rho_{g^{\prime}}^{\left ( \eta \right )} \left ( x, k \right ) + b
& \geqslant \;
-\rho_{g^{\prime}, \gamma}^{\left ( \eta \right )} \left ( z \right ) + b 
\nonumber \\
\label{eq:from-separation-to-strong-G-and-N-shattering-part-1}
& \geqslant \; 1.
\end{align}
The strong G-shattering of $\left ( \left \{ z \right \}, b \right )$ by
$\left \{ \rho_g^{\left ( \eta \right )},
\rho_{g^{\prime}}^{\left ( \eta \right )} \right \}$
has been established.
The strong N-shattering of $\left ( \left \{ z \right \}, b, c \right )$
springs from \eqref{eq:from-separation-to-strong-G-and-N-shattering-part-1}
and the definition of $c$.
\end{proof}

\begin{remark}
\label{remark:necessity-of-Pi_gamma}
It is noticeable that the proof of
Lemma~\ref{lemma:from-separation-to-strong-G-and-N-shattering}
makes use of the inequality
$$
\max_{k \neq y}
\rho_{g^{\prime}}^{\left ( \eta \right )} \left ( x, k \right ) \geqslant
-\rho_{g^{\prime}, \gamma}^{\left ( \eta \right )} \left ( z \right ),
$$
and we have even
$$
\max_{k \neq y}
\rho_{g^{\prime}} \left ( x, k \right ) \geqslant
-\rho_{g^{\prime}, \gamma} \left ( z \right ),
$$
which is not intuitive since the key argument establishing that
$\gamma\text{-G-dim} \left ( \rho_{\mathcal{G}} \right )
\leqslant \gamma\text{-dim} \left ( \rho_{\mathcal{G}} \right )$ is
$$
\forall g \in \mathcal{G}, \forall z \in \mathcal{Z}, \;\;
\max_{k \neq y}
\rho_g \left ( x, k \right ) \leqslant
-\rho_g \left ( z \right ).
$$
What made possible this ``inversion of the inequality'' 
is the introduction of the squashing function $\pi_{\gamma}$.
This basic observation highlights the fact that
this operator plays a central part in our thesis 
stating the usefulness of the $\gamma$-$\Psi$-dimensions.
The definition of the bias $b$ is coherent with
the restriction of the domain of the vectors $\mathbf{b}_n$
to the positive hyperoctant.
\end{remark}

\subsection{Margin Graph Dimension - Uniform Convergence Norm}

The proof of 
Lemma~\ref{lemma:extended-Lemma-3.5-in-AloBenCesHau97-gamma-G-dim-optimized}
borrows from the proofs of classical results, including
the two state-of-the-art combinatorial results:
Lemma~3.5 in \citet{AloBenCesHau97} and Theorem~1 in \citet{MenVer03}.
Central in this proof is the following basic combinatorial result.

\begin{lemma}
\label{lemma:discrete-combinatorial-result-for-G}
Let $\mathcal{G}$ be a function class
satisfying Definition~\ref{def:margin-multi-category-classifiers}
and $\rho_{\mathcal{G}}$ the function class deduced
from $\mathcal{G}$ according to
Definition~\ref{def:class-of-margin-functions}.
For $\gamma \in \left ( 0, 1 \right ]$,
let $\rho_{\mathcal{G}, \gamma}$ be the function class deduced
from $\mathcal{G}$ according to
Definition~\ref{def:class-of-regret-functions}.
For $\tilde{\mathcal{G}} \subset \mathcal{G}$, $s_{\mathcal{Z}^n}
= \left \{ z_i = \left ( x_i, y_i \right ): 
1 \leqslant i \leqslant n \right \} \subset \mathcal{Z}$,
$\gamma \in \left ( 0, 1 \right ]$ and
$\eta \in \left ( 0, \frac{\gamma}{2} \right ]$,
let $\mathcal{F}_{\gamma} =
\left ( \left . \rho_{\tilde{\mathcal{G}}, \gamma}
\right |_{s_{\mathcal{Z}^n}} 
\right )^{\left ( \eta \right )}$
and let $\mathcal{F} =
\left ( \left . \rho_{\tilde{\mathcal{G}}}
\right |_{\mathcal{S} \left ( s_{\mathcal{Z}^n} \right )} 
\right )^{\left ( \eta \right )}$
with $\mathcal{S} \left ( s_{\mathcal{Z}^n} \right ) = 
\left \{ \left ( x_i, k \right ): \left ( i, k \right ) 
\in \left \llbracket 1; n \right \rrbracket \times 
\left \llbracket 1; C \right \rrbracket \right \}$.
If $\mathcal{F}_{\gamma}$
is $2$-separated in the pseudo-metric $d_{\infty, \mathbf{z}_n}$,
then
\begin{equation}
\label{eq:discrete-combinatorial-result-for-G}
\left | \mathcal{F}_{\gamma} \right | \leqslant
\left ( 3 M_{\gamma} n \right )^{\log_2 \left ( \Sigma \right )},
\end{equation}
where
$\Sigma = \sum_{u=0}^{d_G} {n \choose u} M_{\gamma}^u$
with $M_{\gamma} = \left \lfloor \frac{\gamma}{\eta} \right \rfloor$
and $d_G$ is the maximal cardinality of a subset of $s_{\mathcal{Z}^n}$
strongly G-shattered by $\mathcal{F}$.
\end{lemma}

\begin{proof}
Notice first that 
Inequality~\eqref{eq:discrete-combinatorial-result-for-G}
is trivially true for $\left | \mathcal{F}_{\gamma} \right | = 1$.
Indeed, the minimal value of its right-hand side,
corresponding to $d_G = 0$, is 1.
Thus, the rest of the proof makes use of the restriction
$\left | \mathcal{F}_{\gamma} \right | \geqslant 2$.
A direct consequence is that according to
Lemma~\ref{lemma:from-separation-to-strong-G-and-N-shattering},
$d_G \geqslant 1$.
Since two examples $z_i$ and $z_j$ can be such that
$x_i = x_j$ (provided that $y_i \neq y_j$), then it is possible that
$\left | \mathcal{S} \left ( s_{\mathcal{Z}^n} \right ) \right | < Cn$.
Furthermore,
$s_{\mathcal{Z}^n} \subset \mathcal{S} \left ( s_{\mathcal{Z}^n} \right )$
implies that
$d_G \leqslant \text{S-G-dim} \left ( \mathcal{F} \right )$.
A subset of $s_{\mathcal{Z}^n}$ of cardinality 
$u \in \left \llbracket 1; n \right \rrbracket$ is denoted by 
$s_{\mathcal{Z}^u}^{\prime} = \left \{
z_i^{\prime}: \; 1 \leqslant i \leqslant u \right \}$, with the convention
$$
\forall \left ( i, j \right ): 1 \leqslant i < j \leqslant u, \;
\left ( z_i^{\prime}, z_j^{\prime} \right )
= \left ( z_v, z_w \right ) \Longrightarrow 1 \leqslant v < w \leqslant n.
$$
For every subset $\bar{\mathcal{G}}$ of $\tilde{\mathcal{G}}$,
denote by $s \left ( \bar{\mathcal{G}} \right )$ the number of pairs
$\left ( s_{\mathcal{Z}^u}^{\prime}, \mathbf{b}_u^{\prime} \right )$
with $s_{\mathcal{Z}^u}^{\prime} \subset s_{\mathcal{Z}^n}$ and
$\mathbf{b}_u^{\prime} \in
\left \llbracket 1 ; M_{\gamma} -1
\right \rrbracket^u$
strongly G-shattered by $\bar{\mathcal{F}} =
\left ( \left . \rho_{\bar{\mathcal{G}}}
\right |_{\mathcal{S} \left ( s_{\mathcal{Z}^n} \right )} 
\right )^{\left ( \eta \right )}$
(the convention above has been introduced to avoid handling duplicates).
Since $d_G \geqslant 1$, combinatorics gives:
\begin{equation}
\label{eq:discrete-combinatorial-result-for-G-partial-1}
s \left ( \tilde{\mathcal{G}} \right ) \leqslant 
\sum_{u=1}^{d_G} {n \choose u} M_{\gamma}^u = \Sigma -1.
\end{equation}
In order to derive a lower bound on
$s \left ( \tilde{\mathcal{G}} \right )$,
we build a $2$-separating tree of $\mathcal{F}_{\gamma}$
\citep[see Definition~3.4 in][]{RudVer06}.
Let $\bar{\mathcal{F}}_{\gamma} 
= \left ( \left . \rho_{\bar{\mathcal{G}}, \gamma}
\right |_{s_{\mathcal{Z}^n}} 
\right )^{\left ( \eta \right )}$
be one of its nodes such
that $\left | \bar{\mathcal{F}}_{\gamma} \right | \geqslant 2$ (inner node).
Its two sons, $\bar{\mathcal{F}}_{\gamma,+}$ and $\bar{\mathcal{F}}_{\gamma,-}$,
are built as follows.
Split $\bar{\mathcal{F}}_{\gamma}$ arbitrarily into
$\left \lfloor \frac{\left | \bar{\mathcal{F}}_{\gamma} \right |}{2} 
\right \rfloor$ pairs
(with possibly a function remaining alone).
For each pair
$\left ( f_{\gamma}, f_{\gamma}^{\prime} \right )$,
find $z_i \in s_{\mathcal{Z}^n}$ such that
$\left | f_{\gamma} \left ( z_i \right ) 
- f_{\gamma}^{\prime} \left ( z_i \right ) \right | \geqslant 2$.
By the pigeonhole principle, the same example is  picked for at least
$\left \lceil
\left \lfloor \frac{\left | \bar{\mathcal{F}}_{\gamma} \right |}{2}
\right \rfloor \frac{1}{n} \right \rceil$ pairs.
Let $z_{i_0}$ be such an example,
and let $\left ( f_{\gamma,+}, f_{\gamma,-} \right )$
denote the corresponding pairs, whose components are reordered (when needed)
so that
$$
f_{\gamma,+} \left ( z_{i_0} \right ) > 
f_{\gamma,-} \left ( z_{i_0} \right ).
$$
Among the functions $f_{\gamma,+}$, at least
$\left \lceil \left \lceil
\left \lfloor \frac{\left | \bar{\mathcal{F}}_{\gamma} \right |}{2}
\right \rfloor \frac{1}{n} \right \rceil \frac{1}{M_{\gamma} -1} \right \rceil$
take the same value at $z_{i_0}$.
Let $v \left ( z_{i_0} \right )$ be such a value.
We define $\bar{\mathcal{F}}_{\gamma,+}$ (resp. $\bar{\mathcal{F}}_{\gamma,-}$)
to be the set of functions $f_{\gamma,+}$ (resp. $f_{\gamma,-}$)
belonging to a pair associated with 
$\left ( z_{i_0}, v \left ( z_{i_0} \right ) \right )$.
By construction, their common cardinality is bounded from below by:
\begin{equation}
\label{eq:discrete-combinatorial-result-for-G-partial-0}
\left | \bar{\mathcal{F}}_{\gamma,+} \right |
= \left | \bar{\mathcal{F}}_{\gamma,-} \right |
\geqslant \frac{\left | \bar{\mathcal{F}}_{\gamma} \right |}{3 M_{\gamma} n}.
\end{equation}
Let $\bar{\mathcal{G}}_+$ and $\bar{\mathcal{G}}_-$ be two subsets of
$\bar{\mathcal{G}}$ in bijection with
$\bar{\mathcal{F}}_{\gamma,+}$ and
$\bar{\mathcal{F}}_{\gamma, -}$ respectively, such that
$\bar{\mathcal{F}}_{\gamma, +} =
\left ( \left . \rho_{\bar{\mathcal{G}}_+, \gamma}
\right |_{s_{\mathcal{Z}^n}} 
\right )^{\left ( \eta \right )}$ and
$\bar{\mathcal{F}}_{\gamma, -} =
\left ( \left . \rho_{\bar{\mathcal{G}}_-, \gamma}
\right |_{s_{\mathcal{Z}^n}} 
\right )^{\left ( \eta \right )}$.
Let $\bar{\mathcal{F}}_+ =
\left ( \left . \rho_{\bar{\mathcal{G}}_+}
\right |_{\mathcal{S} \left ( s_{\mathcal{Z}^n} \right )} 
\right )^{\left ( \eta \right )}$ and
$\bar{\mathcal{F}}_- =
\left ( \left . \rho_{\bar{\mathcal{G}}_-}
\right |_{\mathcal{S} \left ( s_{\mathcal{Z}^n} \right )} 
\right )^{\left ( \eta \right )}$.
According to Lemma~\ref{lemma:from-separation-to-strong-G-and-N-shattering},
the sets $\bar{\mathcal{F}}_+$ and $\bar{\mathcal{F}}_-$ satisfy:
\begin{equation}
\label{eq:discrete-combinatorial-result-for-G-partial-2}
\begin{cases}
\forall f_+ \in \bar{\mathcal{F}}_+, \;\; 
f_+ \left ( z_{i_0} \right ) - b_{i_0} \geqslant 1 \\
\forall f_- \in \bar{\mathcal{F}}_-, \;\;
\max_{k \neq y_{i_0}} f_- \left ( x_{i_0}, k \right ) + b_{i_0} \geqslant 1
\end{cases}
\end{equation}
with $b_{i_0} = v \left ( z_{i_0} \right ) -1$.
Let $\bar{\mathcal{F}} =
\left ( \left . \rho_{\bar{\mathcal{G}}}
\right |_{\mathcal{S} \left ( s_{\mathcal{Z}^n} \right )} 
\right )^{\left ( \eta \right )}$.
Since 
$\bar{\mathcal{F}}_+ \bigcup \bar{\mathcal{F}}_- \subset \bar{\mathcal{F}}$,
obviously, any pair strongly G-shattered by either
$\bar{\mathcal{F}}_+$ or $\bar{\mathcal{F}}_-$
is also strongly G-shattered by $\bar{\mathcal{F}}$.
Furthermore, according to
\eqref{eq:discrete-combinatorial-result-for-G-partial-2},
$\bar{\mathcal{F}}$ strongly G-shatters the pair
$\left ( \left \{ z_{i_0} \right \}, b_{i_0} \right )$ 
which is strongly G-shattered by
neither $\bar{\mathcal{F}}_+$ nor $\bar{\mathcal{F}}_-$.
At last, let us consider any pair
$\left ( s_{\mathcal{Z}^u}^{\prime}, \mathbf{b}_u^{\prime} \right )$
strongly G-shattered by both $\bar{\mathcal{F}}_+$
and $\bar{\mathcal{F}}_-$. Let the pair
$\left ( s_{\mathcal{Z}^{u+1}}^{\prime\prime},
\mathbf{b}_{u+1}^{\prime\prime} \right )$
be such that $s_{\mathcal{Z}^{u+1}}^{\prime\prime} =
s_{\mathcal{Z}^u}^{\prime} \bigcup \left \{ z_{i_0} \right \}$
and the vector
$\mathbf{b}_{u+1}^{\prime\prime}$ is deduced from $\mathbf{b}_u^{\prime}$
by inserting the component $b_{i_0}$ at the right place.
Clearly, neither $\bar{\mathcal{F}}_+$ nor $\bar{\mathcal{F}}_-$
strongly G-shatters
$\left ( s_{\mathcal{Z}^{u+1}}^{\prime\prime},
\mathbf{b}_{u+1}^{\prime\prime} \right )$,
simply because they do not strongly G-shatter
the pair $\left ( \left \{ z_{i_0} \right \}, b_{i_0} \right )$.
On the contrary, it springs once more from
\eqref{eq:discrete-combinatorial-result-for-G-partial-2}
that $\left ( s_{\mathcal{Z}^{u+1}}^{\prime\prime},
\mathbf{b}_{u+1}^{\prime\prime} \right )$
is strongly G-shattered by $\bar{\mathcal{F}}$.
Summarizing, for each pair
$\left ( s_{\mathcal{Z}^u}^{\prime}, \mathbf{b}_u^{\prime} \right )$
strongly G-shattered by both $\bar{\mathcal{F}}_+$
and $\bar{\mathcal{F}}_-$, we can exhibit by means of an injective mapping
a pair
$\left ( s_{\mathcal{Z}^{u+1}}^{\prime\prime},
\mathbf{b}_{u+1}^{\prime\prime} \right )$
strongly G-shattered by $\bar{\mathcal{F}}$ but not by
$\bar{\mathcal{F}}_+$ or $\bar{\mathcal{F}}_-$.
Collecting all terms, we obtain
\begin{align}
s \left ( \bar{\mathcal{G}} \right ) 
& \geqslant \;
s \left ( \bar{\mathcal{G}}_+ \right ) +
s \left ( \bar{\mathcal{G}}_- \right ) + 1 \nonumber \\
\label{eq:discrete-combinatorial-result-for-G-partial-3}
& \geqslant \; 
\ell \left ( \bar{\mathcal{F}}_{\gamma} \right ) -1,
\end{align}
where the function $\ell$ returns the number
of leaves of the (sub)tree whose root is its argument.
Thus, finishing the proof boils down to exhibiting the appropriate lower
bound on $\ell \left ( \bar{\mathcal{F}}_{\gamma} \right )$.
To that end, we proceed by induction on the depth of the node.
The hypothesis is that
\begin{equation}
\label{eq:discrete-combinatorial-result-for-G-partial-4}
\ell \left ( \bar{\mathcal{F}}_{\gamma} \right ) \geqslant
\left | \bar{\mathcal{F}}_{\gamma} 
\right |^{\frac{1}{\log_2 \left ( 3 M_{\gamma} n \right )}}.
\end{equation}
It is obviously true for the leaves (which are of cardinality $1$).
Suppose now that it is true for the two sons of an inner node. Then,
Inequality~\eqref{eq:discrete-combinatorial-result-for-G-partial-0} gives:
\begin{align*}
\ell \left ( \bar{\mathcal{F}}_{\gamma} \right )
& = \; 
\ell \left ( \bar{\mathcal{F}}_{\gamma,+} \right ) +
\ell \left ( \bar{\mathcal{F}}_{\gamma,-} \right ) \\
& \geqslant \;
2 \left ( \frac{\left | \bar{\mathcal{F}}_{\gamma} \right |}{3 M_{\gamma} n}
\right )^{\frac{1}{\log_2 \left ( 3 M_{\gamma} n \right )}} \\
& = \;
\left | \bar{\mathcal{F}}_{\gamma}
\right |^{\frac{1}{\log_2 \left ( 3 M_{\gamma} n \right )}}.
\end{align*}
The induction hypothesis has been proved.
Combining Inequalities~\eqref{eq:discrete-combinatorial-result-for-G-partial-1},
\eqref{eq:discrete-combinatorial-result-for-G-partial-3} and
\eqref{eq:discrete-combinatorial-result-for-G-partial-4}
produces by transitivity:
$$
\left | \mathcal{F}_{\gamma}
\right |^{\frac{1}{\log_2 \left ( 3 M_{\gamma} n \right )}}
\leqslant \Sigma,
$$
or equivalently
\begin{align*}
\left | \mathcal{F}_{\gamma} \right |
& \leqslant \;
\Sigma^{\log_2 \left ( 3 M_{\gamma} n \right )} \\
& = \; \left ( 3 M_{\gamma} n \right )^{\log_2 \left ( \Sigma \right )},
\end{align*}
i.e., Inequality~\eqref{eq:discrete-combinatorial-result-for-G},
the result announced.
\end{proof}
With Lemma~\ref{lemma:discrete-combinatorial-result-for-G} at hand,
the proof of
Lemma~\ref{lemma:extended-Lemma-3.5-in-AloBenCesHau97-gamma-G-dim-optimized}
is straightforward.

\begin{proof}
Let us consider any vector $\mathbf{z}_n \in \mathcal{Z}^n$ and let
$s_{\mathcal{Z}^n} = \left \{ z_i: \; 1 \leqslant i \leqslant n \right \}$ 
be the smallest subset of $\mathcal{Z}$
containing all the components of $\mathbf{z}_n$.
Note that its cardinality can be strictly inferior to $n$, in case that
$\mathbf{z}_n$ has two identical components.
By definition,
$$
\mathcal{M} \left ( \epsilon, \rho_{\mathcal{G}, \gamma},
d_{\infty, \mathbf{z}_n} \right )
= \mathcal{M} \left ( \epsilon, \left .
\rho_{\mathcal{G}, \gamma} \right |_{s_{\mathcal{Z}^n}},
d_{\infty, \mathbf{z}_n} \right ).
$$
Furthermore, setting $\eta = \frac{\epsilon}{2}$
in \eqref{eq:L_infty-norm-lemma-3.2-in-AloBenCesHau97}, one obtains:
$$
\mathcal{M} \left ( \epsilon,
\left . \rho_{\mathcal{G}, \gamma}
\right |_{s_{\mathcal{Z}^n}}, d_{\infty, \mathbf{z}_n} \right )
\leqslant
\mathcal{M} \left ( 2,
\left ( \left . \rho_{\mathcal{G}, \gamma} \right |_{s_{\mathcal{Z}^n}}
\right )^{(\frac{\epsilon}{2})}, d_{\infty, \mathbf{z}_n} \right ).
$$
Let $\mathcal{S} \left ( s_{\mathcal{Z}^n} \right ) = 
\left \{ \left ( x_i, k \right ): \left ( i, k \right ) 
\in \left \llbracket 1; n \right \rrbracket \times 
\left \llbracket 1; C \right \rrbracket \right \}$.
The packing numbers of 
$\left ( \left . \rho_{\mathcal{G}, \gamma} \right |_{s_{\mathcal{Z}^n}}
\right )^{(\frac{\epsilon}{2})}$
can be upper bounded thanks to
Lemma~\ref{lemma:discrete-combinatorial-result-for-G}, leading to
\begin{equation}
\label{eq:extended-Lemma-3.5-in-AloBenCesHau97-gamma-G-dim-optimized-partial-1}
\mathcal{M} \left ( 2,
\left ( \left . \rho_{\mathcal{G}, \gamma} \right |_{s_{\mathcal{Z}^n}}
\right )^{(\frac{\epsilon}{2})}, d_{\infty, \mathbf{z}_n} \right )
\leqslant 
\left ( \frac{6 \gamma n}{\epsilon} \right )^{\log_2 \left ( \Sigma \right )}
\end{equation}
where
$\Sigma = \sum_{u=0}^{d_G} {n \choose u}
\left ( \frac{2 \gamma}{\epsilon} \right )^u$
with $d_G$ being the maximal cardinality of a subset of $s_{\mathcal{Z}^n}$
strongly G-shattered by
$\left ( \left . \rho_{\mathcal{G}} 
\right |_{\mathcal{S} \left ( s_{\mathcal{Z}^n} \right )}
\right )^{(\frac{\epsilon}{2})}$.
According to \eqref{eq:extended-lemma-3.2-in-AloBenCesHau97-for-G},
\begin{align*}
d_G
& \leqslant \;
\text{S-G-dim} \left ( \left (
\left . \rho_{\mathcal{G}} 
\right |_{\mathcal{S} \left ( s_{\mathcal{Z}^n} \right )}
\right )^{(\frac{\epsilon}{2})} \right ) \\
& \leqslant \;
\left ( \frac{\epsilon}{4}\right )\mbox{-G-dim} \left (
\left . \rho_{\mathcal{G}} \right 
|_{\mathcal{S} \left ( s_{\mathcal{Z}^n} \right )} \right ) \\
& \leqslant \;
d_G \left ( \frac{\epsilon}{4} \right ).
\end{align*}
Since by hypothesis,
$n \geqslant d_G \left ( \frac{\epsilon}{4} \right )$,
$\Sigma$ can be bounded from above by replacing in its formula
$d_G$ with $d_G \left ( \frac{\epsilon}{4} \right )$ and resorting
to Corollary~3.3 in \citet{MohRosTal12}, leading to:
\begin{align}
\Sigma
& \leqslant \;
\sum_{u=0}^{d_G \left ( \frac{\epsilon}{4} \right )} {n \choose u} 
\left ( \frac{2 \gamma}{\epsilon} \right )^u \nonumber \\
\label{eq:extended-Lemma-3.5-in-AloBenCesHau97-gamma-G-dim-optimized-partial-2}
& \leqslant \;
\left ( \frac{2 \gamma e n}
{d_G \left ( \frac{\epsilon}{4} \right ) \epsilon} 
\right )^{d_G \left ( \frac{\epsilon}{4} \right )},
\end{align}
where the standard convention that the last term takes the value $1$ for
$d_G \left ( \frac{\epsilon}{4} \right ) = 0$ is made.
Substituting
\eqref{eq:extended-Lemma-3.5-in-AloBenCesHau97-gamma-G-dim-optimized-partial-2}
into 
\eqref{eq:extended-Lemma-3.5-in-AloBenCesHau97-gamma-G-dim-optimized-partial-1} 
and taking the supremum over $\mathcal{Z}^n$ concludes the proof of
\eqref{eq:extended-Lemma-3.5-in-AloBenCesHau97-gamma-G-dim-optimized}.
\end{proof}

\subsection{Margin Graph Dimension - $L_2$-norm}

The sketch of the proof of the two $L_2$-norm combinatorial results,
Lemma~\ref{lemma:extended-Theorem-1-in-MenVer03} and
Lemma~\ref{lemma:extended-Theorem-1-in-MenVer03-gamma-N-dim},
is basically the same. Compared to the sketch of the proof of
Lemma~\ref{lemma:extended-Lemma-3.5-in-AloBenCesHau97-gamma-G-dim-optimized},
it exhibits two major differences. 
First, the construction of the $2$-separating tree 
is more sophisticated, since it rests on a small deviation principle 
(in place of the sole pigeonhole principle).
Second, one additional step is involved,
which implements a probabilistic extraction principle.
This additional step makes the result dimension free.
We begin the proof with the formulation of the small deviation principle.
This extension of Lemma~5 in \citet{MenVer03} is tailored to our needs.

\begin{lemma}
\label{lemma:dedicated-Lemma-5-in-MenVer03}
Let $T$ be a random variable taking values in $\left \llbracket 0;
M \right \rrbracket$ with $M \geqslant 2$.
Suppose that
$\text{Var} \left [ T \right ] \geqslant 9$.
Then there exists either $\left ( \alpha, \beta \right )
\in \left \llbracket 1; M-1 \right \rrbracket \times 
\left [ \frac{1}{M^2}, \frac{1}{2} \right ]$
such that
$$
\begin{cases}
\mathbb{P} \left \{ T \geqslant \alpha + 1 \right \} 
\geqslant \max \left \{ \frac{1}{2} \beta, \frac{1}{M^2} \right \} \\
\mathbb{P} \left \{ T \leqslant \alpha - 1 \right \} \geqslant 1 - \beta
\end{cases}
$$
or $\left ( \alpha', \beta' \right )
\in \left \llbracket 1; M-1 \right \rrbracket \times 
\left [ \frac{1}{M^2}, \frac{1}{2} \right ]$
such that
$$
\begin{cases}
\mathbb{P} \left \{ T \geqslant \alpha' + 1 \right \} 
\geqslant 1 - \beta' \\
\mathbb{P} \left \{ T \leqslant \alpha' - 1 \right \} 
\geqslant \max \left \{ \frac{1}{2} \beta', \frac{1}{M^2} \right \}
\end{cases}.
$$
\end{lemma}

\begin{proof}
We first note that the hypothesis
$\text{Var} \left [ T \right ] \geqslant 9$ implies that $M \geqslant 6$.
Let $M_T$ be the smallest median of $T$ belonging to
$\left \llbracket 0; M \right \rrbracket$.
Then, several cases must be distinguished, according to
the values of $M_T$ and $M - M_T$.
Since they can all be treated the same
and the one implying the largest upper bound on the variance, i.e.,
the one from which springs the hypothesis on the variance, is
$M - M_T \geqslant 2$ and $M_T \geqslant 2$,
we focus on it in the sequel.
Let us define the sequences
$\left ( \beta_k \right )_{k \in \mathbb{N}^*}$ and
$\left ( \beta_k' \right )_{k \in \mathbb{N}^*}$ as follows:
$$
\forall k \in \mathbb{N}^*, \;\;
\begin{cases}
\beta_k = \mathbb{P} \left \{ T \geqslant M_T + k \right \} \\
\beta_k' = \mathbb{P} \left \{ T \leqslant M_T - k \right \}
\end{cases}.
$$
Note that by definition of $M_T$, both $\beta_1$ and $\beta_1'$
are inferior or equal to $\frac{1}{2}$.
Assume that the conclusion of the lemma fails. We claim that
$$
\begin{cases}
\forall k \in \left \llbracket 2; M - M_T \right \rrbracket, \;\;
\beta_k \leqslant \max \left \{
\frac{2 \left ( M - M_T \right ) - k + 1}
{\left ( M - M_T \right )^3}, \frac{1}{2^k} \right \} \\
\forall k \in \left \llbracket 2; M_T \right \rrbracket, \;\;
\beta_k' \leqslant
\max \left \{ \frac{2 M_T - k + 1}{M_T^3}, \frac{1}{2^k} \right \}
\end{cases}.
$$
Indeed, assume that
$\beta_k > \max \left \{ \frac{2 \left ( M - M_T \right ) - k + 1}
{\left ( M - M_T \right )^3}, \frac{1}{2^k} \right \}$
for some $k \in \left \llbracket 2; M - M_T \right \rrbracket$
and let $k_0$ be the smallest such index.
By construction, $\beta_{k_0} > \max \left \{ \frac{1}{2} \beta_{k_0-1},
\frac{1}{\left ( M - M_T \right )^2} \right \}$
(even for $k_0 = 2$ and $k_0 = M - M_T$), so that
$$
\begin{cases}
\mathbb{P} \left \{ T \geqslant M_T + k_0 \right \}
= \beta_{k_0} > \max \left \{ \frac{1}{2} \beta_{k_0-1},
\frac{1}{M^2} \right \}\\
\mathbb{P} \left \{ T \leqslant M_T + k_0 -2 \right \}
= 1 - \mathbb{P} \left \{ T \geqslant M_T + k_0 -1 \right \}
= 1 - \beta_{k_0-1}
\end{cases}.
$$
Since $\beta_{k_0-1} \geqslant \beta_{k_0}
> \frac{1}{M^2}$ and $\beta_{k_0-1} \leqslant
\beta_1 \leqslant \frac{1}{2}$, so that
$\beta_{k_0-1} \in \left [ \frac{1}{M^2}, \frac{1}{2} \right ]$,
this implies that the conclusion of the lemma would hold with $\alpha$ being
$M_T + k_0 -1$ and $\beta = \beta_{k_0-1}$, which contradicts the assumption
that the conclusion of the lemma fails. The inequality
$\beta_k' \leqslant \max \left \{
\frac{2 M_T - k + 1}{M_T^3}, \frac{1}{2^k} \right \}$
can be proved in a symmetrical way.
As a consequence, upper bounding the maxima by the corresponding sums gives:
\begin{align*}
\text{Var} \left [ T \right ] 
& = \; \text{Var} \left [ T - M_T \right ] \\
& \leqslant \; \mathbb{E} \left [ \left ( T - M_T \right )^2 \right ] \\
& = \; \sum_{t=0}^{+\infty}
\mathbb{P} \left \{ \left ( T - M_T \right )^2 > t \right \} \\
& = \;
\sum_{t=0}^{+\infty} \left ( \mathbb{P} \left \{ T > M_T + \sqrt{t} \right \} +
\mathbb{P} \left \{ T < M_T - \sqrt{t} \right \} \right ) \\
& = \;
\sum_{k=1}^{M-M_T} \left ( 2k-1 \right ) \beta_k +
\sum_{k=1}^{M_T} \left ( 2k-1 \right ) \beta_k' \\
& < \;
2 \sum_{k=1}^{+\infty} \frac{2k-1}{2^k} +
\sum_{k=2}^{M-M_T} \left ( 2k-1 \right ) 
\frac{2 \left ( M - M_T \right ) - k + 1}
{\left ( M - M_T \right )^3}
+ \sum_{k=2}^{M_T} \left ( 2k-1 \right ) \frac{2 M_T - k + 1} {M_T^3} \\
& \leqslant \; 
6 + 2 \max_{\Delta \in \mathbb{N} \setminus \left \{ 1, 2 \right \}} 
\frac{\left ( 8 \Delta + 11 \right ) \left ( \Delta - 1 \right )}{6 \Delta^2} \\
& < \; 9.
\end{align*}
This is in contradiction with the hypothesis that
$\text{Var} \left [ T \right ] \geqslant 9$ and thus concludes the proof.
\end{proof}

\begin{lemma}
\label{lemma:dedicated-Lemma-6-in-MenVer03}
Let $\mathcal{T} = \left \{ t_i: \; 1 \leqslant i \leqslant n \right \}$
be a finite set
and $\mathbf{t}_n = \left( t_i \right)_{1 \leqslant i \leqslant n}$.
Let $\mathcal{F}$ be a class of functions from $\mathcal{T}$
into $\left \llbracket 0; M_{\mathcal{F}} \right \rrbracket$ 
with $M_{\mathcal{F}} \geqslant 2$.
Suppose that $\mathcal{F}$ is of cardinality at least $2$ and
is $6$-separated in the pseudo-metric $d_{2, \mathbf{t}_n}$.
Then there exist an index $i_0 \in \left \llbracket 1; n \right \rrbracket$
and either $\left ( \alpha, \beta \right )
\in \left \llbracket 1; M_{\mathcal{F}} -1 \right \rrbracket
\times \left [ \frac{1}{M_{\mathcal{F}}^2}, \frac{1}{2} \right ]$
such that
$$
\begin{cases}
\left | \left \{ f \in \mathcal{F}: \; f \left ( t_{i_0} \right ) 
\geqslant \alpha + 1 \right \} \right | 
\geqslant \max \left \{ \frac{1}{2} \beta,
\frac{1}{M_{\mathcal{F}}^2} \right \} \left | \mathcal{F} \right | \\
\left | \left \{ f \in \mathcal{F}: \; f \left ( t_{i_0} \right ) 
\leqslant \alpha - 1 \right \} \right | 
\geqslant \left ( 1 - \beta \right ) \left | \mathcal{F} \right |
\end{cases}
$$
or $\left ( \alpha', \beta' \right )
\in \left \llbracket 1; M_{\mathcal{F}} -1 \right \rrbracket
\times \left [ \frac{1}{M_{\mathcal{F}}^2}, \frac{1}{2} \right ]$
such that
$$
\begin{cases}
\left | \left \{ f \in \mathcal{F}: \; f \left ( t_{i_0} \right ) 
\geqslant \alpha' + 1 \right \} \right | 
\geqslant \left ( 1 - \beta' \right ) \left | \mathcal{F} \right | \\
\left | \left \{ f \in \mathcal{F}: \; f \left ( t_{i_0} \right ) 
\leqslant \alpha' - 1 \right \} \right | 
\geqslant \max \left \{ \frac{1}{2} \beta',
\frac{1}{M_{\mathcal{F}}^2} \right \} \left | \mathcal{F} \right |
\end{cases}.
$$
\end{lemma}

\begin{proof}
Let us endow $\mathcal{F}$ with the uniform (counting) measure.
Then, the separation assumption on $\mathcal{F}$ can be used to
derive a lower bound on
$\mathbb{E} \left [ d_{2, \mathbf{t}_n}^2 \left ( f, f' \right ) \right ]$.
Indeed, with probability $1 - \left | \mathcal{F} \right |^{-1}$
we have $f \neq f'$ and, whenever this event occurs,
$d_{2, \mathbf{t}_n} \left ( f, f' \right ) \geqslant 6$. As a consequence,
\begin{align*}
\mathbb{E} \left [ d_{2, \mathbf{t}_n}^2 \left ( f, f' \right ) \right ]
& \geqslant \;
\left ( 1 - \left | \mathcal{F} \right |^{-1} \right ) 36 \\
& \geqslant \; 18.
\end{align*}
Furthermore,
\begin{align*}
\mathbb{E} \left [ d_{2, \mathbf{t}_n}^2 \left ( f, f' \right ) \right ]
& = \;
\frac{1}{n} \sum_{i=1}^n \mathbb{E} \left [ \left ( f \left ( t_i \right )
- f' \left ( t_i \right ) \right )^2 \right ] \\
& = \;
\frac{2}{n} \sum_{i=1}^n \text{Var} \left [ f \left ( t_i \right ) \right ].
\end{align*}
Thus, there exists $i_0 \in \left \llbracket 1; n \right \rrbracket$
such that 
$$
\text{Var} \left [ f \left ( t_{i_0} \right ) \right ]
\geqslant \frac{1}{2}
\mathbb{E} \left [ d_{2, \mathbf{t}_n}^2 \left ( f, f' \right ) \right ]
\geqslant 9.
$$
This implies that the random variable $f \left ( t_{i_0} \right )$
satisfies the hypotheses of Lemma~\ref{lemma:dedicated-Lemma-5-in-MenVer03},
and the conclusion then springs from the application of this lemma.
\end{proof}
Lemma~\ref{lemma:dedicated-Lemma-6-in-MenVer03}
will be used in the proof of the combinatorial result
involving the margin Natarajan dimension:
Lemma~\ref{lemma:extended-Theorem-1-in-MenVer03-gamma-N-dim}. 
However, we established it
in this section, because its proof can be easily simplified to produce
the following variant, appropriate for the margin Graph dimension.

\begin{lemma}
\label{lemma:dedicated-Lemma-6-in-MenVer03-gamma-G-dim}
Let $\mathcal{T} = \left \{ t_i: \; 1 \leqslant i \leqslant n \right \}$
be a finite set
and $\mathbf{t}_n = \left( t_i \right)_{1 \leqslant i \leqslant n}$.
Suppose that $\mathcal{F} \subset \mathbb{Z}^{\mathcal{T}}$ 
is of cardinality at least $2$ and
is $5$-separated in the pseudo-metric $d_{2, \mathbf{t}_n}$.
Then there exist an index $i_0 \in \left \llbracket 1; n \right \rrbracket$
and either $\left ( \alpha, \beta \right )
\in \mathbb{Z} \times \left ( 0, \frac{1}{2} \right ]$
such that
$$
\begin{cases}
\left | \left \{ f \in \mathcal{F}: \; f \left ( t_{i_0} \right )
\geqslant \alpha + 1 \right \} \right |
\geqslant \frac{1}{2} \beta \left | \mathcal{F} \right | \\
\left | \left \{ f \in \mathcal{F}: \; f \left ( t_{i_0} \right )
\leqslant \alpha - 1 \right \} \right |
\geqslant \left ( 1 - \beta \right ) \left | \mathcal{F} \right |
\end{cases}
$$
or $\left ( \alpha', \beta' \right )
\in \mathbb{Z} \times \left ( 0, \frac{1}{2} \right ]$
such that
$$
\begin{cases}
\left | \left \{ f \in \mathcal{F}: \; f \left ( t_{i_0} \right )
\geqslant \alpha' + 1 \right \} \right |
\geqslant \left ( 1 - \beta' \right ) \left | \mathcal{F} \right | \\
\left | \left \{ f \in \mathcal{F}: \; f \left ( t_{i_0} \right )
\leqslant \alpha' - 1 \right \} \right |
\geqslant \frac{1}{2} \beta' \left | \mathcal{F} \right |
\end{cases}.
$$
\end{lemma}
The following lemma is the basic combinatorial result underlying
Lemma~\ref{lemma:extended-Theorem-1-in-MenVer03}.

\begin{lemma}
\label{lemma:extended-Proposition-12-in-MenVer03}
Let $\mathcal{G}$ be a function class
satisfying Definition~\ref{def:margin-multi-category-classifiers}
and $\rho_{\mathcal{G}}$ the function class deduced
from $\mathcal{G}$ according to
Definition~\ref{def:class-of-margin-functions}.
For $\gamma \in \left ( 0, 1 \right ]$,
let $\rho_{\mathcal{G}, \gamma}$ be the function class deduced
from $\mathcal{G}$ according to
Definition~\ref{def:class-of-regret-functions}.
For $\tilde{\mathcal{G}} \subset \mathcal{G}$, $s_{\mathcal{Z}^n}
= \left \{ z_i = \left ( x_i, y_i \right ): 
1 \leqslant i \leqslant n \right \} \subset \mathcal{Z}$,
$\gamma \in \left ( 0, 1 \right ]$ and
$\eta \in \left ( 0, \frac{\gamma}{2} \right ]$,
let $\mathcal{F}_{\gamma} =
\left ( \left . \rho_{\tilde{\mathcal{G}}, \gamma}
\right |_{s_{\mathcal{Z}^n}} 
\right )^{\left ( \eta \right )}$
and let $\mathcal{F} =
\left ( \left . \rho_{\tilde{\mathcal{G}}}
\right |_{\mathcal{S} \left ( s_{\mathcal{Z}^n} \right )} 
\right )^{\left ( \eta \right )}$
with $\mathcal{S} \left ( s_{\mathcal{Z}^n} \right ) = 
\left \{ \left ( x_i, k \right ): \left ( i, k \right ) 
\in \left \llbracket 1; n \right \rrbracket \times 
\left \llbracket 1; C \right \rrbracket \right \}$.
If $\mathcal{F}_{\gamma}$
is $5$-separated in the pseudo-metric $d_{2, \mathbf{z}_n}$,
then
\begin{equation}
\label{eq:extended-Proposition-12-in-MenVer03}
\left | \mathcal{F}_{\gamma} \right | \leqslant \Sigma^2,
\end{equation}
where 
$\Sigma = \sum_{u=0}^{d_G} {n \choose u} M_{\gamma}^u$
with $M_{\gamma} = \left \lfloor \frac{\gamma}{\eta} \right \rfloor$
and $d_G$ is the maximal cardinality of a subset of $s_{\mathcal{Z}^n}$
strongly G-shattered by $\mathcal{F}$.
\end{lemma}

\begin{proof}
The principle of the proof is the one of the proof of
Lemma~\ref{lemma:discrete-combinatorial-result-for-G}.
Two of the three main formulas still apply:
Inequalities~\eqref{eq:discrete-combinatorial-result-for-G-partial-1} and
\eqref{eq:discrete-combinatorial-result-for-G-partial-3}.
For $\left | \mathcal{F}_{\gamma} \right | \geqslant 2$,
the incidence of the change of pseudo-metric is concentrated
in the derivation of the $2$-separating tree of $\mathcal{F}_{\gamma}$,
and thus the lower bound on
$\ell \left ( \mathcal{F}_{\gamma} \right )$.
Since the inner nodes $\bar{\mathcal{F}}_{\gamma}$
are $5$-separated in the pseudo-metric $d_{2, \mathbf{z}_n}$,
then according to 
Lemma~\ref{lemma:dedicated-Lemma-6-in-MenVer03-gamma-G-dim}, 
for each of these nodes,
we can ensure that there exists $\beta \in \left ( 0, \frac{1}{2} \right ]$
such that the two sons 
$\bar{\mathcal{F}}_{\gamma,+}$ and $\bar{\mathcal{F}}_{\gamma,-}$
verify either
$\left | \bar{\mathcal{F}}_{\gamma,+} \right | \geqslant
\left ( 1 - \beta \right ) \left | \bar{\mathcal{F}}_{\gamma} \right |$ and
$\left | \bar{\mathcal{F}}_{\gamma,-} \right | \geqslant
\frac{1}{2} \beta \left | \bar{\mathcal{F}}_{\gamma} \right |$
or vice versa
(in place of \eqref{eq:discrete-combinatorial-result-for-G-partial-0}).
As a consequence, the counterpart of
\eqref{eq:discrete-combinatorial-result-for-G-partial-4} is:
\begin{equation}
\label{eq:lower-bound-on-l-Graph-L_2-norm}
\ell \left ( \bar{\mathcal{F}}_{\gamma} \right )
\geqslant \left | \bar{\mathcal{F}}_{\gamma} \right |^{\frac{1}{2}}.
\end{equation}
Once more, the proof is an induction on the depth of the node.
Inequality~\eqref{eq:lower-bound-on-l-Graph-L_2-norm}
is obviously true for the leaves (which are of cardinality $1$).
Suppose now that it is true for the two sons of an inner node. Then,
\begin{align*}
\ell \left ( \bar{\mathcal{F}}_{\gamma} \right )
& = \; 
\ell \left ( \bar{\mathcal{F}}_{\gamma,+} \right ) +
\ell \left ( \bar{\mathcal{F}}_{\gamma,-} \right ) \\
& \geqslant \:
\left [ \left ( 1 - \beta \right )^{\frac{1}{2}}
+ \left ( \frac{\beta}{2} \right )^{\frac{1}{2}} \right ]
\left | \bar{\mathcal{F}}_{\gamma} \right |^{\frac{1}{2}} \\
& \geqslant \:
\left | \bar{\mathcal{F}}_{\gamma} \right |^{\frac{1}{2}}.
\end{align*}
Finally,
combining Inequalities~\eqref{eq:discrete-combinatorial-result-for-G-partial-1},
\eqref{eq:discrete-combinatorial-result-for-G-partial-3} and
\eqref{eq:lower-bound-on-l-Graph-L_2-norm}
produces \eqref{eq:extended-Proposition-12-in-MenVer03} by transitivity.
\end{proof}
The following lemma, a slight improvement of Lemma~13 in \citet{MenVer03},
implements the probabilistic extraction principle.

\begin{lemma}
\label{lemma:improved-lemma-13-in-MenVer03}
Let $\mathcal{T} = \left \{ t_i: \; 1 \leqslant i \leqslant n \right \}$
be a finite set,
$\mathbf{t}_n = \left( t_i \right)_{1 \leqslant i \leqslant n}$
and $M_{\mathcal{F}} \in \mathbb{R}_+^*$.
Let $\mathcal{F}$ be a class of functions from $\mathcal{T}$
into $\left [ 0, M_{\mathcal{F}} \right ]$
with finite cardinality $\left | \mathcal{F} \right | \geqslant 2$.
Assume that for some $\epsilon \in \left ( 0, M_{\mathcal{F}} \right ]$,
$\mathcal{F}$ is $\epsilon$-separated
with respect to the pseudo-metric $d_{2, \mathbf{t}_n}$, and let
$$
r = \frac{\ln \left ( \left | \mathcal{F} \right | \right )}{K_e \epsilon^4}
$$
with
$$
K_e = \frac{3}{112 M_{\mathcal{F}}^4}.
$$
Then, there exists a subvector $\mathbf{t}_q^{\prime}$
of $\mathbf{t}_n$ of size $q \leqslant r$ such that
$\mathcal{F}$ is $\frac{\epsilon}{2}$-separated with respect to
the pseudo-metric $d_{2, \mathbf{t}_q^{\prime}}$.
\end{lemma}

\begin{proof}
This proof uses an abuse of notation that will be repeated in the sequel:
the symbol $\mathbb{P}$ designates different probability
measures, some of which implicitly defined.
We first note that the statement is trivially true for $r \geqslant n$
(it suffices to set $\mathbf{t}_q^{\prime} = \mathbf{t}_n$). Thus, we proceed
under the hypothesis $r \in \left [ 1, n \right )$.
Let us set $\mathcal{F} = \left \{ f_j: \; 1 \leqslant j \leqslant
\left | \mathcal{F} \right | \right \}$
and $\mathcal{D}_{\mathcal{F}} = \left \{
f_j - f_{j'}: \; 1 \leqslant j < j'  \leqslant \left | \mathcal{F} \right |
\right \}$. The set $\mathcal{D}_{\mathcal{F}}$
has cardinality $\left | \mathcal{D}_{\mathcal{F}} \right | <
\frac{1}{2} \left | \mathcal{F} \right |^2$.
Let $\left ( \epsilon_i \right )_{1 \leqslant i \leqslant n}$ be a sequence
of $n$ independent Bernoulli random variables with common expectation
$\mu = \frac{r}{2n}$.
Then, by application of the $\epsilon$-separation property,
for every $\delta_f \in \mathcal{D}_{\mathcal{F}}$,
\begin{equation}
\label{eq:partial2-lemma-13-in-MenVer03}
\mathbb{P}
\left ( \frac{1}{n} \sum_{i=1}^n \epsilon_i \delta_f \left ( t_i \right )^2
< \frac{\epsilon^2 \mu}{2} \right )
\leqslant \mathbb{P} \left ( \frac{1}{n}
\sum_{i=1}^n \left ( \mu - \epsilon_i \right ) \delta_f \left ( t_i \right )^2
> \frac{\epsilon^2 \mu}{2} \right ).
\end{equation}
Since by construction, for every $i \in \llbracket 1; n \rrbracket$,
$\mathbb{E} \left [
\left ( \mu - \epsilon_i \right ) \delta_f \left ( t_i \right )^2 \right ] = 0$
and $\left | \mu - \epsilon_i \right | \delta_f \left ( t_i \right )^2
\leqslant M_{\mathcal{F}}^2 \left ( 1 - \mu \right ) < M_{\mathcal{F}}^2$
with probability one,
the right-hand side of \eqref{eq:partial2-lemma-13-in-MenVer03} can be
bounded from above thanks to Bernstein's inequality.
Given that
$$
\frac{1}{n} \sum_{i=1}^n \mathbb{E} \left [ \left ( \mu - \epsilon_i \right )^2
\delta_f \left ( t_i \right )^4 \right ]
\leqslant M_{\mathcal{F}}^4 \mu \left ( 1 - \mu \right )
< M_{\mathcal{F}}^4 \mu,
$$
we obtain

\begin{align*}
\mathbb{P}
\left ( \frac{1}{n} \sum_{i=1}^n \epsilon_i \delta_f \left ( t_i \right )^2
< \frac{\epsilon^2 \mu}{2} \right )
& \leqslant \;
\mathbb{P} \left ( \frac{1}{n}
\sum_{i=1}^n \left ( \mu - \epsilon_i \right ) \delta_f \left ( t_i \right )^2
> \frac{\epsilon^2 \mu}{2} \right ) \\
& \leqslant \;
\exp \left ( - \frac{3 \mu n \epsilon^4}
{4 \left (6 M_{\mathcal{F}}^4
+ M_{\mathcal{F}}^2 \epsilon^2 \right )} \right ) \\
& \leqslant \;
\exp \left ( - \frac{3 r \epsilon^4}
{56 M_{\mathcal{F}}^4} \right ) \\
& = \;
\left | \mathcal{F} \right |^{-2}.
\end{align*}
Therefore, given the assumption on $r$, applying
the union bound provides us with:

\begin{align}
\mathbb{P} \left ( \exists \delta_f \in \mathcal{D}_{\mathcal{F}}: \;
\left ( \frac{1}{r} \sum_{i=1}^n \epsilon_i \delta_f \left ( t_i \right )^2
\right )^{\frac{1}{2}} < \frac{\epsilon}{2} \right )
& = \;
\mathbb{P} \left ( \exists \delta_f \in \mathcal{D}_{\mathcal{F}}: \;
\frac{1}{n} \sum_{i=1}^n \epsilon_i \delta_f \left ( t_i \right )^2
< \frac{\epsilon^2 \mu}{2}
\right ) \nonumber \\
& \leqslant \;
\sum_{\delta_f \in \mathcal{D}_{\mathcal{F}}}
\mathbb{P} \left ( \frac{1}{n}
\sum_{i=1}^n \epsilon_i \delta_f \left ( t_i \right )^2
< \frac{\epsilon^2 \mu}{2} \right ) \nonumber \\
& \leqslant \;
\left | \mathcal{D}_{\mathcal{F}} \right |
\cdot \left | \mathcal{F} \right |^{-2} \nonumber \\
\label{eq:partial3-lemma-13-in-MenVer03}
& < \; \frac{1}{2}.
\end{align}
Moreover, if $\mathcal{S}_1$ is the random set
$\left \{ i \in \llbracket 1; n \rrbracket: \; \epsilon_i = 1 \right \}$,
then by Markov's inequality,

\begin{equation}
\label{eq:partial4-lemma-13-in-MenVer03}
\mathbb{P} \left ( \left | S_1 \right | > r \right ) =
\mathbb{P} \left ( \sum_{i=1}^n \epsilon_i > r \right ) \leqslant
\frac{1}{2}.
\end{equation}
Combining \eqref{eq:partial3-lemma-13-in-MenVer03}
and \eqref{eq:partial4-lemma-13-in-MenVer03}
by means of the union bound provides us with
$$
\mathbb{P}
\left \{ \left ( \exists \delta_f \in \mathcal{D}_{\mathcal{F}}: \;
\left ( \frac{1}{r} \sum_{i=1}^n \epsilon_i \delta_f \left ( t_i \right )^2
\right )^{\frac{1}{2}}
< \frac{\epsilon}{2} \right ) \text{ or }
\left ( \left | S_1 \right | > r \right )
\right \} < 1
$$
which implies that
$$
\mathbb{P}
\left \{ \left ( \forall \delta_f \in \mathcal{D}_{\mathcal{F}}: \;
\left \| \delta_f \right \|_{L_2 \left (
\mu_{\left ( t_i \right )_{i \in \mathcal{S}_1}} \right )}
\geqslant \frac{\epsilon}{2} \right ) \text{ and }
\left ( \left | S_1 \right | \leqslant r \right )
\right \} > 0.
$$
This translates into the fact that there exists
a subvector $\mathbf{t}_q^{\prime}$
of $\mathbf{t}_n$ of size $q \leqslant r$
such that the class
$\mathcal{F}$ is $\frac{\epsilon}{2}$-separated with respect to
the pseudo-metric $d_{2, \mathbf{t}_q^{\prime}}$, i.e., our claim.
\end{proof}
The proof of Lemma~\ref{lemma:extended-Theorem-1-in-MenVer03}
is the following one.

\begin{proof}
Let us consider any vector $\mathbf{z}_n \in \mathcal{Z}^n$ and let
$s_{\mathcal{Z}^n} = \left \{ z_i: \; 1 \leqslant i \leqslant n \right \}$ 
be the smallest subset of $\mathcal{Z}$
containing all the components of $\mathbf{z}_n$.
Note that its cardinality can be strictly inferior to $n$, in case that
$\mathbf{z}_n$ has two identical components.
Let $\tilde{\mathcal{G}}$ be a subset of $\mathcal{G}$ of cardinality
$\mathcal{M} \left ( \epsilon, \rho_{\mathcal{G}, \gamma},
d_{2, \mathbf{z}_n} \right )$ such that
$\left . \rho_{\tilde{\mathcal{G}}, \gamma}
\right |_{s_{\mathcal{Z}^n}}$ is
$\epsilon$-separated with respect to $d_{2, \mathbf{z}_n}$
and in bijection with $\tilde{\mathcal{G}}$.
By application of Lemma~\ref{lemma:improved-lemma-13-in-MenVer03}
with $\mathcal{F} =
\left . \rho_{\tilde{\mathcal{G}}, \gamma}
\right |_{s_{\mathcal{Z}^n}}$,
corresponding to $K_e = \frac{3}{112 \gamma^4}$,
there exists a subvector
$\mathbf{z}_q^{\prime}$ of $\mathbf{z}_n$ of size
\begin{equation}
\label{eq:partial-1-extended-Theorem-1-in-MenVer03}
q \leqslant
\frac{\ln \left ( \left | \tilde{\mathcal{G}} \right | \right )}
{K_e \epsilon^4}
\end{equation}
such that denoting
$s_{\mathcal{Z}^q} = \left \{ z_i^{\prime} = 
\left ( x_i^{\prime}, y_i^{\prime} \right ): \; 
1 \leqslant i \leqslant q \right \}$
($\left | s_{\mathcal{Z}^q} \right | \leqslant q$), the class
$\left . \rho_{\tilde{\mathcal{G}}, \gamma} \right |_{s_{\mathcal{Z}^q}}$ is
$\frac{\epsilon}{2}$-separated with respect to the pseudo-metric
$d_{2, \mathbf{z}_q^{\prime}}$.
Applying \eqref{eq:L_2-norm-lemma-3.2-in-AloBenCesHau97}
with $\mathcal{F} =
\left . \rho_{\tilde{\mathcal{G}}, \gamma} \right |_{s_{\mathcal{Z}^q}}$,
$N=5$ and the corresponding largest possible value for $\eta$,
$\frac{\epsilon}{12}$, it appears that the set
$\left ( \left . \rho_{\tilde{\mathcal{G}}, \gamma} \right |_{s_{\mathcal{Z}^q}}
\right )^{\left ( \frac{\epsilon}{12} \right )}$ is $5$-separated
with respect to the pseudo-metric $d_{2, \mathbf{z}_q^{\prime}}$.
Consequently, Lemma~\ref{lemma:extended-Proposition-12-in-MenVer03} applies to
this latter function class, whose cardinality is by construction
that of $\tilde{\mathcal{G}}$. It gives:
\begin{align}
\left | \tilde{\mathcal{G}} \right |
& = \;
\left | \left ( \left . \rho_{\tilde{\mathcal{G}}, \gamma} 
\right |_{s_{\mathcal{Z}^q}}
\right )^{\left ( \frac{\epsilon}{12} \right )} \right | \nonumber \\
& \leqslant \; 
\left ( \sum_{u=0}^{d_G} {q \choose u} 
\left ( \frac{12 \gamma}{\epsilon} \right )^u \right )^2 \nonumber \\
\label{eq:partial-2-extended-Theorem-1-in-MenVer03}
& \leqslant \;
\left ( \frac{12 \gamma e q}{d_G \epsilon}
\right )^{2 d_G},
\end{align}
where
$d_G$ is the maximal cardinality of a subset of $s_{\mathcal{Z}^q}$
strongly G-shattered by $\left ( \left . \rho_{\tilde{\mathcal{G}}} 
\right |_{\mathcal{S} \left ( s_{\mathcal{Z}^q} \right )}
\right )^{\left ( \frac{\epsilon}{12} \right )}$.
A substitution of the upper bound on $q$ provided by
\eqref{eq:partial-1-extended-Theorem-1-in-MenVer03} into
\eqref{eq:partial-2-extended-Theorem-1-in-MenVer03} gives:
$$
\left | \tilde{\mathcal{G}} \right |
\leqslant
\left ( K_1 \left ( \frac{\gamma}{\epsilon} \right )^5
\frac{\ln \left ( \left | \tilde{\mathcal{G}} \right | \right )}{d_G}
\right )^{2 d_G}
$$
with $K_1 = 448 e$.
In order to upper bound
$\ln \left ( \left | \tilde{\mathcal{G}}
\right |^{\frac{1}{d_G}} \right )$,
we resort once more to
\eqref{eq:from-margin-Graph-dimension-to-margin-Natarajan-dimension-part-7},
this time with $u_0 = 1$. Thus,
$$
\left | \tilde{\mathcal{G}} \right |^{\frac{1}{d_G}}
\leqslant
\ln^2 \left ( \left | \tilde{\mathcal{G}} \right |^{\frac{1}{d_G}}
\right ) \left ( K_1 \left ( \frac{\gamma}{\epsilon} \right )^5 \right )^2
$$
and $\left | \tilde{\mathcal{G}} \right | =
\mathcal{M} \left ( \epsilon, \rho_{\mathcal{G}, \gamma},
d_{2, \mathbf{z}_n} \right )$
imply that
\begin{equation}
\label{eq:partial-4-extended-Theorem-1-in-MenVer03}
\mathcal{M} \left ( \epsilon, \rho_{\mathcal{G}, \gamma},
d_{2, \mathbf{z}_n} \right )
\leqslant
\left ( K_2 \left ( \frac{\gamma}{\epsilon} \right )^5
\right )^{4 d_G}
\end{equation}
with $K_2 = 2K_1$.
Due to the construction of
$\left ( \left . \rho_{\tilde{\mathcal{G}}} 
\right |_{\mathcal{S} \left ( s_{\mathcal{Z}^q} \right )}
\right )^{\left ( \frac{\epsilon}{12} \right )}$,
which makes it possible to apply
Formula~\eqref{eq:extended-lemma-3.2-in-AloBenCesHau97-for-G},
\begin{align}
d_G
& \leqslant \;
\text{S-G-dim} \left ( \left ( \left . \rho_{\tilde{\mathcal{G}}} 
\right |_{\mathcal{S} \left ( s_{\mathcal{Z}^q} \right )}
\right )^{\left ( \frac{\epsilon}{12} \right )} \right ) \nonumber \\
& \leqslant \;
\left ( \frac{\epsilon}{24} \right )\text{-G-dim} \left (
\left . \rho_{\tilde{\mathcal{G}}} 
\right |_{\mathcal{S} \left ( s_{\mathcal{Z}^q} \right )}
\right ) \nonumber \\
\label{eq:partial-5-extended-Theorem-1-in-MenVer03}
& \leqslant \;
d_G \left ( \frac{\epsilon}{24} \right ).
\end{align}
By substitution of \eqref{eq:partial-5-extended-Theorem-1-in-MenVer03}
into \eqref{eq:partial-4-extended-Theorem-1-in-MenVer03},
we obtain that for every vector $\mathbf{z}_n \in \mathcal{Z}^n$,
\begin{equation}
\label{eq:partial-3-extended-Theorem-1-in-MenVer03}
\mathcal{M} \left ( \epsilon, \rho_{\mathcal{G}, \gamma},
d_{2, \mathbf{z}_n} \right ) \leqslant
\left ( K_2 \left ( \frac{\gamma}{\epsilon} \right )^5
\right )^{4 d_G \left ( \frac{\epsilon}{24} \right )}.
\end{equation}
At last, \eqref{eq:partial-3-extended-Theorem-1-in-MenVer03}
implies \eqref{eq:extended-Theorem-1-in-MenVer03}
since its right-hand side does not depend on $\mathbf{z}_n$.
\end{proof}

\subsection{Margin Natarajan Dimension - Uniform Convergence Norm}
The proof of Lemma~\ref{lemma:extended-Lemma-3.5-in-AloBenCesHau97-gamma-N-dim}
is essentially that of
Lemma~\ref{lemma:extended-Lemma-3.5-in-AloBenCesHau97-gamma-G-dim-optimized},
with the main differences being concentrated in
the basic combinatorial result (the counterpart of
Lemma~\ref{lemma:discrete-combinatorial-result-for-G}).
Thus, we only highlight these differences.
The first one is that the number $s \left ( \bar{\mathcal{G}} \right )$ of pairs
$\left ( s_{\mathcal{Z}^u}^{\prime}, \mathbf{b}_u^{\prime} \right )$
strongly G-shattered by $\bar{\mathcal{F}}$ is replaced with the number
$s^{\prime} \left ( \bar{\mathcal{G}} \right )$ of triplets
$\left ( s_{\mathcal{Z}^u}^{\prime}, \mathbf{b}_u^{\prime},
\mathbf{c}_u^{\prime} \right )$
strongly N-shattered by $\bar{\mathcal{F}}$. Let
$d_N$ denote the maximal cardinality of a subset of $s_{\mathcal{Z}^n}$
strongly N-shattered by $\mathcal{F}$. Once more, under the hypothesis
$d_N \geqslant 1$, combinatorics provides us with:
\begin{equation}
\label{eq:discrete-combinatorial-result-for-N-L_2-partial-1}
s^{\prime} \left ( \tilde{\mathcal{G}} \right ) \leqslant
\sum_{u=1}^{d_N} {n \choose u} M_{\gamma}^u
\left ( C - 1 \right )^u =\Sigma' -1.
\end{equation}
In order to obtain the counterpart of
\eqref{eq:discrete-combinatorial-result-for-G-partial-3},
i.e.,
\begin{align}
s' \left ( \bar{\mathcal{G}} \right ) 
& \geqslant \;
s' \left ( \bar{\mathcal{G}}_+ \right ) +
s' \left ( \bar{\mathcal{G}}_- \right ) + 1 \nonumber \\
\label{eq:discrete-combinatorial-result-for-N-L_2-partial-3}
& \geqslant \; 
\ell \left ( \bar{\mathcal{F}}_{\gamma} \right ) -1,
\end{align}
\eqref{eq:discrete-combinatorial-result-for-G-partial-2} must be replaced with
\begin{equation}
\label{eq:discrete-combinatorial-result-for-N-L_2-partial-2}
\begin{cases}
\forall f_+ \in \bar{\mathcal{F}}_+, \;\; 
f_+ \left ( z_{i_0} \right ) - b_{i_0} \geqslant 1 \\
\forall f_- \in \bar{\mathcal{F}}_-, \;\;
f_- \left ( x_{i_0}, c_{i_0} \right ) + b_{i_0} \geqslant 1
\end{cases}.
\end{equation}
This calls for an additional application of the pigeonhole principle
in the derivation of $\bar{\mathcal{F}}_{\gamma,-}$
and $\bar{\mathcal{F}}_-$,
so that the right-hand side of 
\eqref{eq:discrete-combinatorial-result-for-G-partial-0}
is replaced with
$$
\left | \bar{\mathcal{F}}_{\gamma,-} \right |
\geqslant \frac{\left | \bar{\mathcal{F}}_{\gamma} \right |}
{3 M_{\gamma} \left ( C-1 \right ) n}.
$$
This implies that the counterpart of 
\eqref{eq:discrete-combinatorial-result-for-G-partial-4} is
\begin{equation}
\label{eq:lower-bound-on-l-Natarajan-uniform-convergence-norm}
\ell \left ( \bar{\mathcal{F}}_{\gamma} \right ) \geqslant
\left | \bar{\mathcal{F}}_{\gamma} 
\right |^{\frac{1}{\log_2 \left ( 3 M_{\gamma} \sqrt{C-1} n \right )}}.
\end{equation}
Once more, it is proved by induction on the depth of the node.
Inequality~\eqref{eq:lower-bound-on-l-Natarajan-uniform-convergence-norm}
is obviously true for the leaves (which are of cardinality $1$).
Suppose now that it is true for the two sons of an inner node. Then,
\begin{align*}
\ell \left ( \bar{\mathcal{F}}_{\gamma} \right )
& = \; 
\ell \left ( \bar{\mathcal{F}}_{\gamma,+} \right ) +
\ell \left ( \bar{\mathcal{F}}_{\gamma,-} \right ) \\
& \geqslant \;
\left ( \frac{\left | \bar{\mathcal{F}}_{\gamma} \right |}{3 M_{\gamma} n}
\right )^{\frac{1}{\log_2 \left ( 3 M_{\gamma} \sqrt{C-1} n \right )}} +
\left ( \frac{\left | \bar{\mathcal{F}}_{\gamma} \right |}
{3 M_{\gamma} \left ( C-1 \right ) n}
\right )^{\frac{1}{\log_2 \left ( 3 M_{\gamma} \sqrt{C-1} n \right )}} \\
& = \;
\frac{1}{2} \left (
\left ( \sqrt{C-1}
\right )^{\frac{1}{\log_2 \left ( 3 M_{\gamma} \sqrt{C-1} n \right )}} +
\left ( \sqrt{C-1}
\right )^{-\frac{1}{\log_2 \left ( 3 M_{\gamma} \sqrt{C-1} n \right )}}
\right ) \left | \bar{\mathcal{F}}_{\gamma}
\right |^{\frac{1}{\log_2 \left ( 3 M_{\gamma} \sqrt{C-1} n \right )}} \\
& \geqslant \;
\frac{1}{2} \min_{t \in \mathbb{R}_+^*} \left ( t + \frac{1}{t} \right )
\left | \bar{\mathcal{F}}_{\gamma}
\right |^{\frac{1}{\log_2 \left ( 3 M_{\gamma} \sqrt{C-1} n \right )}} \\
& = \;
\left | \bar{\mathcal{F}}_{\gamma}
\right |^{\frac{1}{\log_2 \left ( 3 M_{\gamma} \sqrt{C-1} n \right )}}.
\end{align*}
Combining 
Inequalities~\eqref{eq:discrete-combinatorial-result-for-N-L_2-partial-1},
\eqref{eq:discrete-combinatorial-result-for-N-L_2-partial-3} and
\eqref{eq:lower-bound-on-l-Natarajan-uniform-convergence-norm},
the counterpart of 
Inequality~\eqref{eq:discrete-combinatorial-result-for-G} is
$$
\left | \mathcal{F}_{\gamma} \right | \leqslant
\left ( 3 M_{\gamma} \sqrt{C-1} n 
\right )^{\log_2 \left ( \Sigma^{\prime} \right )}.
$$

\subsection{Margin Natarajan Dimension - $L_2$-Norm}
The main difference between the proof
of Lemma~\ref{lemma:extended-Theorem-1-in-MenVer03-gamma-N-dim} and the proof
of Lemma~\ref{lemma:extended-Theorem-1-in-MenVer03}
is located in the small deviation principle
(Lemma~\ref{lemma:dedicated-Lemma-6-in-MenVer03} replaces
Lemma~\ref{lemma:dedicated-Lemma-6-in-MenVer03-gamma-G-dim}).
Since the consequences of this change appear in the derivation of
the basic combinatorial result, 
we provide this latter result with its full proof.

\begin{lemma}
\label{lemma:discrete-combinatorial-result-for-N-L_2}
Let $\mathcal{G}$ be a function class
satisfying Definition~\ref{def:margin-multi-category-classifiers}
and $\rho_{\mathcal{G}}$ the function class deduced
from $\mathcal{G}$ according to
Definition~\ref{def:class-of-margin-functions}.
For $\gamma \in \left ( 0, 1 \right ]$,
let $\rho_{\mathcal{G}, \gamma}$ be the function class deduced
from $\mathcal{G}$ according to
Definition~\ref{def:class-of-regret-functions}.
For $\tilde{\mathcal{G}} \subset \mathcal{G}$, $s_{\mathcal{Z}^n}
= \left \{ z_i = \left ( x_i, y_i \right ): 
1 \leqslant i \leqslant n \right \} \subset \mathcal{Z}$,
$\gamma \in \left ( 0, 1 \right ]$ and
$\eta \in \left ( 0, \frac{\gamma}{2} \right ]$,
let $\mathcal{F}_{\gamma} =
\left ( \left . \rho_{\tilde{\mathcal{G}}, \gamma}
\right |_{s_{\mathcal{Z}^n}} 
\right )^{\left ( \eta \right )}$
and let $\mathcal{F} =
\left ( \left . \rho_{\tilde{\mathcal{G}}}
\right |_{\mathcal{S} \left ( s_{\mathcal{Z}^n} \right )} 
\right )^{\left ( \eta \right )}$
with $\mathcal{S} \left ( s_{\mathcal{Z}^n} \right ) = 
\left \{ \left ( x_i, k \right ): \left ( i, k \right )
\in \left \llbracket 1; n \right \rrbracket \times 
\left \llbracket 1; C \right \rrbracket \right \}$.
If $\mathcal{F}_{\gamma}$
is $6$-separated in the pseudo-metric $d_{2, \mathbf{z}_n}$,
then
\begin{equation}
\label{eq:discrete-combinatorial-result-for-N-L_2}
\left | \mathcal{F}_{\gamma} \right | \leqslant
\left ( \Sigma' \right )^{\frac{1}{2} \log_2 \left ( 2 M_{\gamma}^2
\left ( C-1 \right ) \right )}
\end{equation}
where 
$\Sigma' = \sum_{u=0}^{d_N} {n \choose u} M_{\gamma}^u
\left ( C - 1 \right )^u$
with $M_{\gamma} = \left \lfloor \frac{\gamma}{\eta} \right \rfloor$
and $d_N$ is the maximal cardinality of a subset of $s_{\mathcal{Z}^n}$
strongly N-shattered by $\mathcal{F}$.
\end{lemma}

\begin{proof}
Inequality~\eqref{eq:discrete-combinatorial-result-for-N-L_2}
is trivially true for $\left | \mathcal{F}_{\gamma} \right | = 1$.
Indeed, the minimal value of its right-hand side,
corresponding to $d_N = 0$, is 1.
Thus, the rest of the proof makes use of the restriction
$\left | \mathcal{F}_{\gamma} \right | \geqslant 2$.
A direct consequence is that according to 
Lemma~\ref{lemma:from-separation-to-strong-G-and-N-shattering},
$d_N \geqslant 1$.
A subset of $s_{\mathcal{Z}^n}$ of cardinality 
$u \in \left \llbracket 1; n \right \rrbracket$ is denoted by 
$s_{\mathcal{Z}^u}^{\prime} = \left \{
z_i^{\prime}: \; 1 \leqslant i \leqslant u \right \}$, with the convention
$$
\forall \left ( i, j \right ): 1 \leqslant i < j \leqslant u, \;
\left ( z_i^{\prime}, z_j^{\prime} \right )
= \left ( z_v, z_w \right ) \Longrightarrow 1 \leqslant v < w \leqslant n.
$$
For every subset $\bar{\mathcal{G}}$ of $\tilde{\mathcal{G}}$,
denote by $s' \left ( \bar{\mathcal{G}} \right )$ the number of triplets
$\left ( s_{\mathcal{Z}^u}^{\prime},
\mathbf{b}_u^{\prime}, \mathbf{c}_u^{\prime} \right )$
with $s_{\mathcal{Z}^u}^{\prime} \subset s_{\mathcal{Z}^n}$,
$\mathbf{b}_u^{\prime} \in \left \llbracket 1 ; M_{\gamma} -1
\right \rrbracket^u$ and
$\mathbf{c}_u^{\prime} \in \mathcal{Y}^n$
(with $\forall i \in \left \llbracket 1; n \right \rrbracket$,
$c_i \neq y_i$)
strongly N-shattered by $\bar{\mathcal{F}} =
\left ( \left . \rho_{\bar{\mathcal{G}}}
\right |_{\mathcal{S} \left ( s_{\mathcal{Z}^n} \right )} 
\right )^{\left ( \eta \right )}$
(the convention above has been introduced to avoid handling duplicates).
Since $d_N \geqslant 1$,
Inequality~\eqref{eq:discrete-combinatorial-result-for-N-L_2-partial-1}
provides us once more with an upper bound on
$s' \left ( \tilde{\mathcal{G}} \right )$.
In order to derive a lower bound on the same quantity,
we also build a $2$-separating tree of $\mathcal{F}_{\gamma}$.
Let $\bar{\mathcal{F}}_{\gamma} 
= \left ( \left . \rho_{\bar{\mathcal{G}}, \gamma}
\right |_{s_{\mathcal{Z}^n}} 
\right )^{\left ( \eta \right )}$
be one of its nodes such
that $\left | \bar{\mathcal{F}}_{\gamma} \right | \geqslant 2$ (inner node).
Its two sons, $\bar{\mathcal{F}}_{\gamma,+}$ and $\bar{\mathcal{F}}_{\gamma,-}$,
are built by application of Lemma~\ref{lemma:dedicated-Lemma-6-in-MenVer03}
and the pigeonhole principle.
According to Lemma~\ref{lemma:dedicated-Lemma-6-in-MenVer03},
we can ensure that 
there exist an index $i_0 \in \left \llbracket 1; n \right \rrbracket$,
$\left ( \alpha, \beta \right )
\in \left \llbracket 1; M_{\gamma} -1 \right \rrbracket
\times \left [ \frac{1}{M_{\gamma}^2}, \frac{1}{2} \right ]$
and two subsets
$\bar{\mathcal{F}}_{\gamma,+}$ and $\hat{\mathcal{F}}_{\gamma,-}$
of $\bar{\mathcal{F}}_{\gamma}$ verifying either
$\left | \bar{\mathcal{F}}_{\gamma,+} \right | \geqslant
\left ( 1 - \beta \right ) \left | \bar{\mathcal{F}}_{\gamma} \right |$ and
$\left | \hat{\mathcal{F}}_{\gamma,-} \right | \geqslant
\max \left \{ \frac{1}{2} \beta, \frac{1}{M_{\gamma}^2} \right \}
\left | \bar{\mathcal{F}}_{\gamma} \right |$
or vice versa, such that
$$
\begin{cases}
\forall f_{\gamma,+} \in \bar{\mathcal{F}}_{\gamma,+}, \;\;
f_{\gamma,+} \left ( z_{i_0} \right ) \geqslant \alpha + 1 \\
\forall f_{\gamma,-} \in \hat{\mathcal{F}}_{\gamma,-}, \;\;
f_{\gamma,-} \left ( z_{i_0} \right ) \leqslant \alpha -1
\end{cases}.
$$
Let $\bar{\mathcal{G}}_+$ be a subset of $\bar{\mathcal{G}}$ in bijection with
$\bar{\mathcal{F}}_{\gamma,+}$ such that
$\bar{\mathcal{F}}_{\gamma,+} =
\left ( \left . \rho_{\bar{\mathcal{G}}_+, \gamma}
\right |_{s_{\mathcal{Z}^n}} 
\right )^{\left ( \eta \right )}$
and let $\bar{\mathcal{F}}_+ =
\left ( \left . \rho_{\bar{\mathcal{G}}_+}
\right |_{\mathcal{S} \left ( s_{\mathcal{Z}^n} \right )} 
\right )^{\left ( \eta \right )}$.
Let $\hat{\mathcal{G}}_-$ be a subset of $\bar{\mathcal{G}}$ in bijection with
$\hat{\mathcal{F}}_{\gamma,-}$ such that
$\hat{\mathcal{F}}_{\gamma,-} =
\left ( \left . \rho_{\hat{\mathcal{G}}_-, \gamma}
\right |_{s_{\mathcal{Z}^n}} 
\right )^{\left ( \eta \right )}$
and let $\hat{\mathcal{F}}_- =
\left ( \left . \rho_{\hat{\mathcal{G}}_-}
\right |_{\mathcal{S} \left ( s_{\mathcal{Z}^n} \right )} 
\right )^{\left ( \eta \right )}$.
Setting $b_{i_0} = \alpha$, it springs from
Lemma~\ref{lemma:from-separation-to-strong-G-and-N-shattering} that
$$
\begin{cases}
\forall f_+ \in \bar{\mathcal{F}}_+, \;\;
f_+ \left ( z_{i_0} \right ) - b_{i_0} \geqslant 1 \\
\forall f_- \in \hat{\mathcal{F}}_-, \;\;
\max_{k \neq y_{i_0}} f_- \left ( x_{i_0}, k \right ) + b_{i_0} \geqslant 1
\end{cases},
$$
i.e., \eqref{eq:discrete-combinatorial-result-for-G-partial-2}
is obtained with $\bar{\mathcal{F}}_-$ replaced with $\hat{\mathcal{F}}_-$.
There comes the application of the pigeonhole principle, to obtain
Inequality~\eqref{eq:discrete-combinatorial-result-for-N-L_2-partial-2}.
The derivation of the corresponding function classes is as follows. There exists
$c_{i_0} \in \mathcal{Y} \setminus \left \{ y_{i_0} \right \}$
such that among the functions $g_-$ in $\hat{\mathcal{G}}_-$, 
at least 
$\left \lceil \frac{\left | \hat{\mathcal{G}}_- \right |}{C-1} \right \rceil$
of them satisfy $c_{i_0} \in \argmax_{k \neq y_{i_0}} 
f_- \left ( x_{i_0}, k \right )$.
We choose $\bar{\mathcal{G}}_-$ to be any such subset of $\hat{\mathcal{G}}_-$
and let $\bar{\mathcal{F}}_{\gamma,-} =
\left ( \left . \rho_{\bar{\mathcal{G}}_-, \gamma}
\right |_{s_{\mathcal{Z}^n}} 
\right )^{\left ( \eta \right )}$
and $\bar{\mathcal{F}}_- =
\left ( \left . \rho_{\bar{\mathcal{G}}_-}
\right |_{\mathcal{S} \left ( s_{\mathcal{Z}^n} \right )} 
\right )^{\left ( \eta \right )}$.
With 
Inequality~\eqref{eq:discrete-combinatorial-result-for-N-L_2-partial-2} at hand,
Inequality~\eqref{eq:discrete-combinatorial-result-for-N-L_2-partial-3}
is also available.
Thus, finishing the proof still boils down to deriving
a lower bound on $\ell \left ( \bar{\mathcal{F}}_{\gamma} \right )$.
The originality rests on the fact that two cases must be considered,
to take into account the two sources of asymmetry between the cardinalities
of 
$\bar{\mathcal{F}}_{\gamma,+}$ and
$\bar{\mathcal{F}}_{\gamma,-}$.
Indeed, we have either
$\left | \bar{\mathcal{F}}_{\gamma,+} \right | \geqslant
\max \left \{ \frac{1}{2} \beta, \frac{1}{M_{\gamma}^2} \right \}
\left | \bar{\mathcal{F}}_{\gamma} \right |$ and
$\left | \bar{\mathcal{F}}_{\gamma,-} \right | \geqslant
\frac{1 - \beta}{C-1} \left | \bar{\mathcal{F}}_{\gamma} \right |$ or
$\left | \bar{\mathcal{F}}_{\gamma,+} \right | \geqslant
\left ( 1 - \beta \right ) \left | \bar{\mathcal{F}}_{\gamma} \right |$ and
$\left | \bar{\mathcal{F}}_{\gamma,-} \right | \geqslant
\frac{1}{C-1} \max \left \{ \frac{1}{2} \beta, \frac{1}{M_{\gamma}^2} \right \}
\left | \bar{\mathcal{F}}_{\gamma} \right |$.
The induction hypothesis is this time:
\begin{equation}
\label{eq:discrete-combinatorial-result-for-N-L_2-partial-4}
\ell \left ( \bar{\mathcal{F}}_{\gamma} \right )
\geqslant \left | \bar{\mathcal{F}}_{\gamma} 
\right |^{\frac{2}{\log_2 \left ( 2M_{\gamma}^2 \left ( C-1 \right ) \right )}}.
\end{equation}
Once more, it is obviously true for the leaves.
We prove it for the first case (the other one is treated in the same way). 
Then,
\begin{align*}
\ell \left ( \bar{\mathcal{F}}_{\gamma} \right )
& = \; 
\ell \left ( \bar{\mathcal{F}}_{\gamma,+} \right ) +
\ell \left ( \bar{\mathcal{F}}_{\gamma,-} \right ) \\
& \geqslant \; 
\left ( \left ( \frac{1}{M_{\gamma}^2} \right )^{\frac{1}{\log_2 \left ( 
\sqrt{2} M_{\gamma} \sqrt{C-1} \right )}} +
\left ( \frac{1}{2 \left ( C-1 \right )}
\right )^{\frac{1}{\log_2 \left ( 
\sqrt{2} M_{\gamma} \sqrt{C-1} \right )}}
\right )
\left | \bar{\mathcal{F}}_{\gamma} 
\right |^{\frac{1}{\log_2 \left ( 
\sqrt{2} M_{\gamma} \sqrt{C-1} \right )}} \\
& = \;
\frac{1}{2} \left (
\left ( 
\frac{\sqrt{2 \left ( C-1 \right )}}{M_{\gamma}}
\right )^{\frac{1}{\log_2 \left ( \sqrt{2} M_{\gamma} \sqrt{C-1}
\right )}} +
\left ( 
\frac{M_{\gamma}}{\sqrt{2 \left ( C-1 \right )}}
\right )^{\frac{1}{\log_2 \left ( \sqrt{2} M_{\gamma} \sqrt{C-1}
\right )}}
\right ) \\
& ~ \;
\times \left | \bar{\mathcal{F}}_{\gamma} 
\right |^{\frac{1}{\log_2 \left ( \sqrt{2} M_{\gamma} \sqrt{C-1}
\right )}} \\
& \geqslant \;
\frac{1}{2} \min_{t \in \mathbb{R}_+^*} \left ( t + \frac{1}{t} \right )
\left | \bar{\mathcal{F}}_{\gamma} 
\right |^{\frac{1}{\log_2 \left ( \sqrt{2} M_{\gamma} \sqrt{C-1}
\right )}} \\
& = \; \left | \bar{\mathcal{F}}_{\gamma} 
\right |^{\frac{1}{\log_2 \left ( \sqrt{2} M_{\gamma} \sqrt{C-1}
\right )}}.
\end{align*}
Combining 
Inequalities~\eqref{eq:discrete-combinatorial-result-for-N-L_2-partial-1},
\eqref{eq:discrete-combinatorial-result-for-N-L_2-partial-3} and
\eqref{eq:discrete-combinatorial-result-for-N-L_2-partial-4}
produces by transitivity:
$$
\left | \bar{\mathcal{F}}_{\gamma} \right |
\leqslant 
\left ( \Sigma' \right )^{\frac{1}{2} \log_2 \left ( 2 M_{\gamma}^2
\left ( C-1 \right ) \right )},
$$
thus concluding the proof.
\end{proof}
The proof of Lemma~\ref{lemma:extended-Theorem-1-in-MenVer03-gamma-N-dim}
is the following one.

\begin{proof}
The beginning of the proof is identical to the beginning of the proof
of Lemma~\ref{lemma:extended-Theorem-1-in-MenVer03}
up to the use of the basic combinatorial result
(where Lemma~\ref{lemma:discrete-combinatorial-result-for-N-L_2} replaces 
Lemma~\ref{lemma:extended-Proposition-12-in-MenVer03}).

\begin{align}
\left | \tilde{\mathcal{G}} \right |
& = \;
\left | \left ( \left . \rho_{\tilde{\mathcal{G}}, \gamma} 
\right |_{s_{\mathcal{Z}^q}}
\right )^{\left ( \frac{\epsilon}{14} \right )} \right | \nonumber \\
& \leqslant \;
\left (
\sum_{u=0}^{d_N} {q \choose u} \left ( \frac{14 \gamma}{\epsilon} \right )^u 
\left ( C - 1 \right )^u
\right )^{\frac{1}{2} \log_2 \left ( 
2 \left ( \frac{14 \gamma}{\epsilon} \right )^2 \left ( C-1 \right )
\right )} \nonumber \\
\label{eq:extended-Theorem-1-in-MenVer03-gamma-N-dim-part-3}
& \leqslant \;
\left ( \frac{14 \gamma \left ( C - 1 \right ) e q}{d_N \epsilon}
\right )^{\frac{1}{2} \log_2 \left ( 
2 \left ( \frac{14 \gamma}{\epsilon} \right )^2 \left ( C-1 \right )
\right ) d_N},
\end{align}
where
$d_N$ is the maximal cardinality of a subset of $s_{\mathcal{Z}^q}$
strongly N-shattered by $\left ( \left . \rho_{\tilde{\mathcal{G}}} 
\right |_{\mathcal{S} \left ( s_{\mathcal{Z}^q} \right )}
\right )^{\left ( \frac{\epsilon}{14} \right )}$.
A substitution of the upper bound on $q$ provided by
\eqref{eq:partial-1-extended-Theorem-1-in-MenVer03} into
\eqref{eq:extended-Theorem-1-in-MenVer03-gamma-N-dim-part-3} gives:
$$
\left | \tilde{\mathcal{G}} \right | \leqslant
\left ( K \left ( C - 1 \right ) \left ( \frac{\gamma}{\epsilon} \right )^5 
\frac{ \ln \left ( \left | \tilde{\mathcal{G}} \right | \right )}{d_N}
\right )^{\frac{1}{2} \log_2 \left ( 
2 \left ( \frac{14 \gamma}{\epsilon} \right )^2 \left ( C-1 \right )
\right ) d_N}
$$
with $K = \frac{1568}{3} e$.
In order to upper bound
$\ln \left ( \left | \tilde{\mathcal{G}}
\right |^{\frac{1}{d_N}} \right )$,
we resort once more to
\eqref{eq:from-margin-Graph-dimension-to-margin-Natarajan-dimension-part-7},
with 
$u_0 = \frac{1}{4} \log_2 \left (
2 \left ( \frac{14 \gamma}{\epsilon} \right )^2 \left ( C-1 \right ) \right )$.
Thus,
$$
\ln \left ( \left | \tilde{\mathcal{G}} \right |^{\frac{1}{d_N}} \right )
\leqslant
\frac{1}{2} \log_2 \left (
2 \left ( \frac{14 \gamma}{\epsilon} \right )^2 \left ( C-1 \right ) \right )
\left | \tilde{\mathcal{G}} \right |^{\frac{1}
{\log_2 \left ( 2 \left ( \frac{14 \gamma}{\epsilon} \right )^2
\left ( C-1 \right ) \right ) d_N}}
$$
implies that
$$
\left | \tilde{\mathcal{G}} \right |
\leqslant
\left ( \frac{1}{2} \log_2 \left (
2 \left ( \frac{14 \gamma}{\epsilon} \right )^2 \left ( C-1 \right ) \right )
K \left ( C - 1 \right ) \left ( \frac{\gamma}{\epsilon} 
\right )^5 \right )^{
\log_2 \left ( 2 \left ( \frac{14 \gamma}{\epsilon} \right )^2
\left ( C-1 \right ) \right ) d_N}.
$$
Since
\begin{align*}
2 \left ( \frac{14 \gamma}{\epsilon} \right )^2 \left ( C-1 \right ) > 16
& \Longrightarrow \;
\log_2 \left (
2 \left ( \frac{14 \gamma}{\epsilon} \right )^2 \left ( C-1 \right ) \right ) <
\left (
2 \left ( \frac{14 \gamma}{\epsilon} \right )^2 \left ( C-1 \right )
\right )^{\frac{1}{2}} \\
& \Longrightarrow \;
\log_2 \left (
2 \left ( \frac{14 \gamma}{\epsilon} \right )^2 \left ( C-1 \right ) \right ) <
\left ( \frac{K}{2} \left ( C - 1 \right ) 
\left ( \frac{\gamma}{\epsilon} \right )^5 \right )^{\frac{1}{2}},
\end{align*}
the upper bound on $\left | \tilde{\mathcal{G}} \right |$ simplifies into
\begin{equation}
\label{eq:extended-Theorem-1-in-MenVer03-gamma-N-dim-part-4}
\mathcal{M} \left ( \epsilon, \rho_{\mathcal{G}, \gamma},
d_{2, \mathbf{z}_n} \right )
= \left | \tilde{\mathcal{G}} \right |
\leqslant
\left ( \left ( C - 1 \right ) \left ( \frac{4 \gamma}{\epsilon} 
\right )^5 \right )^{\frac{3}{2}
\log_2 \left ( 2 \left ( \frac{14 \gamma}{\epsilon} \right )^2
\left ( C-1 \right ) \right ) d_N}.
\end{equation}
Due to the construction of
$\left ( \left . \rho_{\tilde{\mathcal{G}}} 
\right |_{\mathcal{S} \left ( s_{\mathcal{Z}^q} \right )}
\right )^{\left ( \frac{\epsilon}{14} \right )}$,
which makes it possible to apply
Formula~\eqref{eq:extended-lemma-3.2-in-AloBenCesHau97-for-N},
\begin{align}
d_N & \leqslant \;
\text{S-N-dim} \left ( \left ( \left . \rho_{\tilde{\mathcal{G}}} 
\right |_{\mathcal{S} \left ( s_{\mathcal{Z}^q} \right )}
\right )^{\left ( \frac{\epsilon}{14} \right )} \right ) \nonumber \\
& \leqslant \;
\left ( \frac{\epsilon}{28} \right )\text{-N-dim} \left (
\left . \rho_{\tilde{\mathcal{G}}} 
\right |_{\mathcal{S} \left ( s_{\mathcal{Z}^q} \right )}
\right ) \nonumber \\
\label{eq:extended-Theorem-1-in-MenVer03-gamma-N-dim-part-5}
& \leqslant \;
d_N \left ( \frac{\epsilon}{28} \right ).
\end{align}
A substitution of \eqref{eq:extended-Theorem-1-in-MenVer03-gamma-N-dim-part-5}
into \eqref{eq:extended-Theorem-1-in-MenVer03-gamma-N-dim-part-4}
produces an upper bound on 
$\mathcal{M} \left ( \epsilon, \rho_{\mathcal{G}, \gamma},
d_{2, \mathbf{z}_n} \right )$
which does not depend on $\mathbf{z}_n$, thus concluding the proof.
\end{proof}

\section{Proofs of the Structural Results}

This appendix gathers the proofs of the upper bounds
on the combinatorial dimensions of $\rho_{\mathcal{G}}$
(and $\tilde{\rho}_{\mathcal{G}}$)
as a function of the fat-shattering dimensions of classes including
the classes $\mathcal{G}_k$ of component functions.

\subsection{Margin Graph Dimension of 
\texorpdfstring{$\rho_{\mathcal{G}}$}
{the Class of Margin Functions}}

Since the margin Graph dimension of $\rho_{\mathcal{G}}$
is the fat-shattering dimension of $\tilde{\rho}_{\mathcal{G}}$
subject to $\mathbf{b}_n \in \mathbb{R}_+^n$
(Proposition~\ref{prop:capacities-of-the-two-classes-of-margin-functions}),
the proof of Lemma~\ref{lemma:from-gamma-dimension-to-gamma-dimensions}
is actually provided for the latter dimension.
It makes use of three partial results which are now stated.
Proposition~\ref{prop:Proposition-1.4-in-Tal03}
is an extension of Proposition~1.4 in \citet{Tal03}
holding for the $L_p$-norms with
$p \in \mathbb{N} \setminus \left \{ 0, 1 \right \}$
(instead of simply $p=2$), that explicits the value of the constant.

\begin{proposition}
\label{prop:Proposition-1.4-in-Tal03}
Let $\mathcal{F}$ be a class of real-valued functions on $\mathcal{T}$.
For every $\gamma \in \mathbb{R}_+^*$ satisfying
$\gamma\text{-dim} \left ( \mathcal{F} \right ) > 0$,
$n \in \left \llbracket 1; \gamma\text{-dim} \left ( \mathcal{F} \right )
\right \rrbracket$
and $p \in \mathbb{N} \setminus \left \{ 0, 1 \right \}$,
$$
n \leqslant K_p
\log_2 \left ( \mathcal{M}_p \left ( \gamma, \mathcal{F}, n \right ) \right )
$$
with $K_p = \left ( \frac{2^p}{2^{p-1}-1} \right )^2$.
\end{proposition}

\begin{proof}
Suppose that for $\gamma \in \mathbb{R}_+^*$, the subset
$s_{\mathcal{T}^n} = \left \{ t_i: 1 \leqslant i \leqslant n \right \}$
of $\mathcal{T}$ is ${\gamma}$-shattered by $\mathcal{F}$ and
$\mathbf{b}_n =  \left ( b_i \right )_{1 \leqslant i \leqslant n}
\in \mathbb{R}^n$ is a witness to this shattering.
By definition, there exists a subset
$\bar{\mathcal{F}} = \left \{ f_{\mathbf{s}_n}: \;
\mathbf{s}_n \in \left \{ -1, 1 \right \}^n \right \}$
of $\mathcal{F}$ satisfying
\begin{equation}
\label{eq:Proposition-1.4-in-Tal03_part-1}
\forall \mathbf{s}_n \in \left \{ -1, 1 \right \}^n, \;
\forall i \in \llbracket 1; n \rrbracket, \;\;
s_i \left( f_{\mathbf{s}_n} \left ( t_i \right ) - b_i \right) \geqslant \gamma.
\end{equation}
Let $\mathbf{t}_n = \left ( t_i \right )_{1 \leqslant i \leqslant n}$.
To prove the proposition, it suffices to establish that
\begin{equation}
\label{eq:Proposition-1.4-in-Tal03_part-2}
n \leqslant
K_p \log_2 \left ( \mathcal{M} \left ( \gamma, \bar{\mathcal{F}},
d_{p, \mathbf{t}_n} \right ) \right ).
\end{equation}
For $\left ( \mathbf{s}_n, \mathbf{s}_n^{\prime} \right )
\in \left ( \left \{ -1, 1 \right \}^n \right )^2$,
let $\mathcal{S} \left ( \mathbf{s}_n, \mathbf{s}_n^{\prime} \right )$
be the subset of $\llbracket 1; n \rrbracket$ defined by:
$$
\mathcal{S} \left ( \mathbf{s}_n, \mathbf{s}_n^{\prime} \right )
= \left \{i \in \llbracket 1; n \rrbracket: \;
s_i \neq s_i^{\prime} \right \}.
$$
Then, making use of \eqref{eq:Proposition-1.4-in-Tal03_part-1}, we obtain that
\begin{align*}
d_{p, \mathbf{t}_n} \left ( f_{\mathbf{s}_n}, f_{\mathbf{s}_n^{\prime}} \right )
& \geqslant \;
\left ( \frac{1}{n}
\left | \mathcal{S} \left ( \mathbf{s}_n, \mathbf{s}_n^{\prime} \right ) 
\right |
\left ( 2 \gamma \right )^p \right )^{\frac{1}{p}} \\
& = \; 2 \gamma \left ( \frac{
d_H \left ( \mathbf{s}_n, \mathbf{s}_n^{\prime} \right )}{n}
\right )^{\frac{1}{p}},
\end{align*}
where $d_H$ stands for the Hamming distance.
Thus, a sufficient condition for
$d_{p, \mathbf{t}_n} \left ( f_{\mathbf{s}_n},
f_{\mathbf{s}_n^{\prime}} \right ) \geqslant \gamma$ is
$d_H \left ( \mathbf{s}_n, \mathbf{s}_n^{\prime} \right )
\geqslant \left \lceil \left ( \frac{1}{2} \right )^p n \right \rceil$.
As a consequence, to prove \eqref{eq:Proposition-1.4-in-Tal03_part-2},
it suffices to establish that there is a subset
of the set of vertices of the hypercube $Q_n$ of cardinality
$\left \lceil 2^{\frac{n}{K_p}} \right \rceil$
which is
$\left \lceil \left ( \frac{1}{2} \right )^p n \right \rceil$-separated
with respect to the Hamming distance
(the separation is well-defined since
$\left \lceil 2^{\frac{n}{K_p}} \right \rceil \geqslant 2$).
To that end, a probabilistic approach similar to that of the proof
of Lemma~\ref{lemma:improved-lemma-13-in-MenVer03} is implemented.
For $q \in \left \llbracket 2; 2^n \right \rrbracket$,
let $\boldsymbol{\epsilon}_{q,n} =
\left ( \epsilon_{j,i} \right )_{1 \leqslant j \leqslant q,
1 \leqslant i \leqslant n}$
be a Bernoulli random matrix
(its entries $\epsilon_{j,i}$ are independent Bernoulli random variables
with common expectation $\frac{1}{2}$).
Then, by application of the union bound,
\begin{align*}
& \mathbb{P} \left ( \exists \left ( j, j^{\prime} \right )
\in \llbracket 1; q \rrbracket^2: 1 \leqslant j < j^{\prime} \leqslant q
\text{ and } \sum_{i=1}^n
\mathds{1}_{\left \{ \epsilon_{j,i} \neq \epsilon_{j^{\prime},i} \right \}}
< \left ( \frac{1}{2} \right )^p n \right ) \\
~ & \leqslant
{q \choose 2} \mathbb{P} \left ( \sum_{i=1}^n \epsilon_i
> n \left ( 1 - \left ( \frac{1}{2} \right )^p \right ) \right ),
\end{align*}
where $\left ( \epsilon_i \right )_{1 \leqslant i \leqslant n}$ is
a Bernoulli random vector.
To upper bound the tail probability on the right-hand side, we resort to
Hoeffding's inequality, which gives
$$
\mathbb{P} \left ( \sum_{i=1}^n \epsilon_i - \frac{n}{2}
> \frac{n}{2} \left ( 1 - \left ( \frac{1}{2} \right )^{p-1} \right ) \right )
\leqslant \exp \left ( - \frac{n}{2}
\left ( 1 - \left ( \frac{1}{2} \right )^{p-1} \right )^2 \right ).
$$
By transitivity, this implies that a sufficient condition for
$$
\mathbb{P} \left ( \exists \left ( j, j^{\prime} \right )
\in \llbracket 1; q \rrbracket^2: 1 \leqslant j < j^{\prime} \leqslant q
\text{ and } \sum_{i=1}^n
\mathds{1}_{\left \{ \epsilon_{j,i} \neq \epsilon_{j^{\prime},i} \right \}}
< \left ( \frac{1}{2} \right )^p n \right ) < 1
$$
is
$$
{q \choose 2} \exp \left ( - \frac{n}{2}
\left ( 1 - \left ( \frac{1}{2} \right )^{p-1} \right )^2 \right ) < 1
$$
and consequently
$$
q \leqslant \left \lceil 2^{\frac{n}{K_p}} \right \rceil,
$$
which is precisely the value announced and thus concludes the proof.
\end{proof}
The transition between covering and packing numbers
is provided by a well-known equivalence.

\begin{lemma}
\label{lemma:Kolmogorov}
Let $\left ( \mathcal{E}, \rho \right )$ be a pseudo-metric space.
For every totally bounded set $\mathcal{E}' \subset \mathcal{E}$ and
$\epsilon \in \mathbb{R}_+^*$,
$\mathcal{M} \left ( 2 \epsilon, \mathcal{E}', \rho \right ) \leqslant
\mathcal{N}^{\text{int}} \left ( \epsilon, \mathcal{E}', \rho \right ) \leqslant
\mathcal{M} \left ( \epsilon, \mathcal{E}', \rho \right )$.
\end{lemma}

The optimization of the dependence on $C$ calls for the use of an extension
of Theorem~1 in \citet{MenVer03} holding for the $L_p$-norms with
$p \in \mathbb{N} \setminus \left \{ 0, 1 \right \}$
(instead of simply $p=2$):
Theorem~10 in \citet{Men03}. The following lemma explicits
the value of its constants (absolute or depending on $p$).

\begin{lemma}[After Theorem~2 in \citealp{MusLauGue19}]
\label{lemma:generalized-Sauer-Shelah-lemma-in-L_p-MusLauGue19}
Let $\mathcal{F}$ be a class of functions from $\mathcal{T}$ into
$\left [ -M_{\mathcal{F}}, M_{\mathcal{F}} \right ]$
with $M_{\mathcal{F}} \in \mathbb{R}_+^*$.
$\mathcal{F}$ is supposed to be a uGC class.
For $\epsilon \in \left ( 0, M_{\mathcal{F}} \right ]$,
let $d \left ( \epsilon \right )
= \epsilon\text{-dim} \left ( \mathcal{F} \right )$.
Then for
$\epsilon \in \left ( 0, 2 M_{\mathcal{F}} \right ]$,
$n \in \mathbb{N}^*$ and $p \in \mathbb{N}
\setminus \left \{ 0, 1 \right \}$,
$$
\mathcal{M}_p \left ( \epsilon, \mathcal{F}, n \right )
\leqslant \left ( \frac{12 M_{\mathcal{F}} p^{\frac{1}{7}}}{\epsilon}
\right )^{10 p d \left ( \frac{\epsilon}{36 p} \right )}.
$$
\end{lemma}
With Proposition~\ref{prop:Proposition-1.4-in-Tal03},
Lemma~\ref{lemma:Kolmogorov} and
Lemma~\ref{lemma:generalized-Sauer-Shelah-lemma-in-L_p-MusLauGue19} at hand,
one single formula is needed to establish
Lemma~\ref{lemma:from-gamma-dimension-to-gamma-dimensions}:
a structural result for 
$\mathcal{N}^{\text{int}}_p \left ( \epsilon,
\tilde{\rho}_{\mathcal{G}}, n \right )$.
Proving that 
Lemma~\ref{lemma:from-multivariate-to-univariate-L_p}
still holds true with $\rho_{\mathcal{G}}$ 
replaced with $\tilde{\rho}_{\mathcal{G}}$ is straightforward.
Consequently, the proof of
Lemma~\ref{lemma:from-gamma-dimension-to-gamma-dimensions}
is the following one.

\begin{proof}
Applying in sequence
Proposition~\ref{prop:Proposition-1.4-in-Tal03},
Lemma~\ref{lemma:Kolmogorov} (left-hand side inequality),
Lemma~\ref{lemma:from-multivariate-to-univariate-L_p}
(applied to $\tilde{\rho}_{\mathcal{G}}$),
Lemma~\ref{lemma:Kolmogorov} (right-hand side inequality) and
Lemma~\ref{lemma:generalized-Sauer-Shelah-lemma-in-L_p-MusLauGue19} gives:

$$
\forall \gamma \in \left ( 0, M_{\mathcal{G}} \right ], \;
\gamma\text{-dim} \left ( \tilde{\rho}_{\mathcal{G}} \right )
\leqslant
10 K_p p
\log_2 \left ( \frac{24 M_{\mathcal{G}} p^{\frac{1}{7}} C^{\frac{1}{p}}}
{\gamma} \right )
\sum_{k=1}^C \left ( \frac{\gamma}{72 p C^{\frac{1}{p}}} \right )\text{-dim}
\left ( \mathcal{G}_k \right ).
$$
Let us set $p = \left \lceil \log_2 \left ( C \right ) \right \rceil$
(which is possible since $C \geqslant 3$ implies $p \geqslant 2$).
Then, $C^{\frac{1}{p}} \leqslant 2$,
so that for every $\gamma \in \left ( 0, M_{\mathcal{G}} \right ]$,
\begin{align*}
\gamma\text{-dim} \left ( \tilde{\rho}_{\mathcal{G}} \right )
& \leqslant \;
10 K_{\left \lceil \log_2 \left ( C \right ) \right \rceil}
\log_2 \left ( 2C \right )
\log_2 \left ( \frac{48 M_{\mathcal{G}}
\log_2^{\frac{1}{7}} \left ( 2C \right )}{\gamma} \right )
\sum_{k=1}^C \left (
\frac{\gamma}{144 \log_2 \left ( 2C \right )} \right )\text{-dim}
\left ( \mathcal{G}_k \right ).
\end{align*}
To obtain \eqref{eq:from-gamma-dimension-to-gamma-dimensions},
it suffices to notice that
$$
\forall C \in \mathbb{N} \setminus \llbracket 0; 2 \rrbracket, \;
K_{\left \lceil \log_2 \left ( C \right ) \right \rceil}
\leqslant \min \left \{ 4 \left ( \frac{C}{C-2} \right )^2, 16 \right \}.
$$
\end{proof}

\subsection{Margin Natarajan Dimension of 
\texorpdfstring{$\rho_{\mathcal{G}}$}{the Class of Margin Functions:
General Case}}

The proof of Lemma~\ref{lemma:from-gamma-N-dimension-to-gamma-dimension}
is the following one.

\begin{proof}
Suppose that for $\gamma \in \left ( 0, M_{\mathcal{G}} \right ]$,
the triplet $\left ( s_{\mathcal{Z}^n}, \mathbf{b}_n, \mathbf{c}_n \right )$
is $\gamma$-$N$-shattered by $\rho_{\mathcal{G}}$.
According to 
Proposition~\ref{prop:margin-Natarajan-dimension-alternate-definition},
in order to upper bound $n$, i.e., 
$\gamma\text{-N-dim} \left ( \rho_{\mathcal{G}} \right )$,
one can assume that for every $i \in \llbracket 1; n \rrbracket$,
$y_i < c_i$, and the biases can be negative.
Let $\bar{\mathcal{G}} = \left \{ g^{\mathbf{s}_n}: \;
\mathbf{s}_n \in \left \{ -1, 1 \right \}^n \right \}$
be a subset of $\mathcal{G}$ (of minimal cardinality) such that
$\rho_{\bar{\mathcal{G}}} = \left \{ \rho_{g^{\mathbf{s}_n}}: \;
\mathbf{s}_n \in \left \{ -1, 1 \right \}^n \right \}$
$\gamma$-$N$-shatters
$\left ( s_{\mathcal{Z}^n}, \mathbf{b}_n, \mathbf{c}_n \right )$.
For every pair $\left ( k, l \right ) \in \llbracket 1; C \rrbracket^2$
satisfying $k < l$, let $\mathcal{S}_{k,l}$ be the subset of
$\llbracket 1; n \rrbracket$ defined as follows:
$$
\mathcal{S}_{k,l} = \left \{i \in \llbracket 1; n \rrbracket: \;
y_i = k \text{ and } c_i = l \right \}
$$
and let $n_{k,l} \in \llbracket 0; n \rrbracket$ be its cardinality.
By construction,
$\mathcal{P} = \left \{ \mathcal{S}_{k,l}: \; n_{k,l} > 0 \right \}$
is a partition of $\llbracket 1; n \rrbracket$.
For every vector
$\mathbf{s}_n = \left ( s_i \right )_{1 \leqslant i \leqslant n}
\in \left \{ -1, 1 \right \}^n$,
the function $g^{\mathbf{s}_n}$ satisfies:
$$
\forall i \in \llbracket 1; n \rrbracket, \;\;
\begin{cases}
\text{if } s_i = 1, \;
\rho_{g^{\mathbf{s}_n}} \left ( x_i, y_i \right ) - b_i
\geqslant \gamma \\
\text{if } s_i = -1, \;
\rho_{g^{\mathbf{s}_n}} \left ( x_i, c_i \right ) + b_i
\geqslant \gamma
\end{cases}.
$$
For a fixed $\mathcal{S}_{k,l} \in \mathcal{P}$, this implies that
$$
\forall i \in \mathcal{S}_{k,l}, \;\;
s_i \left ( \frac{1}{2} \left ( g_k^{\mathbf{s}_n} \left ( x_i \right )
- g_l^{\mathbf{s}_n} \left ( x_i \right ) \right ) - b_i \right )
\geqslant \gamma.
$$
Let $\mathcal{D}_{\mathcal{G},k,l} = \left \{ \frac{1}{2} \left (
g_k - g_l \right ): \; \left ( g_k, g_l \right ) \in \mathcal{G}_k \times
\mathcal{G}_l \right \}$.
By definition, we have established that its subset
$\left \{ \frac{1}{2} \left (
g_k^{\mathbf{s}_n} - g_l^{\mathbf{s}_n} \right ): \;
\mathbf{s}_n \in \left \{ -1, 1 \right \}^n \right \}$
$\gamma$-shatters a set of cardinality $n_{k,l}$,
with the consequence that
$$
n_{k,l} \leqslant \gamma\text{-dim} \left (
\mathcal{D}_{\mathcal{G},k,l} \right ).
$$
Summing over all the elements
of the partition $\mathcal{P}$ gives
\begin{equation}
\label{eq:from-gamma-N-dimension-to-gamma-dimension-part-1}
\gamma\text{-N-dim} \left ( \rho_{\mathcal{G}} \right )
\leqslant \sum_{k=1}^{C-1} \sum_{l=k+1}^C
\gamma\text{-dim} \left ( \mathcal{D}_{\mathcal{G},k,l} \right ).
\end{equation}
If the ball $\mathcal{G}_0$ exists, then it contains all the classes
$\mathcal{D}_{\mathcal{G},k,l}$ 
(since $\max  \left \{ \left \| g_k \right \|, \left \| g_l \right \| \right \}
\leqslant \Lambda \Longrightarrow
\left \| \frac{1}{2} \left ( g_k - g_l \right ) \right \|
\leqslant \Lambda$).
Consequently, \eqref{eq:from-gamma-N-dimension-to-gamma-dimension-part-1}
directly implies \eqref{eq:from-gamma-N-dimension-to-gamma-dimension-Banach}.
Otherwise, to upper bound the $\gamma$-dimensions above,
we can resort to the strategy already
used to prove Lemmas~\ref{lemma:from-gamma-dimension-to-gamma-dimensions-old}
and \ref{lemma:from-gamma-dimension-to-gamma-dimensions}.
Applying in sequence
Proposition~\ref{prop:Proposition-1.4-in-Tal03} and
Lemma~\ref{lemma:Kolmogorov} (left-hand side inequality) gives:
\begin{align*}
\gamma\text{-dim} \left ( \mathcal{D}_{\mathcal{G},k,l} \right )
& \leqslant \;
16 \log_2 \left ( \mathcal{M}_2 \left ( \gamma, \mathcal{D}_{\mathcal{G},k,l},
\gamma\text{-dim} \left ( \mathcal{D}_{\mathcal{G},k,l} 
\right ) \right ) \right ) \\
& \leqslant \;
16 \log_2 \left ( \mathcal{N}_2^{\text{int}}
\left ( \frac{\gamma}{2}, \mathcal{D}_{\mathcal{G},k,l},
\gamma\text{-dim} \left ( \mathcal{D}_{\mathcal{G},k,l} 
\right ) \right ) \right ).
\end{align*}
Proceeding as in the proof 
of Lemma~\ref{lemma:from-multivariate-to-univariate-L_p}, we get
$$
\mathcal{N}_2^{\text{int}}
\left ( \frac{\gamma}{2}, \mathcal{D}_{\mathcal{G},k,l},
\gamma\text{-dim} \left ( \mathcal{D}_{\mathcal{G},k,l} 
\right ) \right ) \leqslant
\mathcal{N}_2^{\text{int}}
\left ( \frac{\gamma}{2 \sqrt{2}}, \mathcal{G}_k,
\gamma\text{-dim} \left ( \mathcal{D}_{\mathcal{G},k,l} 
\right ) \right ) \times
\mathcal{N}_2^{\text{int}}
\left ( \frac{\gamma}{2 \sqrt{2}}, \mathcal{G}_l,
\gamma\text{-dim} \left ( \mathcal{D}_{\mathcal{G},k,l} 
\right ) \right ).
$$
To finish the proof of
\eqref{eq:from-gamma-N-dimension-to-gamma-dimension}, 
it suffices to apply 
Lemma~\ref{lemma:Kolmogorov} (right-hand side inequality) and
Theorem~1 in \citet{MenVer03}.
Indeed, this produces:
\begin{align*}
\mathcal{N}_2^{\text{int}}
\left ( \frac{\gamma}{2 \sqrt{2}}, \mathcal{G}_k,
\gamma\text{-dim} \left ( \mathcal{D}_{\mathcal{G},k,l} 
\right ) \right )
& \leqslant \;
\mathcal{M}_2
\left ( \frac{\gamma}{2 \sqrt{2}}, \mathcal{G}_k,
\gamma\text{-dim} \left ( \mathcal{D}_{\mathcal{G},k,l} 
\right ) \right ) \\
& \leqslant \;
\left ( \frac{24 \sqrt{2} M_{\mathcal{G}}}{\gamma}
\right )^{20 d_k \left ( \frac{\gamma}{96 \sqrt{2}} \right )},
\end{align*}
where $d_k \left ( \epsilon \right )
= \epsilon\text{-dim} \left ( \mathcal{G}_k \right )$.
\end{proof}

\subsection{Margin Natarajan Dimension of 
\texorpdfstring{$\rho_{\mathcal{H}_{\Lambda}}$}{the Class of Margin Functions:
C-category SVMs}}

The proof of Lemma~\ref{lemma:gamma-N-dimension-of-M-SVMs} makes use of
that of Theorem~4.6 in \citet{BarSha99}.

\begin{proof}
This proof reuses the notations of the proof of
Lemma~\ref{lemma:from-gamma-N-dimension-to-gamma-dimension},
with $\mathcal{G}$ being instantiated by $\mathcal{H}_{\Lambda}$.
By application of Lemma~4.3 in \citet{BarSha99},
there exists a vector
$\mathbf{s}_n^{\prime} = \left ( s_i^{\prime}
\right )_{1 \leqslant i \leqslant n} \in \left \{ -1, 1 \right \}^n$
satisfying
\begin{equation}
\label{eq:extended-Lemma-4.3-in-BarSha99-part-6}
\forall \mathcal{S}_{k,l} \in \mathcal{P}, \;\;
\left \| \sum_{i \in \mathcal{S}_{k,l}^+} \kappa_{x_i}
-\sum_{i \in \mathcal{S}_{k,l}^-} \kappa_{x_i} \right \|_{\mathbf{H}_{\kappa}}
\leqslant \sqrt{n_{k,l}} \Lambda_{\mathcal{X}},
\end{equation}
where the sets $\mathcal{S}_{k,l}^+$ and $\mathcal{S}_{k,l}^-$ are defined
as follows:
$$
\forall \mathcal{S}_{k,l} \in \mathcal{P}, \;\;
\begin{cases}
\mathcal{S}_{k,l}^+ =
\left \{i \in \mathcal{S}_{k,l}: \; s_i^{\prime} = 1 \right \} \\
\mathcal{S}_{k,l}^- = \mathcal{S}_{k,l} \setminus \mathcal{S}_{k,l}^+
\end{cases}.
$$
For every vector
$\mathbf{s}_n = \left ( s_i \right )_{1 \leqslant i \leqslant n}
\in \left \{ -1, 1 \right \}^n$,
the function
$h^{\mathbf{s}_n}$ satisfies:
$$
\forall i \in \llbracket 1; n \rrbracket, \;\;
\begin{cases}
\text{if } s_i = 1, \;
\rho_{h^{\mathbf{s}_n}} \left ( x_i, y_i \right ) - b_i
\geqslant \gamma \\
\text{if } s_i = -1, \;
\rho_{h^{\mathbf{s}_n}} \left ( x_i, c_i \right ) + b_i
\geqslant \gamma
\end{cases}.
$$
For a fixed $\mathcal{S}_{k,l} \in \mathcal{P}$,
applying the reproducing property gives
\begin{equation}
\label{eq:extended-Lemma-4.3-in-BarSha99-part-1}
\forall i \in \mathcal{S}_{k,l}, \;\;
\begin{cases}
\text{if } s_i = 1, \;
\frac{1}{2} \ps{h_k^{\mathbf{s}_n}
- h_l^{\mathbf{s}_n}, \kappa_{x_i}}_{\mathbf{H}_{\kappa}}
- b_i \geqslant \gamma \\
\text{if } s_i = -1, \;
\frac{1}{2} \ps{h_l^{\mathbf{s}_n}
- h_k^{\mathbf{s}_n}, \kappa_{x_i}}_{\mathbf{H}_{\kappa}}
+ b_i \geqslant \gamma
\end{cases}.
\end{equation}
Let us specify the vector $\mathbf{s}_n$ in the following way:
$\forall i \in \mathcal{S}_{k,l}, \; s_i = s_i^{\prime}$.
By summation over $i \in \mathcal{S}_{k,l}$, it results from
\eqref{eq:extended-Lemma-4.3-in-BarSha99-part-1} that:
$$
\frac{1}{2} \ps{h_k^{\mathbf{s}_n}
- h_l^{\mathbf{s}_n}, \sum_{i \in \mathcal{S}_{k,l}^+} \kappa_{x_i}
-\sum_{i \in \mathcal{S}_{k,l}^-} \kappa_{x_i}
}_{\mathbf{H}_{\kappa}}
- \sum_{i \in \mathcal{S}_{k,l}^+} b_i
+ \sum_{i \in \mathcal{S}_{k,l}^-} b_i
\geqslant n_{k,l} \gamma.
$$
Conversely, consider any vector $\mathbf{s}_n$
such that: $\forall i \in \mathcal{S}_{k,l}, \;
s_i = -s_i^{\prime}$. Then,
$$
\frac{1}{2} \ps{h_l^{\mathbf{s}_n}
- h_k^{\mathbf{s}_n}, \sum_{i \in \mathcal{S}_{k,l}^+} \kappa_{x_i}
-\sum_{i \in \mathcal{S}_{k,l}^-} \kappa_{x_i}
}_{\mathbf{H}_{\kappa}}
+ \sum_{i \in \mathcal{S}_{k,l}^+} b_i
- \sum_{i \in \mathcal{S}_{k,l}^-} b_i
\geqslant n_{k,l} \gamma.
$$
As a consequence, if $\sum_{i \in \mathcal{S}_{k,l}^+} b_i
-\sum_{i \in \mathcal{S}_{k,l}^-} b_i \geqslant 0$,
there are functions $h^{\mathbf{s}_n}$ such that
\begin{equation}
\label{eq:extended-Lemma-4.3-in-BarSha99-part-2}
\frac{1}{2} \ps{h_k^{\mathbf{s}_n}
- h_l^{\mathbf{s}_n}, \sum_{i \in \mathcal{S}_{k,l}^+} \kappa_{x_i}
-\sum_{i \in \mathcal{S}_{k,l}^-} \kappa_{x_i}
}_{\mathbf{H}_{\kappa}}
\geqslant n_{k,l} \gamma,
\end{equation}
whereas if $\sum_{i \in \mathcal{S}_{k,l}^+} b_i
-\sum_{i \in \mathcal{S}_{k,l}^-} b_i < 0$,
there are (different) functions $h^{\mathbf{s}_n}$ such that
\begin{equation}
\label{eq:extended-Lemma-4.3-in-BarSha99-part-3}
\frac{1}{2} \ps{h_l^{\mathbf{s}_n}
- h_k^{\mathbf{s}_n}, \sum_{i \in \mathcal{S}_{k,l}^+} \kappa_{x_i}
-\sum_{i \in \mathcal{S}_{k,l}^-} \kappa_{x_i}
}_{\mathbf{H}_{\kappa}}
\geqslant n_{k,l} \gamma.
\end{equation}
Applying the Cauchy-Schwarz inequality to both
\eqref{eq:extended-Lemma-4.3-in-BarSha99-part-2} and
\eqref{eq:extended-Lemma-4.3-in-BarSha99-part-3} yields
\begin{equation}
\label{eq:extended-Lemma-4.3-in-BarSha99-part-4}
\frac{1}{2} \left \| h_k^{\mathbf{s}_n} - h_l^{\mathbf{s}_n}
\right \|_{\mathbf{H}_{\kappa}}
\left \| \sum_{i \in \mathcal{S}_{k,l}^+} \kappa_{x_i}
-\sum_{i \in \mathcal{S}_{k,l}^-} \kappa_{x_i} \right \|_{\mathbf{H}_{\kappa}}
\geqslant n_{k,l} \gamma.
\end{equation}
Consequently, whatever the sign of
$\sum_{i \in \mathcal{S}_{k,l}^+} b_i
-\sum_{i \in \mathcal{S}_{k,l}^-} b_i$,
there are functions $h^{\mathbf{s}_n}$ specified only
by the components of ${\mathbf{s}_n}$ whose indices
belong to $\mathcal{S}_{k,l}$ such that
\eqref{eq:extended-Lemma-4.3-in-BarSha99-part-4} holds true.
To sum up, we have exhibited an algorithm taking in input
$s_{\mathcal{Z}^n}$, $\left ( \mathbf{b}_n, \mathbf{c}_n \right )$
and $\mathbf{s}^{\prime}_n$, and returning a vector
$\mathbf{s}_n \in \left \{ -1, 1 \right \}^n$ such that the function
$h^{\mathbf{s}_n} \in \mathcal{H}_{\Lambda}$ satisfies:
\begin{equation}
\label{eq:extended-Lemma-4.3-in-BarSha99-part-5}
\forall \mathcal{S}_{k,l} \in \mathcal{P}, \;\;
\frac{1}{2} \left \| h_k^{\mathbf{s}_n} - h_l^{\mathbf{s}_n}
\right \|_{\mathbf{H}_{\kappa}}
\left \| \sum_{i \in \mathcal{S}_{k,l}^+} \kappa_{x_i}
-\sum_{i \in \mathcal{S}_{k,l}^-} \kappa_{x_i} \right \|_{\mathbf{H}_{\kappa}}
\geqslant n_{k,l} \gamma.
\end{equation}
By substitution of \eqref{eq:extended-Lemma-4.3-in-BarSha99-part-6} into
\eqref{eq:extended-Lemma-4.3-in-BarSha99-part-5},
this function also satisfies:
$$
\forall \mathcal{S}_{k,l} \in \mathcal{P}, \;\;
n_{k,l} \leqslant \left (
\frac{\frac{1}{2} \left \| h_k^{\mathbf{s}_n} - h_l^{\mathbf{s}_n}
\right \|_{\mathbf{H}_{\kappa}}
\Lambda_{\mathcal{X}}}{\gamma} \right )^2.
$$
By summation over all the elements of the partition $\mathcal{P}$,
\begin{equation}
\label{eq:extended-Lemma-4.3-in-BarSha99-part-7}
n \leqslant \left ( \frac{\Lambda_{\mathcal{X}}}{2 \gamma} \right )^2
\sum_{k < l} \left \| h_k^{\mathbf{s}_n} - h_l^{\mathbf{s}_n}
\right \|_{\mathbf{H}_{\kappa}}^2.
\end{equation}
Now, since by hypothesis, $\sum_{k=1}^C h_k = 0_{\mathbf{H}_{\kappa}}$,
\begin{align}
\sum_{k < l} \left \| h_k^{\mathbf{s}_n} - h_l^{\mathbf{s}_n}
\right \|_{\mathbf{H}_{\kappa}}^2 & = \;
C \sum_{k=1}^C \left \| h_k^{\mathbf{s}_n} \right \|_{\mathbf{H}_{\kappa}}^2
\nonumber \\
& = \; C \left \| h^{\mathbf{s}_n} \right \|_{\mathbf{H}_{\kappa, C}}^2
\nonumber \\
\label{eq:extended-Lemma-4.3-in-BarSha99-part-8}
& \leqslant C \Lambda^2.
\end{align}
A substitution of \eqref{eq:extended-Lemma-4.3-in-BarSha99-part-8} into
\eqref{eq:extended-Lemma-4.3-in-BarSha99-part-7} then concludes the proof.
\end{proof}

\section{Upper Bound on the Rademacher Complexity}
The proof of Theorem~\ref{theorem:dependence-on-m-C-gamma-L_2-norm}
is the following one.

\begin{proof}
In all three cases, the starting point is 
Inequality~\eqref{eq:New-bound-on-Rademacher-complexity}.
\paragraph{First case:} $d_{\mathcal{G}, \gamma} \in \left ( 0, 2 \right )$\\
This case is the only one for which the entropy integral exists.
Setting for every $j \in \mathbb{N}$,
$h \left ( j \right ) = \gamma 2^{-\frac{2}{2 - d_{\mathcal{G}, \gamma}}j}$,
we obtain
$$
R_m \left ( \rho_{\mathcal{G}, \gamma} \right )
\leqslant
8 \left ( 1 + 2^{\frac{2}{2 - d_{\mathcal{G}, \gamma}}} \right )
\sqrt{\frac{F_1 \left ( C \right )}{m}}
\gamma^{1 - \frac{d_{\mathcal{G}, \gamma}}{2}}
\int_0^{\frac{1}{2}}
\ln \left ( \left ( C-1 \right ) 
\left ( 4 \epsilon^{-\frac{2}{2 - d_{\mathcal{G}, \gamma}}} \right )^5 \right )
d \epsilon.
$$
Let $I \left ( C \right )$ denote the integral. Then,
\begin{align*}
I \left ( C \right ) & = \;
\int_0^{\frac{1}{2}}
\ln \left ( \left ( C-1 \right ) 
\left ( 4 \epsilon^{-\frac{2}{2 - d_{\mathcal{G}, \gamma}}} \right )^5 \right )
d \epsilon \\
& = \;
\frac{1}{2} \left ( \ln \left ( \left ( C-1 \right ) 4^5 \right ) +
10 \frac{1 +\ln \left ( 2 \right )}{2 - d_{\mathcal{G}, \gamma}} \right ).
\end{align*}

\paragraph{Second case:} $d_{\mathcal{G}, \gamma} = 2$
$$
R_m \left ( \rho_{\mathcal{G}, \gamma} \right )
\leqslant
h \left ( N \right )
+ 4 \sqrt{\frac{F_1 \left ( C \right )}{m}}
\sum_{j \in \mathcal{J}}
\frac{h \left ( j \right ) + h \left ( j-1 \right )}{h \left ( j \right )}
\ln \left ( \left ( C-1 \right ) 
\left ( \frac{4 \gamma}{h \left ( j \right )} \right )^5 \right ).
$$
We set $N = \left \lceil \log_2 \left ( 
\frac{\sqrt{m}}{\log_2 \left ( m \right )} \right ) \right \rceil$
and $h \left ( j \right ) = \gamma 2^{-j}$.
Then,
\begin{align*}
R_m \left ( \rho_{\mathcal{G}, \gamma} \right )
& \leqslant \;
\gamma \frac{\log_2 \left ( m \right )}{\sqrt{m}}
+ 12 \sqrt{\frac{F_1 \left ( C \right )}{m}}
\sum_{j=1}^N 
\ln \left ( \left ( C-1 \right ) 
\left ( 4 \cdot 2^j \right )^5 \right ) \\
& \leqslant \;
\gamma \frac{\log_2 \left ( m \right )}{\sqrt{m}}
+ 12 \sqrt{\frac{F_1 \left ( C \right )}{m}}
\left \lceil \log_2 \left ( 
\frac{\sqrt{m}}{\log_2 \left ( m \right )} \right ) \right \rceil
\left \{
\ln \left ( \left ( C-1 \right ) 4^5 \right ) 
+ \frac{5}{2} \ln \left (
4 \frac{\sqrt{m}}{\log_2 \left ( m \right )} \right ) \right \}.
\end{align*}

\paragraph{Third case:} $d_{\mathcal{G}, \gamma} > 2$\\
For $N = \left \lceil \frac{d_{\mathcal{G}, \gamma}-2}
{2 d_{\mathcal{G}, \gamma}}
\log_2 \left ( \frac{m}{\log_2 \left ( m \right )} \right ) \right \rceil$,
let us set
$h \left ( j \right ) = \gamma
\left ( \frac{\log_2 \left ( m \right )}{m} \right )^{\frac{1}
{d_{\mathcal{G}, \gamma}}}
2^{\frac{2}{d_{\mathcal{G}, \gamma}-2}
\left ( -j+N \right )}$.
We then get
$$
R_m \left ( \rho_{\mathcal{G}, \gamma} \right )
\leqslant
\gamma
\left ( \frac{\log_2 \left ( m \right )}{m} \right )^{\frac{1}
{d_{\mathcal{G}, \gamma}}}
\left [ 1 + 4 \left ( 1 + 2^{\frac{2}{d_{\mathcal{G}, \gamma}-2}} \right ) 
\left ( \frac{1}{\gamma} \right )^{\frac{d_{\mathcal{G}, \gamma}}{2}}
\sqrt{\frac{F_1 \left ( C \right )}{ \log_2 \left ( m \right )}}
S_N \right ]
$$
with
\begin{align*}
S_N & = \; \sum_{j=1}^N
2^{j-N}
\ln \left ( \left ( C-1 \right ) 
\left ( \frac{4 \gamma}{h \left ( j \right )} \right )^5 \right ) \\
& \leqslant \;
\ln \left ( \left ( C-1 \right ) 
\left ( \frac{4 \gamma}{h \left ( N \right )} \right )^5 \right )
\sum_{j=1}^N 2^{j-N} \\
& \leqslant \;
2 \ln \left ( \left ( C-1 \right ) 
\left ( 4 \left ( \frac{m}{\log_2 \left ( m \right )} 
\right )^{\frac{1}{d_{\mathcal{G}, \gamma}}}
\right )^5 \right ).
\end{align*}
\end{proof}

\end{document}